%% file: thesis.tex
\documentclass{harvard-thesis}

\usepackage{amsthm}
\usepackage{amsfonts}
\usepackage{enumitem}
\usepackage{mathtools}

\usepackage{scalerel,stackengine}

\usepackage{titlesec}

\input{macros/code-macros.tex}
\input{macros/latex-macros.tex}

\input{macros/more-macros.tex}

\usepackage{pifont}
\newcommand{\cmark}{\ding{51}}%
\newcommand{\xmark}{\ding{55}}%

\usepackage{nameref}

\numberwithin{theorem}{section}
\numberwithin{lemma}{section}

\usepackage{tikz}
\usetikzlibrary{matrix,arrows,decorations.pathmorphing}
\usepackage{tikz-cd}

\usepackage[ruled,vlined]{algorithm2e}
\SetKwInput{KwInput}{Input}
\SetKwInput{KwOutput}{Output}
\DontPrintSemicolon

\graphicspath{ {./figures/} }

\begin{document}

\input{frontmatter/coverpage}
\maketitle
\tableofcontents
\authorsnote

\baselineskip=13.8pt

\include{chapters/ch1-intro} 
\include{chapters/ch2-foundations}
\include{chapters/ch3-rkhs}
\include{chapters/ch4-graphsandgeo}

\include{chapters/ch5-manireg}
\include{chapters/appendix}

\singlespacing

\clearpage
\nocite{*}
\bibliographystyle{plainnat}
{\footnotesize \bibliography{references}}
\addcontentsline{toc}{chapter}{References}
\include{endmatter/colophon}

\end{document}

%% file: macros/code-macros.tex
\DeclareFixedFont{\ttb}{T1}{txtt}{bx}{n}{10} 
\DeclareFixedFont{\ttm}{T1}{txtt}{m}{n}{10}  

\usepackage{color}
\definecolor{deepblue}{rgb}{0,0,0.5}
\definecolor{deepred}{rgb}{0.6,0,0}
\definecolor{deepgreen}{rgb}{0,0.5,0}

\usepackage{listings}

\newcommand\pythonstyle{\lstset{
language=Python,
basicstyle=\ttm,
otherkeywords={self},             
keywordstyle=\ttb\color{deepblue},
emph={MyClass,__init__},          
emphstyle=\ttb\color{deepred},    
stringstyle=\color{deepgreen},
frame=tb,                         
showstringspaces=false            %
}}

\lstnewenvironment{python}[1][]
{
\pythonstyle
\lstset{#1}
}
{}

\newcommand\pythoninline[1]{{\pythonstyle\lstinline!#1!}}

%% file: macros/latex-macros.tex
\newcommand{\R}{\mathbb{R}}
\newcommand{\E}{\mathbb{E}}
\newcommand{\cov}{\text{Cov}}

\renewcommand{\P}{\mathbb{P}}

\DeclareMathOperator*{\argmin}{arg\,min}

\newcommand{\qtxtq}[1]{\qquad \text{#1} \qquad}

\newcommand{\br}{\langle}
\newcommand{\kt}{\rangle}
\newcommand{\tr}{\text{Tr}}

\renewcommand{\l}{\left(}
\renewcommand{\r}{\right)}

\newcommand{\vspan}{\text{span }}

\newcommand{\X}{X}
\renewcommand{\H}{\mathcal{H}}

\newcommand{\Lt}{L^2}
\newcommand{\LT}{L^2}
\newcommand{\otherwise}{\text{otherwise}}

\newcommand{\mC}{\mathbb{C}}

\newcommand{\mT}{\mathbb{T}}
\newcommand{\mZ}{\mathbb{Z}}
\newcommand{\mN}{\mathbb{N}}

\newcommand{\mF}{\mathbb{F}}

\newcommand{\hf}{\hat{f}}

\newcommand{\fhE}{\hat{\mathcal{E}}}

\newcommand{\ol}[1]{\overline{#1}}
\newcommand{\ep}{\varepsilon}
\newcommand{\da}{\delta}
\newcommand{\be}{\beta}
\newcommand{\al}{\alpha}
\newcommand{\ta}{\theta}
\newcommand{\ga}{\gamma}
\newcommand{\si}{\sigma}
\newcommand{\lam}{\lambda}

\newcommand{\Om}{\Omega}
\newcommand{\grad}{\nabla}
\newcommand{\vol}{\text{vol}}
\newcommand{\area}{\text{area}}
\newcommand{\ind}{\textbf{1}}
\renewcommand{\b}[1]{\textbf{#1}}

\newcommand{\ceil}[1]{\lceil #1 \rceil}
\newcommand{\pl}{\partial}
\newcommand{\ppx}[2]{\frac{\partial #1}{\partial #2}}
\newcommand{\ddx}[2]{\frac{\text{d} #1}{\text{d} #2}}

\newcommand{\norm}[1]{\left\lVert#1\right\rVert}
\newcommand{\ns}[1]{\left\lVert#1\right\rVert^2}

\newcommand{\lap}{\Delta}
\newcommand{\lapm}{\Delta_{\mathcal{M}}}
\newcommand{\lapg}{\textbf{L}}
\newcommand{\lapt}{\textbf{L}^t}
\newcommand{\nlap}{\mathcal{L}}

\stackMath
\newcommand\reallywidehat[1]{%
\savestack{\tmpbox}{\stretchto{%
  \scaleto{%
    \scalerel*[\widthof{\ensuremath{#1}}]{\kern-.6pt\bigwedge\kern-.6pt}%
    {\rule[-\textheight/2]{1ex}{\textheight}}
  }{\textheight}%
}{0.5ex}}%
\stackon[1pt]{#1}{\tmpbox}%
}
\newtheorem{theorem}{Theorem}
\newtheorem{lemma}{Lemma}
\newtheorem{corollary}{Corollary}

\theoremstyle{definition}
\newtheorem{definition}[theorem]{Definition}

%% file: macros/more-macros.tex
\newcommand{\fA}{\mathcal{A}}
\newcommand{\fB}{\mathcal{B}}
\newcommand{\fC}{\mathcal{C}}

\newcommand{\fE}{\mathcal{E}}
\newcommand{\fF}{\mathcal{F}}
\newcommand{\fG}{\mathcal{G}}
\newcommand{\fH}{\mathcal{H}}
\newcommand{\fI}{\mathcal{I}}

\newcommand{\fK}{\mathcal{K}}
\newcommand{\fL}{\mathcal{L}}
\newcommand{\fM}{\mathcal{M}}

\newcommand{\fS}{\mathcal{S}}

\newcommand{\bA}{\textbf{A}}
\newcommand{\bB}{\textbf{B}}

\newcommand{\bD}{\textbf{D}}

\newcommand{\bF}{\textbf{F}}

\newcommand{\bL}{\textbf{L}}

\newcommand{\bW}{\textbf{W}}

\newcommand{\ba}{\textbf{a}}

\renewcommand{\bf}{\textbf{f}}

\newcommand{\bx}{\textbf{x}}

%% file: frontmatter/coverpage.tex
\author{Luke Melas-Kyriazi}
\advisor{Arjun K. Manrai}

\degree{Bachelor of Arts}
\field{Mathematics}
\degreeyear{2020}
\degreemonth{May}

\department{Department of Mathematics}
\university{Harvard University} 
\universitycity{Cambridge}
\universitystate{Massachusetts}

%% file: chapters/ch1-intro.tex
\chapter{Introduction} \label{chap:intro}

\section{What is Learning?}

From an early age, our parents and teachers impress upon us the importance of learning. We go to school, do homework, and write senior theses in the name of learning. But what exactly is learning? 

Theories of learning, which aim to answer this question, stretch back as far as Plato. Plato's theory, as presented in the \textit{Phaedo}, understands learning as the rediscovery of innate knowledge acquired at or before birth. For the past two millennia, epistemologists have debated the meaning and mechanisms of learning, with John Locke notably proposing a theory based on the passive acquisition of simple ideas. Scientific approaches to understanding learning emerged beginning in the nineteenth century. Ivan Pavlov's famous classical conditioning experiments, for example, demonstrated how dogs learned to associate one stimulus (i.e. ringing bells) with another (i.e. food). A multitude of disciplines now have subfields dedicated to theories of learning: psychology, neuroscience, pedagogy, and linguistics, to name only a few. 

Over the past few decades, the rise and proliferation of computers has prompted researchers to consider what it means for a computer algorithm to learn. Specifically, the past two decades have seen a proliferation of research in machine learning, the study of algorithms that can perform tasks without being explicitly programmed. Now ubiquitous, these machine learning algorithms are integrated into a plethora of real-world systems and applications. From Google Search to Netflix’s recommendation engine to Apple’s Face ID software, much of the “intelligence” of modern computer applications is a product of machine learning. 

This thesis takes a mathematical approach to machine learning, with the goal of building and analyzing theoretically-grounded learning algorithms. We focus in particular on the subfield of \textit{semi-supervised learning}, in which machine learning models are trained on both unlabeled and labeled data. In order to understand modern semi-supervised learning methods, we develop an toolkit of mathematical methods in spectral graph theory and Riemannian geometry. Throughout the thesis, we will find that understanding the underlying mathematical structure of machine learning algorithms enables us to interpret, improve, and extend upon them. 

\section{Lessons from Human and Animal Learning}

Although this thesis is concerned entirely with machine learning, the ideas presented within are grounded in our intuition from human and animal learning. That is, we design our mathematical models to match our intuition about what should and should not be considered learning. 

An example here is illustrative. Consider a student who studies for a test using a copy of an old exam. If the student studies in such a way that he or she develops an understanding of the material and can answer new questions about it, he or she has learned something. If instead the student memorizes all the old exam’s questions and answers, but cannot answer any new questions about the material, the student has not actually learned anything. In the jargon of machine learning, we would say that the latter student does not \textit{generalize}: he makes few errors on the questions he has seen before (the \textit{training} data) and many errors on the questions he has not seen before (the \textit{test} data). 

Our formal definition of learning, given in \autoref{chap:foundations}, will hinge upon this idea of generalization. Given a finite number of examples from which to learn, we would like to be able to make good predictions on new, unseen examples. 

Our ability to learn from finite data rests on the foundational assumption that our data has some inherent structure. Intuitively, if we did not assume that our world had any structure, we would not be able to learn anything from past experiences; we need some prior knowledge, an \textit{inductive bias}, to be able to generalize from observed data to unseen data. We can formalize this intuitive notion in the \hyperref[thm:nofreelunch]{No Free Lunch Theorem}, proven in \autoref{chap:foundations}.

Throughout this thesis, we adopt the inductive bias that the functions we work with should be simple. At a high level, this bias is Occam’s Razor: we prefer simpler explanations of our data to more complex ones. Concretely, this bias takes the form of \textit{regularization}, in which we enforce that the norm of our learned function is small. 

The thesis builds up to a type of regularization called \textit{manifold regularization}, in which the norm of our function measures its smoothness with respect to the manifold on which our data lie. Understanding manifold regularization requires developing a substantial amount of mathematical machinery, but it is worth the effort because it will enable us to express the inductive bias that our functions should be simple. 

\section{Types of Learning}

In computational learning,  types of learning are generally categorized by the data available to the learner. Below, we give an overview of the three primary types of computational learning: supervised, semi-supervised, and unsupervised learning. An illustration is shown in \autoref{types_of_learning}. 

\subsection{Supervised Learning}

The goal of supervised learning is to approximate a function $f: X \to Y$ using a training set $S = \{x_i, y_i\}_{i=1}^N$. Note that the space of inputs $X$ and the space of outputs $Y$ are entirely general. For example, $X$ or $Y$ may contain vectors, strings, graphs, or molecules. Usually, we will consider problems for which $Y$ is $\R$ (regression) or for which $Y$ is a set of classes $Y = \fC = \{0, 1, \cdots, n-1\}$ (classification). The special case $Y = \{0,1\}$ is called binary classification. 

The defining feature of supervised learning is that the training set $S$ is fully-labeled, which means that every point $x_i$ has a corresponding label $y_i$. 

\paragraph{Example: Image Classification} Image classification is the canonical example of a supervised learning task in the field of computer vision. Here, $X$ is the set of (natural) images and $Y$ is a set of $|\fC|$ categories. Given an image $x_i \in X$, the task is to classify the image, which is to assign it a label $y_i \in Y$. The standard large-scale classification dataset ImageNet \cite{deng2009imagenet} has $|\fC| = 1000$ categories and $|S| \approx 1,200,000$ hand-labeled training images. 

\begin{figure}[]
    \centering
    \includegraphics[width=\textwidth]{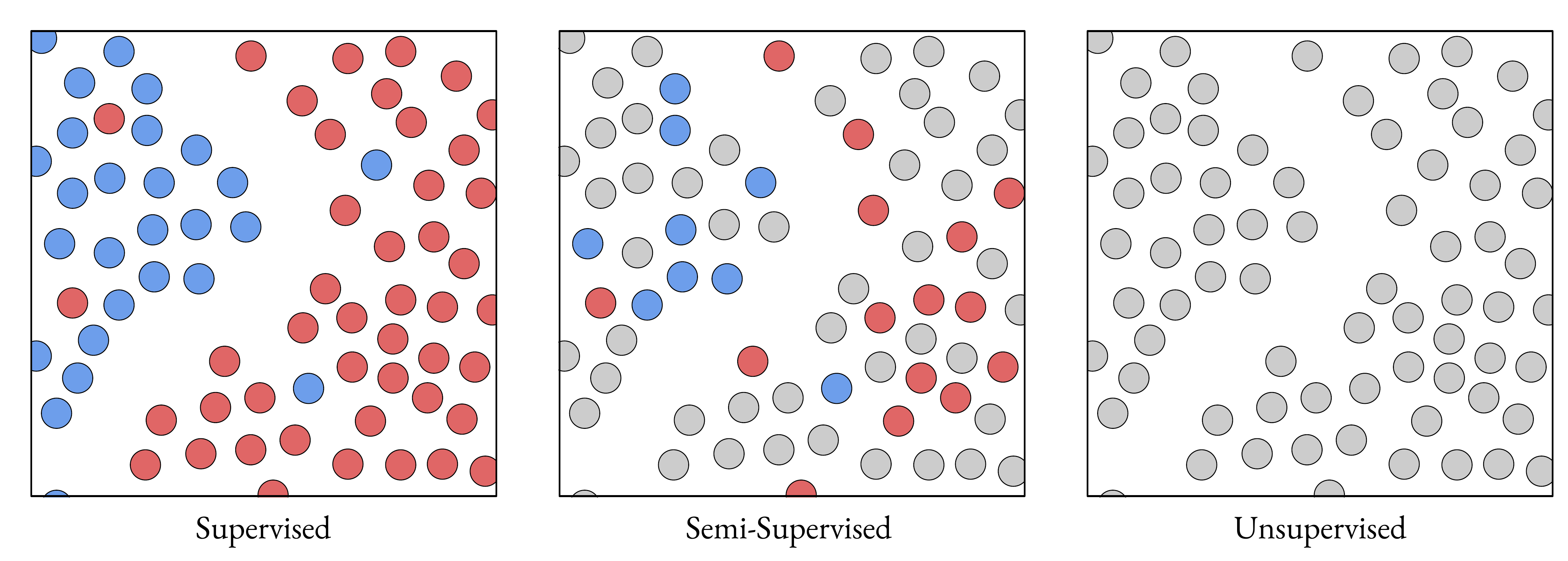}
    \caption[Types of Learning]{An illustration of supervised, semi-supervised, and unsupervised learning.}
    \label{types_of_learning}
\end{figure}

\subsection{Semi-Supervised Learning}

In semi-supervised learning, the learner is given access to labeled training set $S_L = \{x_i, y_i\}_{i=1}^{N_L}$ along with unlabeled data $S_U = \{x_i\}_{i=1}^{N_U}$. Usually, the size of the unlabeled data is much larger than the size of the labeled data: $N_U \gg N_L$. 

It is possible to turn any semi-supervised learning problem into a supervised learning problem by discarding the unlabeled data $S_U$ and training a model using only the labeled data $S_L$. The challenge of semi-supervised learning is to use the information in the unlabeled data to train a better model than could be trained with only $S_L$. Semi-supervised learning is the focus of this thesis. 

\paragraph{Example: Semi-Supervised Semantic Segmentation} Semantic segmentation is the task of classifying every pixel of an image into a set of categories; it may be thought of as pixelwise image classification. Semantic segmentation models play a key role in self-driving car systems, as a self-driving car needs to identify what objects (vehicles, bikes, pedestrians, etc.) are on the road ahead of it.

High-resolution images contain millions of pixels, so labeling them for the task of semantic segmentation is time-consuming and expensive. For example, for one popular dataset with 5000 images, each image took over 90 minutes for a human to annotate \cite{Cordts2016TheCD}.\footnote{Labeling images for segmentation is so arduous that it has become a large industry: Scale AI, a startup that sells data labeling services to self-driving car companies, is valued at over a billion dollars. According to their website, they charge \$6.40 per annotated frame for image segmentation. If you were to record video at 30 frames-per-second for 24 hours and try to label every frame, you would have to label 2,592,000 images. Many of these images would be quite similar, but even if you subsampled to 1 frame-per-second, it would require labeling 86,400 images.}\footnote{Annotation is even more costly in domains such as medical image segmentation, where images must be annotated by highly-trained professionals.}

In semi-supervised semantic segmentation, we train a machine learning model using a small number of labeled images and a large number of unlabeled images. In this way, it is possible to leverage a large amount of easily-collected unlabeled data alongside a small amount of arduously-annotated labeled data. 

\subsection{Unsupervised Learning}

In unsupervised learning, we are given data $X = \{x_i\}_{i=1}^{N}$ without any labels. In this case, rather than trying to learn a function $f$ to a space of labels, we aim to learn useful representations or properties of our data. For example, we may try to cluster our data into semantically meaningful groups, learn a generative model of our data, or perform dimensionality reduction on our data. 

\paragraph{Example: Dimensionality Reduction for Single-cell RNA Data} Researchers in biology performing single-cell RNA sequencing often seek to visualize high-dimensional sequencing data. That is, they aim to embed their high-dimensional data into a lower-dimensional space (e.g. the $2D$ plane) in such a way that it retains its high-dimensional structure. They may also want to cluster their data either before or after applying dimensionality reduction. Both of these tasks may be thought of as unsupervised learning problems, as their goal is to infer the structure of unlabeled data. 

Finally, we should note that there are a plethora of other subfields and subclassifications of learning algorithms: reinforcement learning, active learning, online learning, multiple-instance learning, and more.\footnote{For an in-depth review of many of these fields, reader is encouraged to look at \cite{murphy2012machine}.} For our purposes, we are only concerned with the three types of learning above. 

\section{Manifold Learning}

As we observed above, in order to learn anything from data, we need to assume that the data has some inherent structure. In some machine learning methods, this assumption is implicit. By contrast, the field of \textit{manifold learning} is defined by the fact that it makes this assumption explicit: it assumes that the observed data lie on a low-dimensional manifold embedded in a higher-dimensional space. Intuitively, this assumption, which is known as the \textit{manifold assumption} or sometimes the \textit{manifold hypothesis}, states that the shape of our data is relatively simple.

For example, consider the space of natural images (i.e. images of real-world things). Since images are stored in the form of pixels, this space lies within the pixel space $\R^{H \times W \times 3}$ consisting of all ordered sets of $3\cdot H \cdot W$ real numbers. However, we expect the space of natural images to be much lower dimensional than the pixel space; the pixel space is in some sense almost entirely filled with images that look like ``noise.'' Moreover, we can see that the space of natural images is nonlinear, because the (pixel-wise) average of two natural images is not a natural images. The manifold assumption states that the space of natural images has the differential-geometric structure of a low-dimensional manifold embedded in the high-dimensional pixel space.\footnote{In fact, a significant amount of work has gone into trying to identify the intrinsic dimensionality of the image manifold \cite{gong2019intrinsic}.}

It should be emphasized that manifold learning is not a type of learning in the sense of supervised, semi-supervised, and unsupervised learning. Whereas these types of learning characterize the learning task (i.e. how much labeled data is available), manifold learning refers to a set of methods based on the manifold assumption. Manifold learning methods are used most often in the semi-supervised and unsupervised settings,\footnote{In particular, the manifold learning hypothesis underlies most popular dimensionality reduction techniques: PCA, Isomaps \cite{tenenbaum2000global}, Laplacian Eigenmaps \cite{belkin2003laplacian}, Diffusion maps \cite{coifman2006diffusion}, local linear embeddings \cite{roweis2000nonlinear}, local tangent space alignment \cite{zhang2007linear}, and many others.} but they may be used in the supervised setting as well.

\section{Overview}

This thesis presents the mathematics underlying manifold learning. The presentation combines three areas of mathematics that are not usually linked together: statistical learning, spectral graph theory, and differential geometry. 

The thesis builds up to the idea of \textit{manifold regularization} in the final chapter. At a high level, manifold regularization enables us to learn a function that is simple with respect to the data manifold, rather than the ambient space in which it lies. 

In order to understand manifold learning and manifold regularization, we first need to understand (1) kernel learning, and (2) the relationship between manifolds and graphs. 

Chapters \ref{chap:foundations} and \ref{chap:rkhs} are dedicated to (1). \autoref{chap:foundations} lays the foundations for supervised and semi-supervised learning. \autoref{chap:rkhs} develops the theory of supervised kernel learning in Reproducing Kernel Hilbert Spaces. This theory lays mathematically rigorous foundations for large classes of regularization techniques. 

\autoref{chap:graphsandgeo} is dedicated to (2). It explores the relationship between graphs and manifolds through the lens of the Laplacian operator, a linear operator that can be defined on both graphs and manifolds. Although at first glance these two types of objects may not seem to be very similar, we will see that the Laplacian reveals a remarkable correspondence between them. By the end of the chapter, we will have developed a unifying mathematical view of these seemingly disparate techniques. 

Finally, \autoref{chap:manireg} presents manifold regularization. We will find that, using the Laplacian of a graph generated from our data, it is simple to add manifold regularization to many learning algorithms. At the end of the chapter, we will prove that this graph-based method is theoretically grounded: the Laplacian of the data graph converges to the Laplacian of the data manifold in the limit of infinite data. 

This thesis is designed for a broad mathematical audience. Little background is necessary apart from a strong understanding of linear algebra. A few proofs will require additional background, such as familiarity with Riemannian geometry. Illustrative examples from mathematics and machine learning are incorporated into the text whenever possible. 

%% file: chapters/ch2-foundations.tex
\chapter{Foundations} \label{chap:foundations}

The first step in understanding machine learning algorithms is to define our learning problem. 
In this chapter, we will only work in the supervised setting, generally following the approaches from \cite{mitnotes,shalev2014understanding,castro20182di70}. \autoref{chap:manireg} will extend the framework developed here to the semi-supervised setting. 





\subsection{Learning Algorithms \& Loss Functions}

A learning algorithm $\fA$ is a map from a finite dataset $S$ to a candidate function $\hat{f}$, where $\hat{f}$ is measurable. Note that $\fA$ is stochastic because the data $S$ is a random variable. We assume that our data $(x_i, y_i)$ are drawn independently and identically distributed from a probability space $X \times Y$ with measure $\rho$. 

We define what it means to ``do well'' on a task by introducing a loss function, a measurable function $L: X \times Y \times F \to [0, \infty)$. This loss almost always takes the form $L(x,y,f) = L'(y, f(x))$ for some function $L'$, so we will write the loss in this way moving foward. Intuitively, we should think of $L(y, \hat{f}(x))$ as measuring how costly it is to make a prediction $\hat{f}(x)$ if the true label for $x$ is $y$. If we predict $f(x) = y$, which is to say our prediction at $x$ is perfect, we would expect to incur no loss at $x$ (i.e. $L(y, \hat{f}(x)) = 0$). 

Choosing an appropriate loss function is an important part of using machine learning in practice. Below, we give examples of tasks with different data spaces $X, Y$ and different loss functions $L$.
    
\paragraph{Example: Image Classification}
Image classification, the task of classifying an image $x$ into one of $C$ possible categories, is perhaps the most widely-studied problem in computer vision. Here $x \in R^{H \times W \times 3}$, where $H$ and $W$ are the image height and width, and $3$ corresponds to the three color channels (red, green, and blue). Our label space is a finite set $Y = \fC$ where $|\fC| = C$. A classification model outputs a discrete distribution $f(x_i) = p = (p_1, \dots, p_C)$ over classes, with $p_c$ corresponding to the probability that the input image $x$ has class $c$. \\

As our loss function, we use cross-entropy loss: 
\[ L(y, f(x)) = -\frac{1}{N} \sum_{c=1}^C \ind\{y_i = c\} \log(p_c), \quad p = f(x_i) \]

\paragraph{Example: Semantic Segmentation}
As mentioned in the introduction, semantic segmentation is the task of classifying every pixel in an input image. Here, $X = \fC^{H \times W \times 3}$ like in image classification above, 
but $Y = \fC^{H \times W}$ unlike above. The output $f(x) = p = (p_{c}^{(h,w)})$ is a distribution over classes for each pixel. \\

As our loss function, we use cross-entopy loss averaged across pixels: 
\[ \hspace*{-3pt} L(y, f(x)) = -\frac{1}{N\cdot H\cdot W}  \sum_{h=1}^H \sum_{w=1}^W  \sum_{c=1}^C 1\{y_i^{(h,w)} = c\} \log(p_c^{(h,w)}), \quad p = f(x_i^{(h,w)}) \]

\paragraph{Example: Crystal Property Prediction}
A common task in materials science is to predict the properties of a crystal (e.g. formation energy) from its atomic structure (an undirected graph). As a learning problem, this is a regression problem with $X$ as the set of undirected graphs and $Y=\R$. \\

For the loss function, it is common to use mean absolute error (MAE) due to its robustness to outliers:
\[ L(y, f(x)) = |y - f(x)|\]

\section{The Learning Problem}

Learning is about finding a function $\hat{f}$ that generalizes from our finite data $S$ to the infinite space $X \times Y$. This idea may be expressed as minimizing the expected loss $\fE$, also called the risk: 
\[ \fE(f) = \E[L(y, f(x))] = \int_{X \times Y} L(y, f(x)) \, d\rho(x,y) \]
Our objective in learning is to minimize the risk: 
\[ f^{*} = \argmin_{f \in \fF} \E[L(y, f(x))] = \argmin_{f \in \fF} \int_{X \times Y} L(y, f(x)) \, d\rho(x,y) \]
Since we have finite data, even computing the risk is impossible. Instead, we approximate it using our data, producing the empirical risk: 
\begin{equation} \label{eq:erm}
\fhE(f) = \frac1n \sum_{i=1}^N L(y_i, f(x_i)) \approx \int_{X \times Y} L(y, f(x)) \, d\rho(x,y)
\end{equation}
This concept, \textit{empirical risk minimization}, is the basis of much of modern machine learning. 

One might hope that by minimizing the empirical risk over all measurable functions, we would be able to approximate the term on the right hand side of \ref{eq:erm} and find a function $\hat{f} = \argmin_{f \in \fF}\fhE(f)$ resembling the desired function $f^*$. However, without additional assumptions or priors, this is not possible. In this unconstrained setting, no model can achieve low error across all data distributions, a result known as the No Free Lunch Theorem. 

The difference between the performance of our empirically learned function $\hat{f}$ and the best possible function is called the generalization gap or generalization error. We aim to minimize the probability that this error exceeds $\ep$: 
\[ \P\l \fE(\hat{f}) - \inf_{f \in \fF}\fE(f) > \ep \r \]
Note that here $\P$ refers to the measure $\rho^N$ and that $\hf$ is a random variable because it is the output of $A$ with random variable input $S$.\footnote{Technically $\fA$ could also be random, but for simplicity we will only consider deterministic $A$ and random $S$ here.} 

It would be desirable if this gap were to shrink to zero in the limit of infinite data: 
\begin{equation} \label{eq:consistency} 
    \lim_{n\to\infty} \P \l  \fE(\hat{f}) - \inf_{f \in \fF}\fE(f) > \ep  \r = 0 \qquad \forall \ep > 0
\end{equation}
A learning algorithm with this property is called \textit{consistent} with respect to $\rho$. Stronger, if property \ref{eq:consistency} holds for all fixed distributions $\rho$, the algorithm is \textit{universally consistent}. Even stronger still, an algorithm that is consistent across finite samples from all distributions is \textit{uniformly universally consistent}:
\begin{equation} \label{eq:univconsistency} 
\lim_{n\to\infty} \sup_\rho \P\l \fE(\hat{f}) - \inf_{f \in \fF}\fE(f) > \ep \r = 0 \qquad \forall \ep > 0
\end{equation}

Unfortunately, this last condition is \textit{too} strong. This is the famous ``No Free Lunch'' Theorem. 

\begin{theorem}[No Free Lunch Theorem] \label{thm:nofreelunch}
    No learning algorithm achieves uniform universal consistency. That is, for all $\ep > 0$: 
    \[ \lim_{n\to\infty} \sup_\rho \P\l \fE(\hat{f}) - \inf_{f \in \fF}\fE(f) > \ep \r = \infty \]
\end{theorem}
For a simple proof, the reader is encouraged to see \cite{shalev2014understanding} (Section 5.1). 

\section{Regularization}

The No Free Lunch Theorem states that learning in an entirely unconstrained setting is impossible. Nonetheless, if we constrain our problem, we can make meaningful statements about our ability to learn. 

Looking at Equation \ref{eq:univconsistency}, there are two clear ways to constrain the learning problem: (1) restrict ourselves to a class of probability distributions, replacing $\sup_{\rho}$ with $\sup_{\rho \in \Theta}$, or (2) restrict ourselves to a limited class of target functions $\fH$, replacing $\inf_{f \in \fF}$ with $\inf_{f \in \fH}$. We examine the latter approach, as is common in statistical learning theory.

To make learning tractable, we optimize over a restricted set of hypotheses $\fH$. But how should we choose $\fH$? On the one hand, we would like $\fH$ to be large, so that we can learn complex functions. On the other hand, with large $\fH$, we will find complex functions that fit our training data but do not generalize to new data, a concept known as \textit{overfitting}.


Ideally, we would like to be able to learn complex functions when we have a lot of data, but prefer simpler functions to more complex ones when we have little data. We introduce \textit{regularization} for precisely this purpose. Regularization takes the form of a penalty $R$ added to our loss term, biasing learning toward simpler and smoother functions. 

Most of this thesis is concerned with the question of what it means to be a ``simple'' or ``smooth'' function. Once we can express and compute what it means to be simple or smooth, we can add this as a regularization term to our loss. 

Moreover, if we have any tasks or problem-specific notions of what it means to be a simple function, we can incorporate them into our learning setup as regularization terms. In this way, we can inject into our algorithm prior knowledge about the problem's structure, enabling more effective learning from smaller datasets. 

With regularization, learning problem turns into:
\[ \arg\min_{f \in \fH }  \fhE(f, x, y) + \lam R(f, x, y) \]
where $\fH$ can be a relatively large hypothesis space. 

The parameter $\lam$ balances our empirical risk term and our regularization term. When $\lam$ is large, the objective is dominated by the regularization term, meaning that simple functions are preferred over ones that better fit the data. When $\lam$ is small, the objective is dominated by the empirical risk term, so functions with lower empirical risk are preferred even when they are complex. Tuning $\lam$ is an important element of many practical machine learning problems, and there is a large literature around automatic selection of $\lam$ \cite{abdessalem2017automatic}. 

\textit{Notation: } The full expression $L + \lam R$ is often called the loss function and denoted by the letter $L$. We will clarify notation in the following chapters whenever it may be ambiguous. 

Often, $R$ depends only on the function $f$ and its parameters. We will call this data-independent regularization and write $R(f)$ for ease of notation. The reader may be familiar with common regularization functions (e.g. L1/L2 weight penalties), nearly all of which are data-independent. Manifold regularization, explored in \autoref{chap:manireg}, is an example of data-dependent regularization. 

\paragraph{Example (Data-Independent): Linear Regression} 
In linear regression, it is common to add a regularization term based on the magnitude of the weights to the standard least-squares objective: 
\[ R(f) = ||w||^\al \text{ for $\al > 0$ }\]
When $\al = 2$, this is denoted Ridge Regression, and when $\al = 1$, it is denoted Lasso Regression. Both of these are instances of Tikhonov regularization, a data-independent regularization method explored in the following chapter. 

\paragraph{Example (Data-Dependent): Image Classification}
When dealing with specialized domains such as images, we can incorporate additional inductive biases into our regularization framework. For example, we would expect an image to be classified in the same category regardless of whether it is rotated slightly, cropped, or flipped along a vertical line. \\

Recent work in visual representation learning employs these transformations to define new regularization functions. For example, \cite{xie2019unsupervised} introduces a 
regularization term penalizing the difference between a function's predictions on an image and an augmented version of the same image: 
\[ R(f, x) = KL(f(x), f(\text{Aug}(x)) \]
where $\text{Aug}$ is an augmentation function, such as rotation by $15^{\circ}$, and $KL(\cdot, \cdot)$ is the Kullback–Leibler divergence, a measure of the distance between two distributions (because $f(x)$ is a distribution over $C$ possible classes). This method currently gives state-of-the-art performance on image classification in settings with small amounts of labeled data \cite{xie2019unsupervised}.

%% file: chapters/ch3-rkhs.tex
\chapter{Kernel Learning} \label{chap:rkhs}

In the previous chapter, we described the learning problem as the minimization of the regularized empirical risk over a space of functions $\H$.  

This chapter is dedicated to constructing an appropriate class of function spaces $\H$, known as \textit{Reproducing Kernel Hilbert Spaces}. Our approach is inspired by \cite{vern2016book,mitnotes,berlinet2011reproducing,manton2015primer}. 

Once we understand these spaces, we will find that our empirical risk minimization problem can be greatly simplified. Specifically, the Representer Theorem \ref{thm:representer_theorem} states that its solution can be written as the linear combination of functions (kernels) evaluated at our data points, making optimization over $\H$ as simple as optimization over $\R^n$. 

At the end of the chapter, we develop these tools into the general framework of \textit{kernel learning} and describe three classical kernel learning algorithms. Due to its versatility and simplicity, kernel learning ranks among the most popular approaches to machine learning in practice today. 

\subsection{Motivation}

Our learning problem, as developed in the last chapter, is to minimize the regularized empirical risk 
\[ \arg\min_{f \in \H } \fhE(f, x, y) + \lam R(f, x, y) \]
over a hypothesis space $\H$. The regularization function $R$ corresponds to the inductive bias that simple functions are preferable to complex ones, effectively enabling us to optimize over a large space $\H$. 

At this point, two issues remain unresolved: (1) how to define $\H$ to make optimization possible, and (2) how to define $R$ to capture the complexity of a function. 

If our functions were instead vectors in $\R^d$, both of our issues would be immediately resolved. First, we are computationally adept at solving optimization problems over finite-dimensional Euclidean space. Second, the linear structure of Euclidean space affords us a natural way of measuring the size or complexity of vectors, namely the norm $\norm{v}$. Additionally, over the course of many decades, statisticians have developed an extensive theory of linear statistical learning in $\R^d$. 

In an ideal world, we would be able to work with functions in $\H$ in the same way that we work with vectors in $\R^d$. It is with this motivation that mathematicians developed Reproducing Kernel Hilbert Spaces.   

Informally, a Reproducing Kernel Hilbert Space (RKHS) is a potentially-infinite-dimensional space that looks and feels like Euclidean space. It is defined as a Hilbert space (a complete inner product space) satisfying an additional smoothness property (the \textit{reproducing} property). Like in Euclidean space, we can use the norm $\norm{\cdot}_K$ corresponding to the inner product of the RKHS to measure the complexity of functions in the space. Unlike in Euclidean space, we need an additional property to ensure that if two functions are close in norm, they are also close pointwise. This property is essential because it ensures that functions with small norm are near $0$ everywhere, which is to say that there are no ``complex'' functions with small norm. 

An RKHS is associated with a kernel $K: \X \times \X \to \R$, which may be thought of as a measure of the similarity between two data points $x,x'\in\X$. The defining feature of kernel learning algorithms, or optimization problems over RKHSs, is that the algorithms access the data \textit{only} by means of the kernel function. As a result, kernel learning algorithms are highly versatile; the data space $\X$ can be anything, so long as one can define a similarity measure between pairs of points. For example, it is easy to construct kernel learning algorithms for molecules, strings of text, or images. 

\section{Reproducing Kernel Hilbert Spaces}

We are now ready to formally introduce Reproducing Kernel Hilbert Spaces. 

Recall that a \textit{Hilbert space} $V$ is a complete vector space equipped with an inner product $\br\cdot,\cdot\kt$. In this chapter (except for a handful of examples), we will only work with real vector spaces, but all results can be extended without much hassle to complex-dimensional vector spaces. 

For a set $\X$, we denote by $\R^\X$ the set of functions $\X \mapsto \R$. We give $\R^\X$ a vector space structure by defining addition and scalar multiplication pointwise: 
\[ (f_1 + f_2)(x) = f_1(x) + f_2(x) \qquad (a \cdot f)(x) = a \cdot f(x) \]
Linear functionals, defined as members of the dual space of $R^X$, may be thought of as linear functions $\R^\X \to \R$. A special linear functional $e_x$, called the \textit{evaluation functional}, sends a function $f$ to its value at a point $x$: 
\[ e_x(f) = f(x) \]
When these evaluation functionals are bounded, our set takes on a remarkable amount of structure. 
\begin{definition}[RKHS]
Let $\X$ be a nonempty set. We say $\H$ is a \textit{Reproducing Kernel Hilbert Space} on $\X$ if 
\begin{enumerate}
    \item $\H$ is a vector subspace of $\R^\X$
    \item $\H$ is equipped with an inner product $\br\cdot,\cdot\kt$ (it is a Hilbert Space)
    \item For all $x\in\X$, the linear evaluation functional $e_x: \H \to \R$ is bounded. 
\end{enumerate}
\end{definition}

The last condition implies that $e_x$ is continuous (even Lipschitz continuous). To see this, we can write:
\[ \norm{e_x(f + h) - e_x(f)} = \norm{e_x(h)} \le M \norm{h} \qtxtq{for some constant $M$} \]
Letting $\norm{h} \to 0$, we have the continuity of $e_x$. 

Importantly, by the well-known Riesz Representation Theorem, each evaluation functional $e_x: \H \to \R$ naturally corresponds to a function $k_x \in \H$. We call $k_x$ the \textit{kernel function} of $x$, or the kernel function centered at $x$. 

\begin{theorem}[Riesz Representation Theorem]
If $\phi$ is a bounded linear functional on a Hilbert space $\H$, then there is a unique $g \in\H$ such that 
\[ \phi(x) = \br g,f \kt \]
for all $f \in \H$.
\end{theorem}
    
\begin{corollary}
Let $\H$ be a RKHS on $\X$. For every $x \in \X$, there exists a unique $k_x \in \H$ such that 
\[ \br k_x, f \kt = f(x) \]
for all $f \in \H$. 
\end{corollary}

The kernel function of $x$ is ``reproducing'' in the sense that its inner product with a function $f$ reproduces the value of $f$ at $x$.

\begin{definition}[Reproducing Kernel]
The function $K: \H \times \H \to \R$ defined by 
    \[ K(x,y) = k_y(x) \]
is called the \textit{reproducing kernel} of $\H$. 
\end{definition}

The kernel $K$ is symmetric, as the inner product is symmetric: 
\[ K(x, y) = k_y(x) = \br k_y, k_x \kt = \br k_x, k_y \kt = k_x(y) = K(y, x) \]
If we were working in a complex vector space, the kernel would have conjugate symmetry. 

\begin{theorem}[Equivalence Between Kernels and RKHS]
Every RKHS has a unique reproducing kernel, and every reproducing kernel induces a unique RKHS. 
\end{theorem}
\begin{proof}
We have already seen by means of the Riesz Representation Theorem that every RKHS induces a unique kernel. The converse is a consequence of the Cauchy-Schwartz inequality, which states $\br x,y\kt \le \norm{x}\norm{y}$. If $K$ is a reproducing kernel on a Hilbert space $\fH$, then
\[ e_x(f) = \br k_x, f \kt \le \norm{k_x} \norm{f} = \sqrt{K(x,x)} \cdot \norm{f} \]
so $e_x$ is bounded, and $\fH$ is an RKHS. 
\end{proof}
The existence of a reproducing kernel is sometimes called the \textit{reproducing kernel property}. 

We note that although our original definition of an RKHS involved its evaluation functionals, it turns out to be much easier to think about such a space in terms of its kernel function than its evaluation functionals. 



\subsection{Examples}

We now look at some concrete examples of Reproducing Kernel Hilbert Spaces, building up from simple spaces to more complex ones. 

\paragraph{Example: Linear Functions in $\R^d$} We begin with the simplest of all Reproducing Kernel Hilbert Spaces, Euclidean spaces. Consider $\fH = \R^d$ with the canonical basis vectors $e_1, \dots, e_d$ and the standard inner product: 
\[ \br x, w \kt = \sum_{i=1}^{n} x_i w_i \]
With the notation above, $\X$ is the discrete set $\{1, \dots, d\}$, and $e_i \in \fH$ is the kernel function 
\[ \br e_i, x \kt = x(i) = x_i \]
The reproducing kernel $K: \R^d \times \R^d \to \R$ is simply the identity matrix 
\[ K(i,j) = \br e_i, e_j \kt = \ind{i == j} \]
so that for any $x,x' \in \R^d$, we have 
\[ K(x,x') = \br x,x' \kt \]
In general, for any discrete set $\X$, the Hilbert space $\LT(\X) = \{f \in \R^{\X}:$ $\sum_{x}|f(x)|^2 < \infty \}$ of square-summable functions has a RKHS structure induced by the orthonormal basis vectors $e_y(x) = \ind\{x=y\}$. 

\paragraph{Example: Feature Maps in $\R^p$} We can extend the previous example by considering a set of linearly independent maps $D = \{\phi_i\}_{i=1}^{p}$ for $\phi_i: \X \to \R$. Let $\H$ be the span: 
\[ \H = \vspan\{D\} = \{f: X \to \R: f(x) = \sum_{i=1}^{p} w_i \phi_i(x) \text{ for some } w \in \R^p \} \]
The maps $\phi_i$ are called \textit{feature maps} in the machine learning community. 

We define the inner product on $\H$ by
\[ \br x,x' \kt_{\fH} = \br \phi(x), \phi(x') \kt_{\R^p} = \sum_{i=1}^p \phi_i(x) \phi_i(x') \]
and the kernel $K: \X \times \X \to \R$ is simply  
\[ K(x, x') = \br \phi(x),\phi(x') \kt_{\R^p}  \]
Linear functions correspond to the case where $\X = \{1, \dots, d\}$, $p = d$, and $\phi_i(x) = x_i$. 

\paragraph{Example: Polynomials}

One of the most common examples of feature maps are the polynomials of degree at most $s$ in $\R^d$. For example, for $s = 2$ and $d = 2$,
\[ \phi(x) = (1, \sqrt{2} x_1, \sqrt{2} x_2, \sqrt{2} x_1 x_2, x_1^2, x_2^2)  \]
with corresponding polynomial kernel
\begin{align*}
K(x,x') &= 1 + 2 x_1 x_1' + 2 x_2 x_2' + 2 x_1 x_2 x_1' x_2' + x_1^2 x_1'^2 + x_2^2 x_2'^2 \\
&= (1 + \br x, x'\kt)^2
\end{align*}
In general, the RKHS of polynomials of degree at most $s$ in $\R^d$ has kernel $(1 + \br x, x'\kt)^s$ and is a space of degree $\binom{s+d}{d}$.

\paragraph{Example: Paley-Wiener spaces} The Paley-Wiener spaces are a classical example of a RKHS with a \textit{translation invariant} kernel, which is to say a kernel of the form $K(x, x') = K'(\norm{x - x'})$ for some function $K'$. Paley-Wiener spaces are ubiquitous in signal processing, where translation invariance is a highly desirable property. 

Since we are interested in translation-invariance, it is natural to work in frequency space. Recall the Fourier transform:
\[ \hf(\xi) = \int_{-\infty}^{\infty} f(x) e^{-2\pi i x \xi}\,dt \]
Consider functions with limited frequencies, which is to say those whose Fourier transforms are supported on a compact region $[-A,A]$. Define the Paley-Wiener space $PW_A$ as 
\[ PW_A = \{ \hf: f \in \LT([-A, A]) \} \]
where $\LT$ refers to square-integrable functions. 

We can endow $PW_A$ with the structure of an RKHS by showing that it is isomorphic (as a Hilbert space) to $\LT([-A,A])$. By the definition of $PW_A$, for every $\hf \in PW_A$, there exists an $f \in \LT[(-A,A)]$ such that
\[ \hf(\xi) = \int_{-\infty}^{\infty} f(x) e^{-2\pi i x \xi}\,dx = \int_{-A}^{A} f(x) e^{-2\pi i x \xi}\,dx \]
We claim that this transformation, viewed as a map $\LT([-A, A]) \to PW_A$, is an isomorphism. It is clearly linear, so we need to show that it is bijective. 

To show bijectivity, note that the functions $\{ x \mapsto e^{2\pi i n x / A} \}_{n\in\mN}$ form a basis for $\LT([-A,A])$. Then if $\hf(n/A) = 0$ for every $n \in \mN$, we have $f = 0$ almost everywhere, and vice-versa. Therefore $\LT([-A, A])$ and $PW_A$ are isomorphic.

We can now give $PW_A$ the inner product 
\[ \br \hf_1,\hf_2 \kt_{PW_A} = \br f_1,f_2 \kt_{L_2} = \int_{-A}^A f_1(x) f_2(x) \,dx  \]
Since for any $\hf \in PW_A$, 
\[ |\hf(x)| = \left| \br f, e^{2\pi i x \xi} \kt_{\Lt} \right| \le \norm{e^{2\pi i x \xi}}_{\Lt} \norm{f}_{\Lt} = \sqrt{2A} \norm{\hf} \]
so the evaluation functionals $f \mapsto f(x)$ are bounded, and $PW_A$ is an RKHS. 

To obtain the kernel, we can use the fact that 
\[ \br \hf, \reallywidehat{k_y} \kt_{\Lt} = \br f,k_y \kt_{PW_A} = f(y) = \br \hf, e^{2\pi i y t} \kt_{\Lt} \]
which gives by the inverse Fourier transform that $k_y(x) = \reallywidehat{e^{2\pi i y \xi}}(x)$. Computing this integral gives the kernel: 
\begin{align*}
K(x,y) &= k_y(x) = \int_{-A}^A e^{2\pi i (x - y)\xi}\,d\xi \\
&= \begin{cases} 2A & x = y \\ \sin(2\pi A (x-y))/(\pi(x-y)) & x \ne y \end{cases}
\end{align*}
This kernel is a transformation of the $\text{sinc}$ function, defined as:
\[ \text{sinc}(x) = \begin{cases} 1 & x = 0 \\ \sin(x)/x & x \ne 0 \end{cases}, \qquad K(x,y) = 2A \text{sinc}(2A\pi(x-y)) \]

\paragraph{Example: Sobolev Spaces} Sobolev spaces are spaces of absolutely continuous functions that arise throughout real and complex analysis. 

A function $f: [0,1] \to \R$ is \textit{absolutely continuous} if for every $\ep > 0$ there exists $\da  > 0$ such that, if a finite sequence of pairwise disjoint sub-intervals $\{(x_k, y_k)\} \subset [0,1]$ satisfies $\sum_{k} y_k - x_k < \ep$, then $\sum_{k} |f(y_k) - f(x_k)| < \da$. 

Intuitively, absolutely continuous functions are those that satisfy the fundamental theorem of calculus. Indeed, the fundamental theorem of Lebesgue integral calculus states that the following are equivalent:
\begin{enumerate}
    \item $f$ is absolutely continuous
    \item $f$ has a derivative almost everywhere and $f(x) = f(a) + \int_a^x f'(t) dt$ for all $x \in [a,b]$. 
\end{enumerate}
Let $\fH$ be the set of absolutely continuous functions with square-integrable derivatives that are $0$ at $0$ and $1$: 
\[ \fH = \{f : f' \in \LT([0,1]), \, f(0) = f(1) = 0, \, f \text{ absolutely continuous} \} \]
We endow $\fH$ with the inner product 
\[ \br f, g \kt = \int_0^1 f'(x) g'(x) dx \]
We see that the values of functions in $\fH$ are bounded
\begin{align*}
|f(x)| 
&= \int_0^x f'(t)\,dt = \int_0^1 f'(t)\ind\{t < x\}\,dt \\
&\le  \left(\int_0^1 f'(t)^2\,dt \right)^{1/2} \left(\int_0^1 \ind\{t < x\} \,dt \right)^{1/2} = \norm{f} \sqrt{x}
\end{align*}
so the evaluation functionals are bounded. It is simple to show that with this inner product, the space $\fH$ is complete, so $\fH$ is an RKHS.

We now compute the kernel $k_x$ in a manner that is non-rigorous, but could be made rigorous with additional formalisms. We begin by integrating by parts:
\begin{align*}
f(x) 
&= \br f, k_x \kt = \int_0^1 f'(t) k_x'(t)\,dt = f(t) k_x'(t)|_0^1 - \int_0^1 f(t) k_x''(t)\,dt \\
&= -f(t) k_x''(t)\,dt 
\end{align*}
We see that if $k_x$ were to satisfy
\[ -k_x''(t) = \da_x(t), \quad k_x(0)= 0, \quad k_x(1) = 0 \]
where $\da_{x}$ is the Dirac delta function, it would be a reproducing kernel. Such a function is called the Green's function, and it gives us the solution: 
\[ k_x(t) = K(t, x) = \begin{cases}
    (1-x)t & t \le x \\ (1-t)x & x \le t
\end{cases}\]
It is now easy to verify that
\begin{align*}
\br f, k_x \kt 
&= \int_0^1 f'(t)k_x'(t)\,dt \\
&= \int_0^x f'(t)(1-x)\,dt + \int_x^1 f'(t)(-x)\,dt \\
&= f(x)
\end{align*}

\paragraph{An Example from Stochastic Calculus}

In the above example, we considered a function $f$ on $[0,1]$ with a square-integrable derivative $f'$ and fixed the value of $f$ to $0$ and $t = 0,1$. We found that the kernel $K(x,t)$ is given by $x(1-t)$ for $x < t$. 

If the reader is familiar with stochastic calculus, this description might sound familiar. In particular, it resembles the definition of a Brownian bridge. This is a stochastic process $X_t$ whose distribution equals that of Brownian motion conditional on $X_0 = X_1 = 0$. Its covariance function is given by $\cov(X_s, X_t) = s(1-t)$ for $s < t$. 

Now consider the space $\fH$ of functions for which we only require $f(0) = 0$: 
\[ \fH = \{f : f' \in \LT([0,1]), \,f(0) = 0, \,f \text{ absolutely continuous} \} \]
If the previous example resembled a Brownian bridge, this example resembles Brownian motion. Indeed, by a similar procedure to the example above, one can show that the kernel function of $\fH$ is given by
\[ K(x,t) = \min(s,t) \]
which matches the covariance $\cov(B_s, B_t) = \min(s,t)$ of Brownian motion. 

This remarkable connection is no coincidence. Given a stochastic process $X_t$ with covariance function $R$, it is possible to define a Hilbert space $\fH$ generated by this $X_t$. A fundamental theorem due to Loeve \cite{loeve1977probability} states that this Hilbert space is congruent to the Reproducing Kernel Hilbert space with kernel $R$. 

\paragraph{Example: The Sobolev Space $H^1$} Consider the space 
\[ \fH = H^1 = \{f : f \in \LT(\R), f' \in \LT(\R), f \text{ absolutely continuous} \} \] 
endowed with the inner product 
\[ \br f,g \kt = \frac{1}{2} \int_{-\infty}^\infty f(t) g(t) + f'(t) g'(t) \,dt \]
which induces the norm 
\[ \norm{f}^2_{H^1} = \frac{1}{2} \l \norm{f}^2_{\fL^2} + \norm{f'}^2_{\fL^2} \r \]
The resulting RKHS $H^1$, another example of a Sobolev space, may be understood in a number of ways. 

From the perspective of the Paley-Wiener spaces example, it is a translation-invariant kernel best viewed in Fourier space. One can use Fourier transforms to show that $K(x, y) = \kappa(|x - y|)$, where 
$\hat{\kappa}(\xi) = \frac{2}{1 + \xi}$. Then an inverse Fourier transform shows $K$ is given by
\[ K(x, y) = \frac{1}{2} e^{-|x-y|}  \]
From the perspective of stochastic calculus, this space corresponds to the Ornstein–Uhlenbeck process
\[ dX_t = -\theta \, X_t \, dt + \sigma \, dB_t  \]
which is square-continuous but not square-integrable. The kernel function of $\fH$ corresponds to the covariance function of the OU process:\footnote{Technically, an OU process with an initial condition drawn from a stationary distribution, or equivalently the limit of an OU process away from a strict boundary condition.}
\[ K(s, t) \propto \cov(B_s, B_t) = \frac{\sigma^2}{2\theta} e^{-\theta|s-t|} \]
Finally, we note that we can generalize this example. For any $\ga > 0$, the kernel 
\[ K(x, y) = \frac{1}{2} e^{-\ga|x-y|} \]
is called the \textit{exponential kernel}, and corresponds to the norm 
\[ \norm{f}^2_{\fH} = \frac{1}{2\ga} \l \norm{f}^2_{\fL^2} + \norm{f'}^2_{\fL^2} \r \]


\subsection{Structure}

Thus far, we have defined an RKHS as a Hilbert space with the reproducing property and given a number of examples of such spaces. However, it is not yet clear \textit{why} we need the reproducing property. Indeed, all of the examples above could have been presented simply as Hilbert spaces with inner products, rather than as RKHSs with kernels. 

The best way of conveying the importance of the reproducing property would be to give an example of a Hilbert space that is not an RKHS and show that it is badly behaved. However, explicitly constructing such an example is impossible. It is equivalent to giving an example of an unbounded linear functional, which can only be done (non-constructively) using the Axiom of Choice.

One commonly and incorrectly cited example of a Hilbert space that is not an RKHS is $\LT(\Om)$, the space of square-integrable functions on a domain $\Om$. This example is not valid because $\LT$ is technically not a set of functions, but rather a set of equivalence classes of functions that differ on sets of measure $0$. Whereas $\LT$ spaces are not concerned with the values of functions on individual points (only on sets of positive measure), Reproducing Kernel Hilbert Spaces are very much concerned with the values of functions on individual points.\footnote{The reader is encouraged to go back and check that all of the examples above (particularly Paley-Wiener spaces) are defined in terms of functions that are well-defined pointwise, rather than equivalence classes of functions.} In this sense, RKHSs behave quite differently from $\LT$ spaces.

\paragraph{Anti-Example} This example illustrates the idea that the norm in $\LT$ does not control the function pointwise. Consider a sequence $f_n \in \LT([0,1])$ defined by 
\[ f_n(x) = \begin{cases} 1 & \frac{1}{2} - \frac{1}{n} \le x \le \frac{1}{2} + \frac{1}{n} \\ 0 & \otherwise
\end{cases}\]
As $n \to \infty$, it converges in $\LT$ norm to the $0$ function. However, its value at $1/2$ is always $f(1/2) = 1$. This is to say, there exist functions with arbitrarily small norm and unbounded values at individual points. 

The purpose of the reproducing property of an RKHS is to prevent this type of behavior. 

\begin{theorem}
Let $\fH$ be an RKHS on $\X$. If $\lim_{n\to\infty} \norm{f_n - f} = 0$, then $\lim_{n\to\infty} f_n(x) = f(x)$ for all $x \in X$. 
\end{theorem}
\begin{proof}
By the existence of reproducing kernels and Cauchy-Schwartz,
\[ |f_n(x) - f(x)| = |(f_n - f)(x)| = |\br f_n - f, k_x \kt| \le \norm{f_n - f}\norm{k_x}  \]
so $\lim_{n\to\infty}|f_n(x) - f(x)| = 0$.
\end{proof}

We may also express $K$ pointwise in terms of the basis of the underlying Hilbert space. 

\begin{theorem}
Denote by $\{e_s\}_{s \in S}$ a basis for the RKHS $\H$. Then 
\[ K(x,y) = \sum_{s\in S} e_s(x) e_s(y) \]
where convergence is pointwise.
\end{theorem}
\begin{proof}
By the reproducing property,
\[ k_y = \sum_{s \in S} \br k_y, e_s \kt e_s = \sum_{s \in S} e_s(y) e_s \]
where the sum converges in norm, and so converges pointwise. Then 
\[ K(x,y) = k_y(x) =  \sum_{s \in S} e_s(y) e_s(x) \]
\end{proof}


\section{Kernels, Positive Functions, and Feature Maps}

\begin{figure}[]
    \centering
    \includegraphics[width=0.66\textwidth]{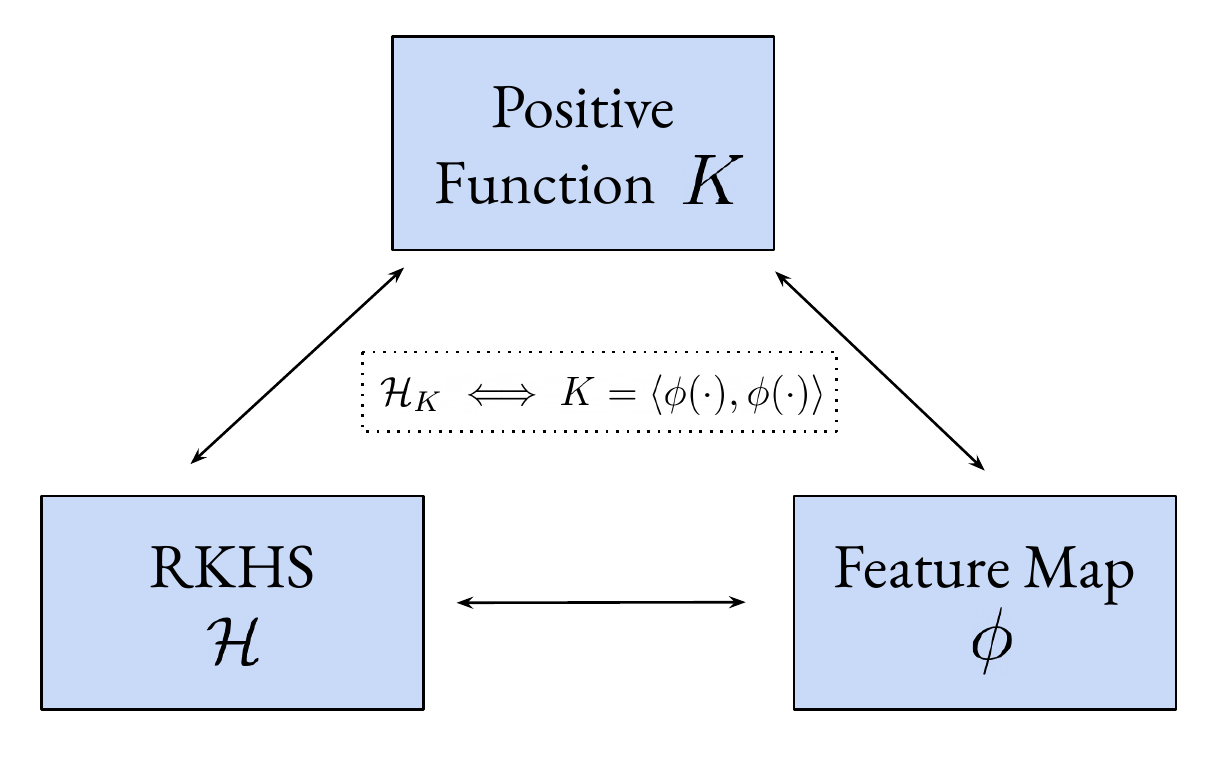}
    \caption[Kernels, Positive Functions, and Feature Maps]{An illustration of the equivalence between kernels, positive functions, and inner products of feature maps.}
    \label{rkhs_equiv}
\end{figure}

At this point, we are ready to fully characterize the set of kernel functions. 

\begin{definition}[Positive Function]
Let $X$ be an arbitrary set. A symmetric function $K: X \times X \to $ is a \textit{positive} function if for any $n$ points $\{x_1, \dots, x_n\}$ in $\X$, the matrix $(K)_{ij} = K(x_i, x_j)$ is positive semidefinite. Equivalently, for any $c_1, \dots, c_n$ in $\R$, we have 
\[ \sum_{i=1}^n \sum_{i=1}^n c_i c_j K(x_i,x_j) \]
\end{definition}

\textit{Note:} Positive functions are sometimes also called positive definite, positive semidefinite, nonnegative, or semipositive. We will use the term \textit{positive} to mean $\ge 0$, and the term \textit{strictly positive} to mean $> 0$. 

We now prove that there is a one-to-one correspondence between kernels and positive functions. 
\begin{theorem}
If $K = \br\cdot,\cdot\kt$ is the kernel of an RKHS $\H$, it is a positive function. 
\end{theorem}
\begin{proof}
First note that $K$ is symmetric, as the inner product on $\H$ is symmetric. Second, we compute 
\[ \sum_{i,j=1}^n c_i c_j K(x_i,x_j) = \br \sum_{i=1}^n c_i x_i, \sum_{i=1}^n c_i x_i \kt = \ns{\sum_{i=1}^n c_i x_i} \ge 0 \]
\end{proof}

The reverse direction is a celebrated theorem attributed to Moore.

\begin{theorem}[Moore-Aronszajn Theorem]
Let $X$ be a set and suppose $K: \X \times \X \to \R$ is a positive function. Then there is a unique Hilbert space $\H$ of functions on $\X \to \R$ for which $K$ is a reproducing kernel.
\end{theorem}
\begin{proof}
Define $k_y$ by $k_y(x) = K(x, y)$. Note that if $K$ were the kernel of an RKHS $\fH$, then the span of the set $\{k_y\}_{y\in X}$ would be dense in $\H$, because if $\br k_y, f \kt = 0$ for all $y \in \X$, then $f(y) = 0$ for all $x\in\X$. 

With this motivation, define $V$ to be the vector space spanned by $\{k_y\}_{y\in X}$. Define the bilinear form $\br\cdot,\cdot\kt$ on $V$ by 
\[ \br\sum_i c_i k_{y_i}, \sum_i c_i' k_{y_i}\kt = \sum_{i,j} c_i c_j' K(y_i, y_j) \]
We aim to show that $\br\cdot,\cdot\kt$ is an inner product. It is positive-definite, bilinear, and symmetric by the properties of $K$, so it remains to be shown that it is well defined. To do so, we need to check $f = 0 \iff \br f,g \kt = 0$ for all $g \in V$. 

($\Longrightarrow$) If $\br f,g \kt = 0$ for all $g \in V$, letting $g = k_y$ we see $\br f,g \kt = f(y) = 0$ for all $y \in \X$. Therefore $f = 0$. 

($\Longleftarrow$) If $f = 0$, $\br f,k_y \kt = \sum_{i} c_i K(x_i, y) = f(y) = 0$. Since the $k_y$ span $V$, each $g \in V$ may be expressed as a linear combination of the $k_y$, and $\br f,g \kt = 0$ for all $g \in V$. 

Therefore $\br\cdot,\cdot\kt$ is well-defined and is an inner product on $V$. Moreover, we may produce the completion $\fG$ of $V$ by considering Cauchy sequences with respect to the norm induced by this inner product. Note that $\fG$ is a Hilbert space. 

All that remains is to identify a bijection between $\fG$ and the set of functions $\X \to \R$. Note that this is where an $\LT$ space fails to be an RKHS. Let $\H$ be the set of functions of the form $\ol{f}(x) = \br f, k_x \kt$, such that 
\[ \H = \{ \ol{f} : f \in \fG \}\]
and observe that elements of $\H$ are functions $X \to \R$. We see that if $\ol{f} = 0$, then $\br f,k_x \kt = 0$ for all $x \in X$, so $h = 0$. Therefore the mapping $f \mapsto \ol{f}$ is linear (by the properties of the inner product) and one-to-one. Thus, the space $\H$ with the inner product $\br \ol{f},\ol{g} \kt_{\H} = \br f,g \kt_{\fG}$ is a Hilbert space with the reproducing kernels $\ol{k_x}$ for $x \in \X$. This is our desired RKHS. 
\end{proof}

There is one final piece in the RKHS puzzle, the concept of feature spaces. 

Let $\X$ be a set. Given a Hilbert space $\mF$, not necessarily composed of functions $\X \to \R$, a \textit{feature map} is a function $\phi: \X \to \mF$. In machine learning, $\X$ and $\mF$ are usually called the \textit{data space} and the \textit{feature space}, respectively. Above, we saw this example in the case $\H = \R^p$. Now $\phi$ may take values in an infinite-dimensional Hilbert space, but the idea remains exactly the same. 

Given a feature map $\phi$, we construct the kernel given by the inner product
\[ K(\cdot,\cdot) = \br\phi(\cdot),\phi(\cdot)\kt \]
or equivalently $\phi(x) = k_x$. As shown above, this kernel defines an RKHS on $\X$. 

Conversely, every kernel $K$ may be written as an inner product $\br\phi(\cdot),\phi(\cdot)\kt$ for some feature map $\phi$. In other words, the following diagram commutes: 
\begin{center}
    \begin{tikzcd}
    \X \times \X \arrow[dd, "\phi"'] \arrow[rrdd, "K"]     &  &            \\
                                                            &  &            \\
    \mF \times \mF \arrow[rr, "{\langle\cdot,\cdot\rangle}"] &  & \mathbb{R}
    \end{tikzcd}
\end{center}

\noindent We note that the Hilbert space $\mF$ and feature map $\phi$ above are not unique. However, the resulting Reproducing Kernel Hilbert Space, composed of functions $\X \to \R$, is unique. In other words, although a feature map specifies a unique RKHS, a single RKHS may have possible feature map representations. 

\begin{theorem}
A function $K: \X \times \X \to \R$ is positive if and only if it may be written as $\br\phi(\cdot),\phi(\cdot)\kt$ for some Hilbert space $\mF$ and some map $\phi: \X \to \mF$. 
\end{theorem}
\begin{proof}
We give a proof for finite-dimensional Hilbert spaces. It may be extended to the infinite-dimensional case with spectral operator theory, but we will not give all the details here. 

First, suppose $K = \br\phi(\cdot),\phi(\cdot)\kt_{\mF}$. Then for $v \in \mF$,
\[ \br v,Kv \kt = \sum_{i=1}^n v_i \sum_{i=1}^n \br \phi(x_i), \phi(x_j) \kt v_j = \br \sum_{i=1}^n v_i \phi(x_i), \sum_{i=1}^n v_i \phi(x_i) \kt \ge 0 \]
so $K$ is positive definite. 

Second, suppose $K$ is positive. Decompose it into $K = V \Lambda V^T$ by the spectral theorem, and let $\phi(x) = \Lambda^{1/2}V^T\ind_x$. Then we have 
\[ \br\phi(x),\phi(x')\kt_{\mF} = \br\ind_x,\ind_{x'}\kt_K = K(x,x') \]
so $K = \br\phi(\cdot),\phi(\cdot)\kt_{\mF}$. 
\end{proof}

We now have a full picture of the relationship between Reproducing Kernel Hilbert Spaces, positive-definite functions, and feature maps. 

\subsection{Geometry}

One way to think of an infinite-dimensional RKHS is as a map $x \mapsto k_x$ that sends every point in $X$ to a point $k_x: \X \to \R$ in an infinite-dimensional feature space. 

The kernel function $K$ defines the geometry of the infinite-dimensional feature space. 

\paragraph{Example: Gaussian Kernel} Let $X = \R^d$ and consider the Gaussian kernel, perhaps the most widely used kernel in machine learning: 
\[ K(x,x') = e^{-\frac{1}{2}\ns{x-x'}} \]
The kernel function $k_x$ corresponding to a point $x$ is a Gaussian centered at $x$. Due to its radial symmetry, this kernel is also called the \textit{radial basis function} (RBF) kernel. 

It turns out that explicitly constructing the RKHS for the Gaussian kernel is challenging (it was only given by \cite{xu2006explicit} in 2006). However, since it is not difficult to show that $K$ is a positive function, we can be sure that such an RKHS exists. 

Let us look at its geometry. We see that each point $x\in\X$ is mapped to a point $k_x$ with unit length, as $\ns{k_x} = K(x,x) = 1$. The distance between two points $k_x, k_y$ is: 
\begin{align*}
\ns{k_x-k_y} &= K(x-y,x-y) = K(x,x) - 2K(x,y) + K(y,y) \\
    &= 2 \left(1 - e^{-\frac{1}{2}\ns{x-y}} \right) < 2
\end{align*}
so any two points are no more than $\sqrt{2}$ apart. 

\paragraph{Example: Min Kernel} Consider the kernel $K(s,t) = \min(s,t)$ for $s,t\in\R$. This kernel induces a squared distance
\begin{align*}
d_K(s,t)^2 &= K(s,s) - 2K(s,t) + K(t,t) \\ &= s + t - 2 \min(s,t) \\ &= \max(s,t) - \min(s,t) \\
&= |s - t|
\end{align*}
the square root of the standard squared Euclidean distance on $\R$. 

In general, so long as the map $x \mapsto k_x$ is unique, the function
\[ d_K(x,y) = \sqrt{K(x-y,x-y)} = \sqrt{K(x,x) - 2K(x,y) + K(y,y)} \]
is a valid distance metric on $\fH$. In this sense, the kernel defines the similarity between two points $x$ and $y$. From a feature map perspective, the distance is 
\[ d_K(x,y) = \norm{\phi(x) - \phi(y)} \]
This metric enables us to understand the geometry of spaces that, like the RKHS for the Gaussian Kernel, are difficult to write down explicitly. 

\subsection{Integral Operators} \label{ssec:integral_ops}

We now take a brief detour to discuss the relationship between kernels and integral operators. This connection will prove useful in \autoref{chap:manireg}. 

We say that a kernel $K: \X \times \X \to \R$ is a \textit{Mercer kernel} if it is continuous and integrable.\footnote{The notation used throughout the literature is not consistent. It is common to see ``Mercer kernel'' used interchangeably with ``kernel''. In practice, nearly every kernel of interest is a Mercer kernel.} That is, $K \in \LT(\X \times \X)$, meaning $\int_{X} \int_{X} K(x,x')\,dx\,dx' < \infty$. 

Suppose that $X$ is compact and define the integral operator $I_K: \LT(\X) \to \LT(\X)$ by 
\[ I_K f(x) = \int_{\X} K(x, x') f(x')\,dx'  \]
It is not difficult to show that $I_K$ is linear, continuous, compact, self-adjoint, and positive. Linearity follows from the linearity of integrals, continuity from Cauchy-Schwartz, compactness from an application of the Arzelà–Ascoli theorem, self-adjointness from an application of Fubini’s theorem, and positivity from the fact that the integral $f I_k f$ is a limit of finite sums of the form $\sum_{i,j} f(x_i) K(x_i,x_j) f(x_j) \ge 0$. 

Since $I_K$ is a compact, positive operator, the spectral theorem states that there exists a basis of $\LT(X)$ composed of eigenfunctions of $I_K$. Denote these eigenfunctions and their corresponding eigenvalues by $\{\phi_i\}_{i=1}^{\infty}$ and $\{\lam_i\}_{i=1}^{\infty}$, respectively. Mercer's theorem states that one can decompose $K$ in this basis: 

\begin{theorem}[Mercer]
\[ K(x,y) = \int_{i=1}^{\infty} \lam_i \phi_i(x)\phi_i(y) \]
where the convergence is absolute and uniform over $\X \times \X$. 
\end{theorem}
This theorem is not challenging to prove, but it requires building significant machinery that would not be of further use. We direct the interested reader to \cite{riesz1990functional} (Section 98) for a detailed proof. 

\section{Tikhonov Regularization and the Representer Theorem}


Having built our mathematical toolkit, we return now to machine learning. Our goal is to minimize the regularized empirical risk $\fhE(f(x), y) + \lam R(f, x, y)$ over a space $\fH$. 

Let $\fH$ be an RKHS, as we are concerned with the values of functions pointwise. Let $R$ be the norm $\ns{f}_K = K(f,f)$, as its purpose is to measure the complexity of a function. 

Denote our data by $S = \{(x_i, y_i)\}_{i=1}^N$, and let $\fhE(f, x, y)$ be the sum of a loss function $L(f(x_i), y_i)$ over the data. Our learning problem is then
\begin{equation} \label{eq:learning_problem_rkhs}
    \arg\min_{f \in \H } \frac{1}{N} \sum_{i=1}^N L(f(x_i), y_i) + \lam \ns{f}_K  
\end{equation}
where $\lam > 0$. This general framework is known as \textit{Tikhonov regularization}.

The Representer Theorem reduces this infinite-dimensional optimization problem to a finite-dimensional one. It states that our desired solution is a linear combination of the kernel functions on the data points. 

\begin{theorem}[Representer Theorem] \label{thm:representer_theorem}
Let $\H$ be an RKHS on a set $\X$ with kernel $K$. Fix a set of points $S = \{x_1, x_2, \dots, x_N\} \subset \X$. Let 
\[ J(f) = L(f(x_1), \dots, f(x_n)) + R(\ns{f}_{\H}) \]
and consider the optimization problem 
\[ \min_{f\in\H} J(f) \]
where $R$ is nondecreasing. Then if a minimizer exists, there is a minimizer of the form 
\[ f = \sum_{i=1}^N \al_i k_{x_i} \]
where $\al_i \in \R$. Moreover, if $P$ is strictly increasing, every minimizer has this form. 
\end{theorem}
\begin{proof}
The proof is a simple orthogonality argument. 

Consider the subspace $T \subset \H$ spanned by the kernels at the data points: 
\[ T = \vspan\{k_{x_i} : x_i \in S\}\]
Since $S$ is a finite dimensional subspace, so it is closed, and every $f \in \H$ may be uniquely decomposed as $f = f_{T} + f_{\perp}$, where $f_T \in T$ and $f_{\perp} \in T^{\perp}$. 

By the reproducing property, we may write $f(x_i)$ as 
\begin{align*}
f(x_i) 
&= \br f, k_{x_i} \kt = \br f_T, k_{x_i} \kt + \br f_\perp, k_{x_i} \kt = \br f_T, k_{x_i} \kt \\
&= f_T(x_i)
\end{align*}
Also note 
\[ R(\ns{f}) = R(\ns{f_T} + \ns{f_\perp}) \ge R(\ns{f_T})  \]
Then $J(f)$ may be written as 
\begin{align*}
J(f) 
&= L(f(x_1), \dots, f(x_n)) + R(\ns{f}) = L(f_T(x_1), \dots, f_T(x_n)) + R(\ns{f}) \\
&\ge L(f_T(x_1), \dots, f_T(x_n)) + R(\ns{f_T}) \\
&= J(f_T)
\end{align*}
Therefore, if $f$ is a minimizer of $J$, $f_T$ is also a minimizer of $J$, and $f_T$ has the desired form. Furthermore, if $R$ is strictly increasing, the $\ge$ above may be replaced with $>$, so $f$ cannot be a minimizer of $J$ unless $f = f_T$. 
\end{proof}

If $L$ is a convex function, then a minimizer to Equation \ref{eq:learning_problem_rkhs} exists, so by the Representer Theorem it has the form 
\[ f(x) = \sum_{i=1}^N \al_i K(x_i, x) \]
Practically, it converts the learning problem from one of dimension $d$ (that of the RKHS) to one of dimension $N$ (the size of our data set). In particular, it enables us to learn even when $d$ is infinite. 

\section{Algorithms}

With the learning problem now fully specified, we are ready to look at algorithms. 

\subsubsection{Regularized Least Squares Regression}

In regularized least squares regression, we aim to learn a function $f: \X \to \R$ minimizing the empirical risk with the loss function $L(f(x), y) = (f(x) - y)^2$. In other words, the learning problem is: 
\[ \arg\min_{f \in \H } \frac{1}{N} \sum_{i=1}^N (f(x_i) - y_i)^2 + \lam \ns{f}_K \]
where $(x_i, y_i) \in \X \times \R$ are our (training) data. 

By the Representer Theorem, the solution $f$ of this learning problem may be written: 
\[ f(x) = \sum_{i=1}^N \al_i K(x_i, x) \]
We now solve for the parameters $\al = (\al_1, \dots, \al_N)$. 

For ease of notation, we write $x = (x_1, \dots, x_N)$, $y = (y_1, \dots, y_N)$. Denote by $K$ the $N \times N$ kernel matrix on the data, also called the \textit{Gram matrix}: $K = (K_{ij}) (K(x_i,x_j))$. With this notation, we have 
\[ (f(x_1), \dots, f(x_n)) = K \al \qtxtq{and} \ns{f}_K = \al^T K \al \]
so our objective may be written as
\begin{align} \label{eq:al_objective}
\arg\min_{f \in \H } \frac{1}{N} (K\al - y)^T(K\al - y) + \lam \al^T K \al 
\end{align}
To optimize, we differentiate with respect to $\al$, set the result to $0$, and solve: 
\begin{align} \label{eq:diff_wrt_al}
0 = \frac{2}{N} K (K\al^{*} - y)  + 2 \lam K \al^{*} = K((K + \lam N I) \al^{*} - y)
\end{align}
Since $K$ is positive semidefinite, $(K + \lam N I)$ is invertible, and 
\[  \al^{*} = (K + \lam N I)^{-1} y \]
is a solution. Therefore 
\[ f(x) = \sum_{i=1}^N \al_i K(x_i, x) \]
with $\al = (K + \lam N I)^{-1} y$ is a minimizer of the learning problem. 

If $\X = \R^d$ with the canonical inner product, the Gram matrix is simply $K = XX^T$, where $X$ is the $N \times d$ matrix of data. Then $\al^{*}$ becomes $\al^{*} = (XX^T + \lam N I)^{-1} y$ and the minimizer of the learning problem may be written as 
\begin{equation} \label{eq:rls_kernel_form}
    X^T(XX^T + \lam N I)^{-1}y
\end{equation}
A Woodbury matrix identity states that for any matrices $U, V$ of the correct size, $(I + UV)^{-1} = I - U(I + VU)^{-1}V$. The expression above may then be written as
\begin{equation} \label{eq:rls_reg_form}
    (X^TX + \lam N I)^{-1}X^Ty
\end{equation}
which is the familiar solution to a least squares linear regression. 

Comparing Equations \ref{eq:rls_kernel_form} and \ref{eq:rls_reg_form}, we see that the former involves inverting a matrix of size $N \times N$, whereas the latter involves inverting a matrix of size $d \times d$. As a result, if $d > N$, it may be advantageous to use \ref{eq:rls_kernel_form} even for a linear kernel. 

\paragraph{A Note on Uniqueness: } The process above showed that $\al^{*} = (K + \lam N I)^{-1} y$ is a solution to Equation \ref{eq:al_objective}, but not that it is unique. Indeed, if the rank of $K$ is less than $N$, multiple optimal $\al \in \R^d$ may exist. However, the function $f \in \H$ constructed from these $\al$ will be the same. To see this, note that Equation $\ref{eq:diff_wrt_al}$ shows that for any optimal $\al$, we have $\al = (K + \lam N I)^{-1} - y + \da$, where $K \da = 0$. Therefore for any two optimal $\al, \al'$ we have
\[ \ns{f - f'} = (\al -\al')^T K (\al - \al') = 0\]
and so $f = f'$.

\subsubsection{Regularized Logistic Regression}
Regularized logistic regression, which is a binary classification problem, corresponds to the logistic loss function 
\[ \log(1 + e^{-y_i f(x_i)})  \]
where the binary labels $y$ are represented as $\{-1,1\}$. Our objective is then 
\[ \arg\min_{f \in \H } \frac{1}{N} \sum_{i=1}^N \log(1 + e^{-y_i f(x_i)}) + \lam \ns{f}_K \]
Our solution takes the form given by the Representer Theorem, so we need to solve 
\[ \arg\min_{\al \in \R^N} \frac{1}{N} \sum_{i=1}^N \log(1 + e^{-y_i (K\al)_i }) + \lam \al^T K \al \]
for $\al$. Unfortunately, unlike for least squares regression, this equation has no closed form. Fortunately, it is convex, differentiable, and highly amenable to gradient-based optimization techniques (e.g. gradient descent). These optimization methods are not a focus of this thesis, so we will not go into further detail, but we note that they are computationally efficient and widely used in practice. 

\subsubsection{Regularized Support Vector Machines}
Regularized support vector classification, also a binary classification problem, corresponds to the hinge loss function 
\[ L_{sup}(f(x), y) = \max(0, 1 - yf(x)) = (1 - y f(x))_{+} \]
where $y_i \in \{-1,1\}$. As always, our objective is 
\[ \arg\min_{f \in \H } \frac{1}{N} \sum_{i=1}^N \log(1 + e^{-y_i f(x_i)}) + \lam \ns{f}_K \]
and our solution takes the form given by the Representer Theorem. Like with logistic regression, we solve
\[ \arg\min_{\al \in \R^N} \frac{1}{N} \sum_{i=1}^N \log(1 + e^{-y_i (K\al)_i }) + \lam \al^T K \al \]
for $\al$ by computational methods. The one caveat here is that we need to use ``subgradient-based'' optimization techniques rather than gradient-based techniques, as the gradient of the hinge loss is undefined at $0$. 

\subsubsection{The Kernel Trick} 

Suppose we have an algorithm $\fA$ where the data $x_i$ are only used in the form $\br x_i, \cdot \kt$. In this case, we can \textit{kernelize} the algorithm by replacing its inner product with a kernel $K$. This process, known as the \textit{kernel trick}, effectively enables us to work in infinite-dimensional feature spaces using only finite computational resources (i.e. only computing the kernel functions $K$). 


\subsection{Building Kernels}

\begin{table}[ht!]
\def\arraystretch{1.4}
\arrayrulecolor{gray}
\begin{center}
\addtolength{\leftskip} {-0.0cm} 
\addtolength{\rightskip}{-0.0cm}
\rowcolors{1}{white}{gray!15}
\makebox[\textwidth][c]{
\begin{tabular}{|l|c|c|c|} \hline
\textbf{Name}               & \textbf{Periodic}              & \textbf{Kernel}                                                                                                                                & \textbf{Areas of Application}   \\ \hline 
Linear             & \xmark                      & $x^T x'$                                                                                                                              & Ubiquitous                       \\
Polynomial         & \xmark                      & $(c + x^T x')^p$                                                                                                                      & Ubiquitous                       \\
Gaussian           & \cmark & $e^{-\frac{1}{2\si}\ns{x-y}}$                                                                                                         & Ubiquitous                       \\
Exponential        & \cmark & $e^{-\si \norm{x-y}}$                                                                                                                 & Ubiquitous                       \\
Tanh               & \xmark                      & $\tanh(\si x^Tx' + b)$                                                                                                                & Neural networks        \\
Dirichlet          & \cmark & $\frac{\sin\left(\left(n +1/2\right) (x-x') \right)}{2\pi\sin((x-x')/2)}$                                                             & Fourier analysis       \\
Poisson            & \cmark & $\frac{1 - \si^2}{\si^2 - 2 \si \cos(x-x') + 1}$                                                                                      & Laplace equation in 2D \\
Sinc               & \cmark & $\frac{\sin(\si (x-x'))}{(x-x')}$                                                                                                     & Signal processing      \\
Rational Quadratic & \cmark & $\sigma^2 \left( 1 + \frac{(x - x')^2}{2 \alpha \ell^2} \right)^{-\alpha}$                                                            & Gaussian processes     \\
Exp-Sine-Squared   & \cmark & $ \sigma^2\exp\left(-\frac{2\sin^2(\pi|x - x'|/p)}{\ell^2}\right)$                                                                    & Gaussian processes     \\
Matérn Kernel      & \cmark & $\sigma^2\frac{2^{1-\nu}}{\Gamma(\nu)}\Bigg(\sqrt{2\nu}\frac{|x-x'|}{\rho}\Bigg)^\nu K_\nu\Bigg(\sqrt{2\nu}\frac{|x-x'|}{\rho}\Bigg)$ & Gaussian processes   \\ \hline
\end{tabular}
}%
\caption[Common Kernel Functions]{Examples of commonly used kernel functions.}
\label{kernel_table}
\end{center}
\end{table}


In practice, applying kernel methods translates to building kernels that are appropriate for one's specific data and task. Using task-specific kernels, it is possible to encode one's domain knowledge or inductive biases into a learning algorithm. The problem of automatically selecting or building a kernel for a given task is an active area of research known as \textit{automatic kernel selection}.

Although building kernels for specific tasks is outside the scope of this thesis, we give below a few building blocks for kernel construction. Using these building blocks, one can create complex kernels from simpler ones.  

\paragraph{Properties} Let $K, K'$ be kernels on $X$, and let $f$ be a function on $X$. Then the following are all kernels:
\begin{itemize}
    \item $K(x,x') + K'(x,x')$
    \item $K(x,x') \cdot K'(x,x')$
    \item $f(x) K(x,x') f(x')$
    \item $K(f(x),f(x'))$
    \item $\exp(K(x,x'))$
    \item $\frac{K(x,x')}{\sqrt{K(x,x)}\sqrt{K(x',x')}}$, called the \textit{normalized} version of $K$ 
\end{itemize}

We remark that all these properties may be thought of as properties of positive functions. 

\paragraph{Kernels from Probability Theory} A few interesting kernels arise from probability theory. For events $A,B$, the following are kernels:
\begin{itemize}
    \item $K(A,B) = P(A \cap B)$ is a kernel. 
    \item $K(A,B) = P(A \cap B) - P(A)P(B)$ is a kernel. 
    \item $H(X) + H(X') - H(X, X')$ 
\end{itemize}
At first glance, the mutual information $I(X, X')$ also looks like a kernel, but this turns out be quite tricky to prove or disprove. The problem was only solved in 2012 by Jakobsen \cite{jakobsen2014mutual}, who showed that $I(X, X')$ is a kernel if and only if $\dim(X) \le 3$.

\paragraph{Common Kernels in Machine Learning} Examples of some common kernels are given in \autoref{kernel_table}, and even more examples are available \href{http://crsouza.com/2010/03/17/kernel-functions-for-machine-learning-applications}{at this link}.

%% file: chapters/ch4-graphsandgeo.tex
\chapter{Graphs and Manifolds} \label{chap:graphsandgeo}

We now turn our attention from the topic of Reproducing Kernel Hilbert Spaces to an entirely new topic: the geometry of graphs and Riemannian manifolds. The next and final chapter will combine these two topics to tackle regularized learning problems on graphs and manifolds. 

The purpose of this chapter is to elucidate the connection between graphs and manifolds.
At first glance, these two mathematical objects may not seem so similar. We usually think about graphs in terms of their combinatorial properties, whereas we usually think about manifolds in terms of their topological and geometric properties. 

Looking a little deeper, however, there is a deep relationship between the two objects. We shall see this relationship manifest in the Laplacian operator, which emerges as a natural operator on both graphs and manifolds. The same spectral properties of the Laplacian enable us to understand the combinatorics of graphs and the geometry of manifolds. 

This chapter explores how the two Laplacians encode the structures of their respective objects and how they relate to one another. By the end of the chapter, I hope the reader feels that graphs are discrete versions of manifolds and manifolds are continuous versions of graphs. 

\subsubsection{Related Work \& Outline}

Numerous well-written references exist for spectral graph theory \cite{spielman2007spectral,chung1996lectures} and for analysis on manifolds \cite{canzani2013analysis}, but these topics are usually treated independent from one another.\footnote{The literature on Laplacian-based analysis of manifolds is slightly more sparse the spectral graph theory literature. For the interested reader, I highly recommend \cite{canzani2013analysis}.}
One notable exception is \cite{bolker2002graph}, illustratively titled ``How is a graph like a manifold?''. 
This paper examines a different aspect of the graph-manifold connection from the one examined here; whereas \cite{bolker2002graph} is concerned with group actions on complex manifolds and their connections to graph combinatorics, this chapter is concerned with spectral properties of the Laplacian on both manifolds and graphs. 

Rather than discuss graphs and then manifolds, or vice-versa, we discuss the two topics with a unifying view. Throughout, we highlight the relationship between the Laplacian spectrum and the concept of connectivity of a graph or manifold. 


We assume that the reader is familiar with some introductory differential geometry (i.e. the definition of a manifold), but has not necessarily seen the Laplacian operator on either graphs or manifolds before. 

\section{Smoothness and the Laplacian}

As seen throughout the past two chapters, we are interested in finding smooth functions. On a graph or a manifold, what does it mean to be a smooth function? The Laplacian holds the key to our answer. 

Let $G = (V,E)$ be a connected, undirected graph with edges $E$ and vertices $V$. The edges of the graph can be weighted or unweighted (with nonnegative weights); we will assume it is \textit{unweighted} except where otherwise specified. When discussing weighted graphs, we denote by $w_{ij}$ the weight on the edge between nodes $i$ and $j$. 

A real-valued function on $G$ is a map $f: V \to \R$ defined on the vertices of the graph. Note that these functions are synonymous with vectors, as they are of finite length. 

Intuitively, a function on a graph is smooth if its value at a node is similar to its value at each of the node's neighbors. Using squared difference to measure this, we arrive at the following expression: 
\begin{equation} \label{eq:graph_quadratic_form}
\sum_{(i,j)\in E} (f(i) - f(j))^2 
\end{equation}
This expression is a symmetric quadratic form, so there exists a symmetric matrix $\lapg$ such that 
\[ \bf^T \lapg \bf = \sum_{(i,j)\in E} (f(i) - f(j))^2 \]
where $\bf = (x(1), \dots, x(n))$ for $n = |V|$. 

We call $\lapg$ the Laplacian of the graph $G$. We may think of $\lapg$ as a functional on the graph that quantifies the smoothness of functions. 

The Laplacian of a weighted graph is defined similarly, by means of the following quadratic form: 
\[ \bx^T \lapg \bx = \sum_{(i,j)\in E} w_{ij} (x(i) - x(j))^2 \]

\textit{Notation: } Some texts work with the normalized Laplacian $\nlap$ rather than the standard Laplacian $\lapg$. The normalized Laplacian is given by $D^{-1/2} \lapg D^{-1/2}$, where $D$ is the diagonal matrix of degrees of vertices (i.e. $D_{ii} = \deg(i)$). 

We now turn our attention to manifolds. Let $(\fM, g)$ be a Riemannian manifold of dimension $n$. As a refresher, this means that $\fM$ is a smooth manifold and $g$ is a map that smoothly assigns to each $x \in \fM$ an inner product $\br\cdot,\cdot\kt_{g_x}$ on the tangent space $T_x\fM$ at $x$. For ease of notation, when it is clear we will write $\fM$ in place of $(\fM, g)$ and $g_x(\cdot, \cdot)$ in place of $\br\cdot,\cdot\kt_{g(x)}$. 

Suppose we wish to quantify the smoothness of a function $f: \fM \to \R$ at a point $x \in \fM$. A natural way of doing this would be to look at the squared norm $\ns{\grad f}$ of the gradient of $f$ at $x$. This quantity is analogous to the squared difference between a node's value and the values of its neighbors in the graph case. Informally, if we write $\ns{\grad f}$ as $f \grad \cdot \grad f$, it looks like a quadratic form. As in the graph case, we associate this form with an operator $\lap$.

Formally, we define $\lap$ as the negative divergence of the gradient, written as $\lap = - \grad \cdot \grad$ or $- \text{div}\, \grad$ or $- \grad^2$. We call $\lap$ the Laplacian or Laplace-Beltrami operator on the manifold $\fM$.

\textit{Notation: } Some texts define $\lap$ as $\text{div}\, \grad$, without a negative sign. In these texts, the Laplace-Beltrami operator is negative semidefinite and its eigenvalue equation is written as $\lap f = - \lam f$ rather than $\lap f = \lam f$. Here, we adopt the negated version for simplicity and for consistency with the graph literature, where the Laplacian is positive semidefinite. 

Since $\ns{\grad f(x)}$ describes the smoothness of a function $f$ at $x$, integrating it over the entire manifold gives a notion of the smoothness of $f$ on $\fM$: 
\[ \int_\fM \ns{\grad f(x)}\,dx \]
This quantity (technically $1/2$ of this quantity) is called the \textit{Dirichlet energy} and denoted by $E[f]$. It plays a role analogous to Equation \ref{eq:graph_quadratic_form} on the graph, and occurs throughout physics as a measure of the variability of a function. In fact, the Laplace operator may be thought of as the functional derivative of the Dirichlet energy. 

\subsection{More Definitions and Properties} \label{section:more_definitions}

Readers familiar with graph theory or analysis may have noticed that the definitions given above are not the most common ways to introduce Laplacians on either graphs or manifolds. 

Usually, one defines the Laplacian of a graph $G$ in terms of the adjacency matrix $A$.\footnote{At first glance, the adjacency matrix might seem to be the most natural matrix to associate to a graph. However, for a variety of reasons, the Laplacian in general turns out to be much more connected to the fundamental combinatorial properties of the graph. The one notable exception to this rule is in studying random walks, where the powers and spectrum of the adjacency matrix define the behavior and equilibrium state of the random walk.} The Laplacian is given by
\[\lapg = D - A \]
where $D_{ii} = \deg(i)$ is the diagonal matrix of degrees of nodes. 
The normalized laplacian is then: 
\[ \nlap = I - D^{-1/2}AD^{-1/2} \]
A simple computation shows that these definition and our original one are equivalent: 
\begin{align*}
x^T(D - A)x 
    &= x^T D x + x^T A x \\
    &= \sum_{i=1}^{n}\deg(i) x_i^2 - \sum_{(i,j)\in E} 2 x_i x_j \\
    &= \sum_{i=1}^{n} \sum_{(i,j)\in E} x_i^2 - \sum_{(i,j)\in E} 2 x_i x_j \\
    &= \sum_{(i,j)\in E} (x_i^2 + x_j^2 - 2 x_i x_j) \\
    &= \sum_{(i,j)\in E} (x_i - x_j)^2  \\
    &= x^T \lapg x
\end{align*}
Some basic properties of the Laplacian, although not obvious from the definition $\lapg = D - A$, are obvious given the quadratic form definition. Namely, $\lapg$ is symmetric and positive semi-definite, since for any $x$,
\[ x^T \lapg x = \sum_{(i,j)\in E} (x_i - x_j)^2 \ge 0 \]
As a result, all eigenvalues of $\lapg$ are non-negative. We can also see that the smallest eigenvalue is $0$, corresponding to an eigenfunction that is a (non-zero) constant function.


Turning to manifolds, the Laplacian $\lap$ is also usually introduced in a different manner from the one above. In the context of multivariable calculus, it is often defined as: 
\[ \lap f = - \dfrac{\partial ^2 f}{ \partial x^ 2} - \dfrac{\partial ^ 2 f}{ \partial y^ 2} - \dfrac{\partial ^ 2 f}{ \partial z^ 2} \]
which is easily verified to be equal to $\text{div}\,\grad f$ in $\R^N$. 
This coordinate-wise definition can be extended to the local coordinates of a Riemannian manifold with metric tensor $g$:
\begin{equation} \label{eq:local_coord_lap_def}
\lap = - \frac{1}{\sqrt{|\det g|}}\sum_{i,j=1}^{n}\left(g^{ij}\sqrt{|\det g|}\ppx{}{x_j}\right)
\end{equation}
However, if one would like to work with coordinates on a manifold, it is much more natural to work in the \textit{canonical} local coordinates. To switch to these coordinates, we use the exponential map $\exp_p: T_p\fM (= \R^n) \to \fM$, which is a local diffeomorphism between a neighborhood of a point $p \in \fM$ and a neighborhood of $0$ in the tangent space $T_p\fM$. This coordinate map gives a canonical identification of a neighborhood of $p$ with $\R^N$, called geodesic normal coordinates. In geodesic normal coordinates, 
$g_{ij} = \delta_{ij}$ and $\ppx{g_{ij}}{x_k} = 0$, so the formula for $\lap$ resembles the formula in Euclidean space. 

Finally, we should note that yet another way to define the Laplacian $\lap$ is as the trace of the Hessian operator $H$: 
\[ \lap = \tr(H) \]
where the Hessian $H$ at $p$ is $\grad_p(df)$, the gradient of the differential of $f$. Note that since the Hessian is coordinate-free (i.e. invariant under isometries), this relation shows us that Laplacian is coordinate-free. 

\subsection{Examples} \label{ssec:examples}

Below, we present a few examples of Riemannian manifolds and graphs along with their Laplacians. 

\paragraph{Example: $\R^n$} The most ordinary of all Riemannian manifolds is $\R^n$ with the Euclidean metric $g = \br\cdot,\cdot\kt_{\R^n}$. In matrix form, $g$ is the identity matrix of dimension $n$: $g_{ij} = \delta_{ij}$ and $\det g = 1$. Following formula \ref{eq:local_coord_lap_def}, we have
\[ \lap_{g,\,\R^n} = - \sum_{i=1}^n \ppx{^2}{x_i^2} \]
which is the familiar form of the divergence of the gradient in $\R^n$. 

\paragraph{Example: $S^1$} The simplest nontrivial Riemannian manifold is the circle $S^1 \subset \R^2$ with the metric induced by $\R^2$. We may parameterize the circle as $(\cos(\ta), \sin(\ta))$, with the resulting metric $g = d\ta^2$ (induced from $\R^2$ as $dx^2 + dy^2 = dr^2 + r^2\,d\ta^2 = d\ta^2$). In matrix form, $g$ is simply the $1$-dimensional matrix $(1)$. Consequently, 
\[ \lap_{g,\,S^1} = - \ppx{^2}{\ta^2} \]
as above. A similar result holds for all one-dimensional manifolds. 

\paragraph{Example: Cycle Graph} A simple graph similar to the smooth circle above is the cycle graph. The Laplacian $\lapg$ of a cycle graph $G$ with $n$ vertices is given by: 
\begin{align*}
\hspace*{-70pt} \lapg & = D - A 
    = \left(\begin{array}{ccccc}  
    2 & 0 & 0 & 0 \\  
    0 & \ddots & 0 & 0 \\  
    0 & 0 & 2 & 0 \\  
    0 & 0 & 0 & 2 \\
    \end{array} \right)  
    - \left(\begin{array}{cccccc} 
    0 & 1      &      0 & 1 \\ 
    1 & 0      & \ddots & 0 \\ 
    0 & \ddots &      0 & 1 \\ 
    1 & 0      &      1 & 0 \\\end{array}\right) \\
    & = \left(\begin{array}{cccccc} 2 & -1 & 0 & 0 & 0 & -1 \\-1 & 2 & -1 & 0 & 0 & 0 \\ 0 & -1 & \ddots & \ddots & 0 & 0 \\ 0 & 0 & \ddots & \ddots & -1 & 0 \\ 0 & 0 & 0 & -1 & 2 & -1 \\-1 & 0 & 0 & 0 & -1 & 2 \\\end{array}\right)
\end{align*}
Readers familiar with numerical analysis might note that this matrix resembles the (negated) second-order discrete difference operator
\[ \ppx{^2u}{x^2} \approx - \frac{- u_{i+1} + 2u_i - u_{i-1} }{ \delta x } \]
which suggests a connection to the manifolds above. As we will see later, the Laplacian spectra of the circle and the cycle graph are closely related. 

\paragraph{Example: $S^2$} Consider the $2$-sphere parameterized in spherical coordinates with the metric induced from $\R^3$: 
\[ T: [0, \pi) \times [0, 2\pi) \to S^2 \]
\[ T(\ta, \phi) = (\sin\ta \cos\phi, \sin\ta\sin\phi, \cos\ta) \]
Changing to spherical coordinates shows that the metric is given by 
\[g = dx^2 + dy^2 + dz^2 = (dx^2 + dy^2) + dz^2 = d\ta^2 + \sin^2\ta d\phi\]
so in matrix form $g$ is
\[ g(\ta, \phi) = \begin{pmatrix} 1 & 0 \\ 0 & \sin^2\ta \end{pmatrix} \]
with determinant $\det g = \sin^2\ta$. Then by formula \ref{eq:local_coord_lap_def}, the Laplacian is 
\begin{align*}
\lap 
&= -\frac{1}{\sqrt{\det g}} \left(\ppx{}{\ta}\left(g_{\ta\ta}\sqrt{\det g}\ppx{}{\ta}\right) + \ppx{}{\phi}\left(g_{\phi\phi}\sqrt{\det g}\ppx{}{\phi}\right) \right) \\
&= - \frac{1}{\sin\ta}\frac{}{\ta}\left(\sin\ta\ppx{}{\ta}\right) - \frac{1}{\sin^2\ta}\ppx{^2}{\phi^2}
\end{align*}
This expression enables us to work with the eigenvalue equation $\lap f = \lam f$ in spherical coordinates, a useful tool in electrodynamics and thermodynamics.

\paragraph{Example: More Classic Graphs} Figure \ref{fundamental_graphs} shows the cycle graph and three more classic graphs---the complete graph, path graph, and star graph---alongside their Laplacians. 

\begin{figure}[]
    \centering
    \includegraphics[width=\textwidth]{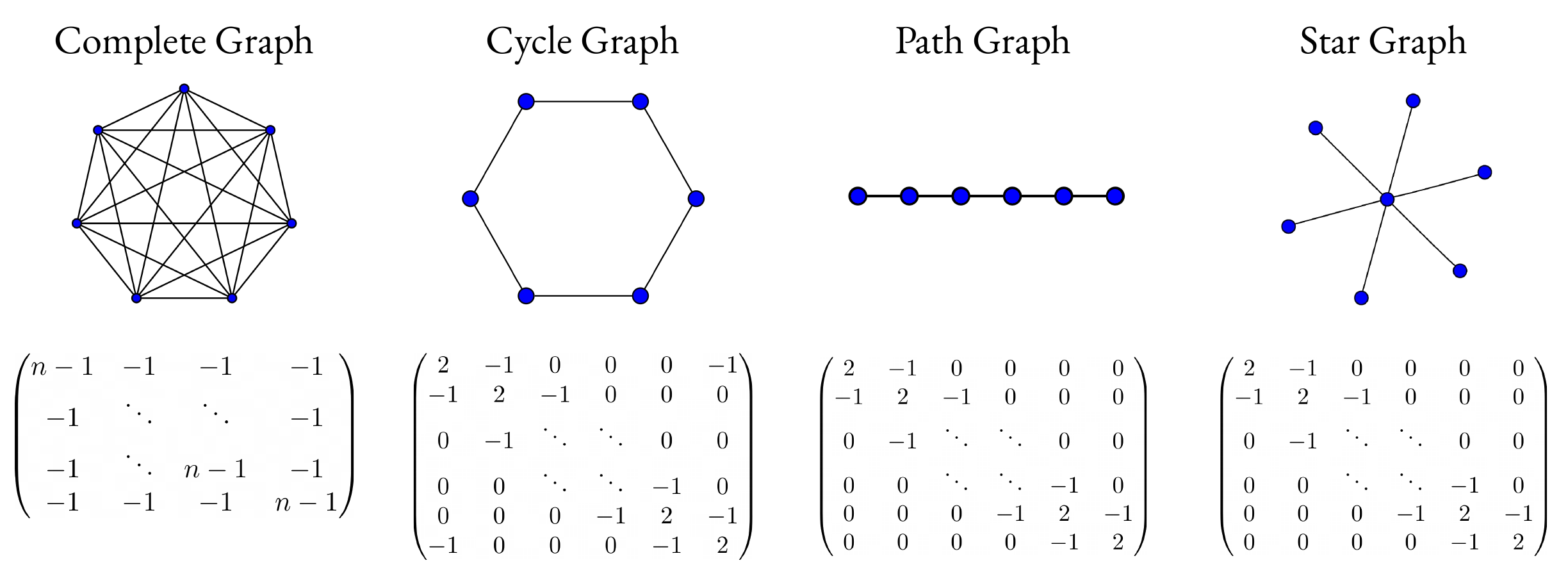}
    \caption[Classic Graphs]{A few classic graphs and their Laplacians.}
    \label{fundamental_graphs}
\end{figure}


\paragraph{Example: Flat Torus} An $n$-dimensional torus is a classic example of a compact Riemannian manifold with genus one, which is to say a single ``hole''. 

Topologically, a torus $\mT$ is the product of spheres, $S^1 \times \cdots \times S^1 = (S^1)^{n}$. Equivalently, a torus may be identified with $\R^n/\Gamma$, where $\Gamma$ is an $n$-dimensional lattice in $\R^n$ (a discrete subgroup of $\R^n$ isomorphic to $\mZ^n$). \footnote{Concretely, $\Gamma$ is the set of linear combinations with integer coefficients of a basis $\{e_1, e_2, \dots, e_n \}$ of $\R^n$.} That is to say, we can identify the torus with a (skewed and stretched) square in $\R^2$ conforming to specific boundary conditions (namely, that opposite sides of the square are the same). We call the torus with $\Gamma = \mZ^n$ the standard torus. 

When endowed with the product metric from $S^1$ (i.e. the $n$-times product of the canonical metric on $S^1$), a torus is called the flat torus.\footnote{In general, a manifold is said to be flat if it has zero curvature at all points. Examples of other spaces commonly endowed with a flat metric include the cylinder, the Möbius band, and the Klein bottle.} As the Laplacian is locally defined by the metric, the Laplacian of any flat surface is the same as the Laplacian in Euclidean space, restricted to functions that are well-defined on the surface.

Intuitively, the flat metric makes the torus look locally like $\R^n$. Among other things, this means that angles and distances work as one would expect in $\R^n$; for example, the interior angles of a triangle on a flat torus add up to $\pi$ degrees. 

\paragraph{Example: Torus Embedded in $\R^3$} The flat metric is not the only metric one can place on a torus. On the contrary, it is natural to picture a torus embedded in $\R^3$, with the familiar shape of a donut (\autoref{plane_to_torus}). The torus endowed with the metric induced from $\R^3$ is a different Riemannian manifold from the flat torus. 

\begin{figure}[]
    \centering
    \includegraphics[width=\textwidth]{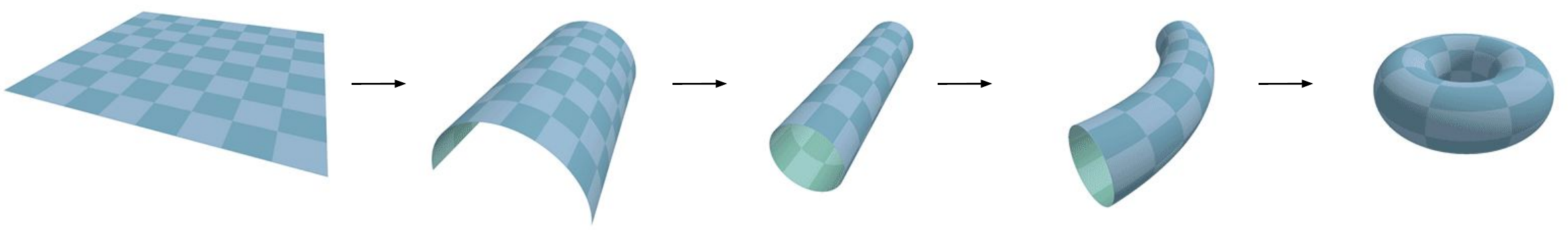}
    \caption[Illustration of Torus]{A fun illustration of how a torus may be created from a square in the plane with periodic boundary conditions.}
    \label{plane_to_torus}
\end{figure}

The torus $\mT$ embedded in $R^3$ with minor radius (i.e. the radius of tube) $r$ and outer radius (i.e. the radius from center of hole to center of tube) $R > r$ may be parameterized as
\[ T: [0, 2\pi) \times [0, 2\pi) \to \mT^2 \]
\[ T(\theta, \phi) = ((R + r \cos \ta)\cos \phi, (R + r \cos \ta)\sin \phi, r \sin \ta) \]
The metric $g$ inherited from $\R^3$ is
\begin{align*}
g 
&= dx^2 + dy^2 + dz^2 \\
&= d((R + r \cos \ta)\cos \phi)^2 + d((R + r \cos \ta)\sin \phi)^2 + d(r \sin \ta)^2 \\
&= \left(d\phi \sin \phi (-(r \cos \ta +R))-r \cos \phi d\ta \sin \ta \right)^2 \\ &\quad +\left(d\phi \cos \phi (r \cos \ta +R)-r \sin \phi d\ta \sin \ta \right)^2+r^2 d\ta^2 \cos ^2\ta \\
&= (R + r\cos\ta)^2 d\phi^2 + r^2 d\ta^2
\end{align*}
and so the corresponding matrix $(g_{ij})$ is
\[ g(\ta, \phi) = \begin{pmatrix} r^2 & 1 \\ 1 & (R + r \cos \ta)^2 \end{pmatrix} \]
The Laplacian of the torus embedded in $\R^3$ is then
\begin{align} 
\hspace*{-5pt} \lap f 
&= - \frac{1}{\sqrt{|\det g|}}\sum_{i,j=1}^{n}\left(g^{ij}\sqrt{|\det g|}\ppx{}{x_j}\right) \\
&= - r^{-2} \left(R + r \cos\ta\right)^{-1} \ppx{}{\ta}\left(R + r \cos \ta\right)\ppx{}{\ta}f - \left(R + r \cos \ta\right)^{-2} \ppx{^2}{\phi^2} f \label{eq:lap_emb_torus}
\end{align}
Whereas the distances and angles on the flat torus act similarly to those in $\R^2$, distances and angles on the embedded torus act as we would expect from a donut shape in $\R^3$. For example, the sum of angles of an triangle drawn on a flat torus is always $\pi$, but this is not true on the torus embedded in $\R^3$.\footnote{A triangle drawn on the ``inside'' of the torus embedded in $\R^3$ has a sum of angles that is less than $\pi$, whereas a triangle drawn on the ``outside'' has a sum of angles that is greater than $\pi$. Although we will not discuss Gaussian curvature in this text, we note that this sum of angles is governed by the curvature of the surface, which is negative on the inside of the torus and positive on the outside. As another example, the sum of angles of a triangle on the $2$-sphere, which has positive Gaussian curvature, is $\tfrac{3\pi}{2}$. }

More formally, the embedded torus is diffeomorphic to the flat torus but not isomorphic to it: there exists a smooth and smoothly invertible map between them, but no such map that preserves distances. In fact, there does not exist a smooth embedding of the flat torus in $\R^3$ that preserves its metric.
\footnote{For the interested reader, we remark that it is known that there does not even exist a smooth metric-preserving (i.e. isometric) $C^2$ embedding of the flat torus in $R^3$. However, results of Nash from 1950 show that there does exist an isometric $C^1$ embedding. In $2012$, the first explicit construction of such an embedding was found; its structure resembles that of a fractal \cite{borrelli2012flat}.}

\section{Lessons from Physics} 
We would be remiss if we introduced the Laplacian without discussing its connections to physics. These connections are most clear for the Laplacian on manifolds, which figures in a number of partial differential equations, including the ubiquitous heat equation. 

\paragraph{Example: Fluid Flow (Manifolds)} Suppose we are physicists studying the movement of a fluid over a continuous domain $D$. We model the fluid as a vector field $v$. Experimentally, we find that the fluid is incompressible, so $\text{div}\, v = 0$, and conservative, so $v = - \grad u$ for some function $u$ (the potential). The potential then must satisfy 
\[ \lap u = 0 \]
This is known as Laplace's Equation, and its solutions are called harmonic functions. 

\paragraph{Example: Fluid Flow (Graphs)} Now suppose we are modeling the flow of a fluid through pipes that connect a set of reservoirs. These reservoirs and pipes are nodes and edges in a graph $G$, and we may represent the pressure at each reservoir as a function $u$ on the vertices. 

Physically, the amount of fluid that flows through a pipe is proportional to the difference in pressure between its vertices, $u_i - u_j$. Since the total flow into each vertex equals the total flow out, the sum of the flows along a vertex $i$ is $0$: 
\begin{align} \label{eq:neighbor_avg}
    0 = \sum_{j \in N(i)} u_i - u_j
\end{align}
Expanding this gives: 
\begin{align*}
    0 &= \sum_{j \in N(i)} u_j - \sum_{j \in N(i)} u_i = \deg(i) u_i - \sum_{j \in N(i)} u_j \\
    &= \l (D - A)u \r_i  = (\lapg u)_i
\end{align*}
We find that $\lapg u$ is $0$, a discrete analogue to the Laplace equation $\lap u = 0$. 

Equivalently, Equation \ref{eq:neighbor_avg} means that each neighbor is the average of its neighbors: 
\[ u_i = \frac{1}{\deg(i)} \sum_{j \in N(i)} u_j \]
We can extend this result from 1-hop neighbors to $k$-hop neighbors, by noting that each of the 1-hop neighbors is an average of their own neighbors and using induction. 

While this result is obvious in the discrete case, it is quite non-obvious in the continuous case. There, the analogous statement is that a harmonic functions equals its average over a ball.
\begin{theorem}[Mean Value Property of Hamonic Functions]
Let $u \in C^2(\Om)$ be a harmonic function on an open set $\Om$. Then for every ball $B_r(x)
 \subset \Om$, we have 
\[ u(x) = \frac{1}{|B_r(x)|} \int_{B_r(x)} u(x)\,dx=  \frac{1}{|\partial B_r(x)|} \int_{\partial B_r(x)} u(x)\,dx \]
where $\partial B_r$ denotes the boundary of $B_r$. 
\end{theorem}
If one were were to only see this continuous result, it might seem somewhat remarkable, but in the context of graphs, it is much more intuitive. 


For graphs, the converse of these results is also clear. If a function $u$ on a graph is equal to the average of its $k$-hop neighbors for any $k$, then the sum in Equation \ref{eq:neighbor_avg} is zero, so $\lapg u = 0$ and $u$ is harmonic. For manifolds, it is also true that if $u$ equals its average over all balls centered at each point $x$, then $u$ is harmonic. 

\paragraph{Example: Gravity} Written in differential form, Gauss's law for gravity says that the gravitational field $g$ induced by an object with mass density $\rho$ satisfies 
\[ \grad g = - 4 \pi G \rho \]
where $G$ is a constant. Like our model of a fluid above, the gravitational field is conservative, so $g = - \grad \phi$ for some potential function $\phi$. We then see 
\[ \lap \phi = 4 \pi G \rho  \]
Generally, a partial differential equation of the form above
\[ \lap u = f \]
is known as the Poisson equation. 

Note that if the mass density is a Dirac delta function, meaning that all the mass is concentrated at a single point, the solution to this expression turns out to be $\phi(r) = - Gm/r$, which is Newton's law of gravitation. 

\paragraph{Example: Springs} Consider a graph in which each node exerts upon its neighbors an attractive force. For example, we could imagine each vertex of the graph as a point a $2D$ plane connected to its neighbors by a spring.

Hooke's Law states that the potential energy of a spring is $\tfrac{k}{2}x^2$, where $x \in \R^2$ is the amount the spring is extended or compressed from its resting displacement. Working in the $2D$ plane, the length of the spring is the difference $\norm{\bx_i - \bx_j}$ where $\bx_i = (x_i, y_i)$ and $\bx_j = (x_j, y_i) \in \R^2$ are the positions of the two nodes. 

If the resting displacement of each spring is $0$, the potential energy in the $(i,j)$ spring is $\tfrac{k}{2} \norm{\bx_i - \bx_j}^2$. The total potential energy of our system is sum of the energies in each spring:
\[ \sum_{(i,j)\in E} \tfrac{k}{2} \norm{\bx_i - \bx_j}^2 \propto x^T \lapg x + y^T \lapg y \] 
We see that finding a minimum-energy arrangement corresponds to minimizing a Laplacian quadratic form. If we were working in $\R^1$ instead of $\R^2$, the expression above would coincide exactly with our traditional notion of the Laplacian $x^T \lapg x$. 

\paragraph{Harmonic Functions}

As seen repeatedly above, we are interested in harmonic functions, those for which $\lapg = 0$. However, on a finite graph, all such functions are constant! 

We can see this from our physical system of springs with resting displacement $0$. Intuitively, if $G$ is connected, the springs will continue pulling the vertices together until they have all settled on a single point, corresponding to a constant function. Alternatively, if $x^T \lapg x = 0$, then each term $(x(i) - x(j))^2$ in the Laplacian quadratic form must be $0$, so $x$ must be constant on each neighborhood. Since $G$ is connected, $x(i)$ must then be constant for all vertices $i$. 

Nonetheless, all is not lost. Interesting functions emerge when we place additional conditions on some of the vertices of the graph. In the case of the spring network, for example, we can imagine nailing some of the vertices onto specific positions in the $2D$ plane. If we let this system come to equilibrium, the untethered vertices will settle into positions in the convex hull of the nailed-down vertices, as shown in \autoref{how_to_draw_a_graph}. 

In fact, a famous theorem of Tutte \cite{Tutte1963HowTD} states that if one fixes the edges of a face in a (planar) graph and lets the others settle into a position that minimizes the total potential energy, the resulting embedding will have no intersecting edges.

\begin{figure}[]
    \centering
    \includegraphics[width=0.75\textwidth]{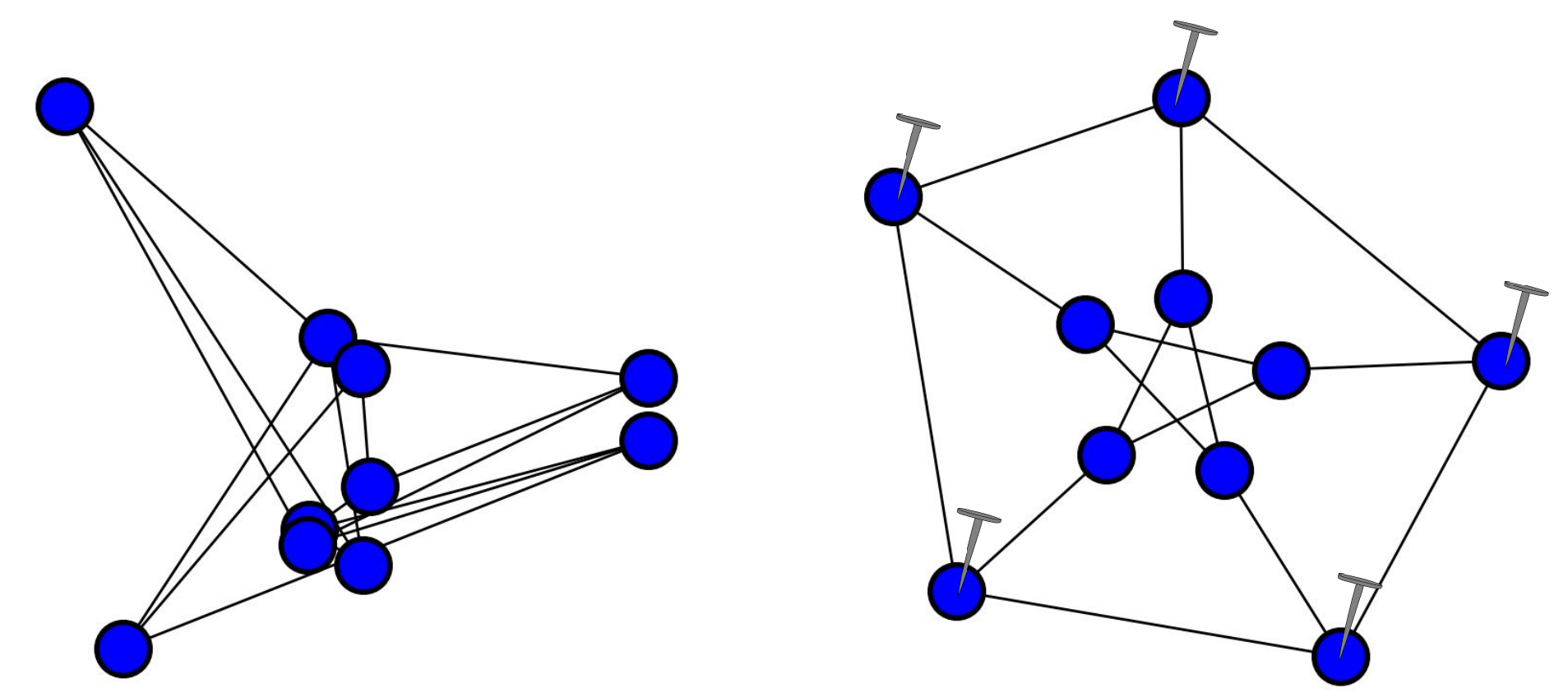}
    \caption[Illustration of Tutte's Theorem]{An illustration of Tutte's Theorem. On the left, we embed a graph into the plane by placing its vertices at random positions. On the right, we show the same graph embedded by taking one of its faces, nailing them in place, and letting the others settle into an arrangement with minimal potential energy.}
    \label{how_to_draw_a_graph}
\end{figure}

\begin{theorem}[Tutte's Theorem]
    Let $G = (V, E)$ be a $3$-connected, planar graph. Let $F$ be a set of vertices that forms a face of $G$. Fix an embedding $F \to \R^2$ such that the vertices of $F$ form a strictly convex polygon. Then this embedding may be extended to an embedding $V \to \R^2$ of all of $G$ such that 
    \begin{enumerate}
        \item Every vertex in $V \setminus F$ lies at the average of its neighbors.
        \item No edges intersect or self-intersect. 
    \end{enumerate}
\end{theorem}

The statements above all have continuous analogues. Like a harmonic function on a finite graph, a harmonic function on a compact manifold without boundary (a \textit{closed} manifold) is constant. 
\begin{theorem}
    If $f$ is a harmonic function on a compact boundaryless region $D$, $f$ is constant. 
\end{theorem}
On a region with boundary, a harmonic function is determined entirely by its values on the boundary. 
\begin{theorem}[Uniqueness of harmonic functions]
    Let $f$ and $g$ be harmonic functions on a compact region $D$ with boundary $\partial D$. If $f = g$ on $\partial D$, then $f = g$ on $D$. 
\end{theorem}
As a result, if a harmonic function is zero on its boundary, it is zero everywhere. This result is often stated in the form of the maximum principle. 
\begin{theorem}[Maximum Principle]
    If $f$ is harmonic on a bounded region, it attains its absolute minimum and maximum on the boundary.
\end{theorem}
The maximum principle corresponds to the idea that if we nail the vertices of the face of a graph to the plane, the other nodes will settle inside of their convex hull; if every point is the average of its neighbors, the maximum must be attained on the boundary. 

\paragraph{Example: More Fluids} Returning to continuous fluids, suppose we are interested in understanding how a fluid evolves over time. For example, we may be interested in the diffusion of heat over a domain $D$. This process is governed by the ubiquitous heat equation: 
\[ \partial_t u(x, t) = \lap u(x, t) \]
One common approach to solving this equation is to guess a solution of the form $u(x,t) = \alpha(t) \phi(x)$ and proceed by separation of variables. This yields:
\[ \frac{\lap \phi(x)}{\phi(x)} = - \frac{\al'(t)}{\al(t)} \]
which implies that 
\[ \al' = - \lam \al \qtxtq{and} \lap \phi = \lam \phi  \]
for some $\lam \in \R$. The equation on the left yields $\al(t) = C e^{-\lam t}$, and the equation on the right shows that $\lam$ is an eigenvalue of $\lap$. This second equation is called the Helmholtz equation, and it shows that the eigenvalues of the Laplacian enable us to understand the processes it governs. Note also that the Laplace equation is a special case of the Helmholtz equation with $\lam = 0$. 

We discuss the heat equation (on both manifolds and graphs) in more detail in \autoref{sec:heat_kernel}. Before doing so, we need to understand the eigenvalues and eigenvectors of the Laplacian operator. 


\section{The Laplacian Spectrum}
Our primary method of understanding the Laplacian will be by means of its eigenvalues, or spectrum. 


We denote the eigenvalues of the Laplacians $\lapg$ and $\lap$ by $\lam_i$, with $\lam_1 \le \lam_2 \le \cdots$. We use the same symbols for both operators, but will make clear at all times which operator's eigenvalues we are referring to. In the graph case these are finite ($\lapg$ has $n$ eigenvalues counting multiplicities), whereas in the case of a manifold they are infinite. 

We have seen that $\lapg$ and $\lap$ are self-adjoint positive-definite operators, so their eigenvalues are non-negative. By the spectral theorem, the eigenfunctions are orthonormal and form a basis for the Hilbert Space of $L^2$ functions on their domain. For a manifold $\fM \subset \R^n$, the eigenfunctions form a basis for $L^2(\fM)$, and for a graph $G = (V, E)$, they form a basis for $L^2(V)$ (i.e. bounded vectors in $\R^n$). 

We have also already seen that the constant function $\ind$ is an eigenfunction of the Laplacian corresponding to eigenvalue $\lam_1 = 0$. 

\textit{Notation: } Unfortunately, graph theorists and geometers use different conventions for the eigenvalues. Graph theorists number the eigenvalues $\lam_1, \lam_2, \dots$, with $\lam_1 = 0$, and prove theorems about the ``second eigenvalue'' of the Laplacian. Geometers number the eigenvalues $0, \lam_1, \dots$, and prove theorems about the ``first eigenvalue'' of the Laplacian. We will use the convention from spectral graph theory throughout this text. 

\paragraph{Can you hear the shape of a drum?} A famous article published in 1966 in the American Mathematical Monthly by Mark Kac asked ``Can you hear the shape of a drum?'' \cite{kac1966can} The sounds made by a drumhead correspond to their frequencies, which are in turn determined by the eigenvalues of the Laplacian on the drum (a compact planar domain). If the shape of the drum is known, the problem of finding its frequencies is the Helmholtz equation above. Kac asked the inverse question: if the eigenvalues of the Laplacian are known, is it always possible to reconstruct the shape of the underlying surface? Formally, if $D$ is a compact manifold with boundary on the plane, do the solutions of $\lap u + \lam u = 0$ with the boundary condition $u|_{\partial D} = 0$ uniquely determine $D$? 

The problem remained unsolved until the early 1990s, when Gordon, Webb and Wolpert answered it negatively \cite{gordon1992one}. The simple counterexample they presented is shown in \autoref{you_cannot_hear_the_shape_of_a_drum}. 

\begin{figure}[]
    \centering
    \includegraphics[width=\textwidth]{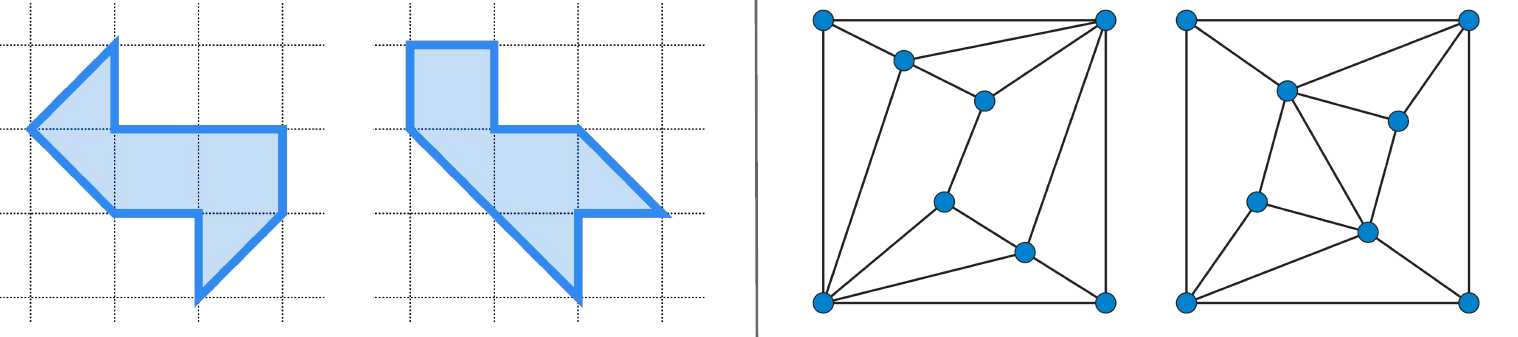}
    \caption[Isospectral Domains and Graphs]{The two domains on the left have the same Laplacian spectrum, but are not isomorphic. The same is true of the two graphs on the right.}
    \label{you_cannot_hear_the_shape_of_a_drum}
\end{figure}

Nonetheless, the difficulty of proving this fact demonstrates just how much information the eigenvalues contain about the Laplacian. Indeed, Kac proved that the eigenvalues of $\lap$ on a domain encode many geometric properties, including the domain's area, perimeter, and genus. 

Similarly, it is not possible to reconstruct the structure of a graph from the eigenvalues of its Laplacian (\autoref{you_cannot_hear_the_shape_of_a_drum}).\footnote{Also, if graphs with identical spectra were isomorphic, we would have a polynomial time solution to the graph isomorphism problem, the problem of determining whether two finite graphs are isomorphic. The graph isomorphism problem is neither known to be solvable in polynomial time nor known to be NP-complete.} 

\subsection{Examples of Laplacian Spectra}

Below, we give examples of the eigenvalues and eigenfunctions of a number of the manifolds and graphs from \autoref{ssec:examples}. 

\paragraph{Example: $\mC^n$ and $\R^n$} 

In $\mC^n$, the eigenvalue equation $\lap f = \lam f$ for the standard Laplacian $\lap = - \sum_{i=1}^n \ppx{^2}{x_i^2}$, is satisfied by the complex exponentials. In other words, the eigenfunctions of $\lap$ are the functions $x \mapsto e^{i \sqrt{\lam}x_i}$ for any $\lam \ge 0$, where $\lam = 0$ corresponds as usual to the constant function. 

In $\R^n$, both the real and imaginary parts of the complex exponentials satisfy $- \sum_{i=1}^n \ppx{^2}{x_i^2} f = \lam f$. These are sine and cosine functions of the form $\sin(\sqrt{\lam} x_i)$ and $\cos(\sqrt{\lam} x_i)$, and as above every real $\lam$ in the continuous region $[0, \infty)$ is an eigenvalue. 

\paragraph{Example: $S^1$} The circle $S^1$, which inherits its metric from $\R^2$, looks locally like $\R^1$ but is globally periodic. The spectrum of its Laplacian are the functions on $S^1$ that solve
\begin{equation} \label{eq:circle_mani_eig}
- \ppx{^2}{\ta_i^2}f = \lam f
\end{equation}
which is to say they are the solutions to this equation in $\R^1$ that are also periodic with period $2\pi$. These solutions take the form
\[ f(\ta) = e^{ik\ta} \]
for $k \in \mZ$. The real and imaginary parts of this expression yield the full set of eigenfunctions 
\[f(\ta) = 1, \qquad f(\ta) = \sin(k\ta), \qquad f(\ta) = \cos(k\ta),  \qquad \text{for } k = \{ 1, 2, \dots \} \]
with corresponding eigenvalues $0, k^2, k^2$ for $k \in \{ 1, 2, \dots \}$. 

From another perspective, $S^1$ is locally like $\R^1$, so a sine/cosine wave with any wavelength locally satisfies Equation \ref{eq:circle_mani_eig}, but in order for it to be well-defined globally, its wavelength must be a multiple of $2 \pi$. Consequently, whereas the spectrum of $\lap$ in $\R^1$ is continuous, the spectrum of $\lap$ in $S^1$ is discrete. Consistent with this intuition, one can prove that all closed manifolds have discrete spectra, whereas non-compact manifolds may have continuous spectra.

Additionally, consider a circle with a non-unit radius $r$. From polar coordinates, we can see that the Riemannian metric is $g = r\,d\ta$ and the Laplacian becomes 
\[ \lap f = - \frac{1}{r} \ppx{}{r}\left(r\ppx{f}{r} \right) - \frac{1}{r^2} \ppx{^2f}{\ta^2} \]
which has eigenvalues $0, k^2, k^2$ for $k \in \{ 1, 2, \dots \}$. As we increase the radius of our circle, we see that the spectrum becomes more dense in $\R$, and as it goes to infinity, we fill the entire region $[0,\infty)$, which is the spectrum of $\R^1$. 

\paragraph{Example: Cycle Graph} As computed above, the Laplacian of the cycle graph is given by
\[ \lapg = \left(\begin{array}{cccccc} 2 & -1 & 0 & 0 & 0 & -1 \\-1 & 2 & -1 & 0 & 0 & 0 \\ 0 & -1 & \ddots & \ddots & 0 & 0 \\ 0 & 0 & \ddots & \ddots & -1 & 0 \\ 0 & 0 & 0 & -1 & 2 & -1 \\-1 & 0 & 0 & 0 & -1 & 2 \\\end{array}\right) \]
In \autoref{circle_and_cycle}, we compute its eigenfunctions numerically for $n = 30$ and $100$ vertices and plot the first six eigenfunctions. Comparing these to the plots of the eigenfunctions of the cycle graph, we see that the (scaled) eigenfunctions of the cycle graph approach those of the circle!


In this way, the cycle graph is a discrete version of a circle. 
\begin{figure}[]
    \centering
    \includegraphics[width=\textwidth]{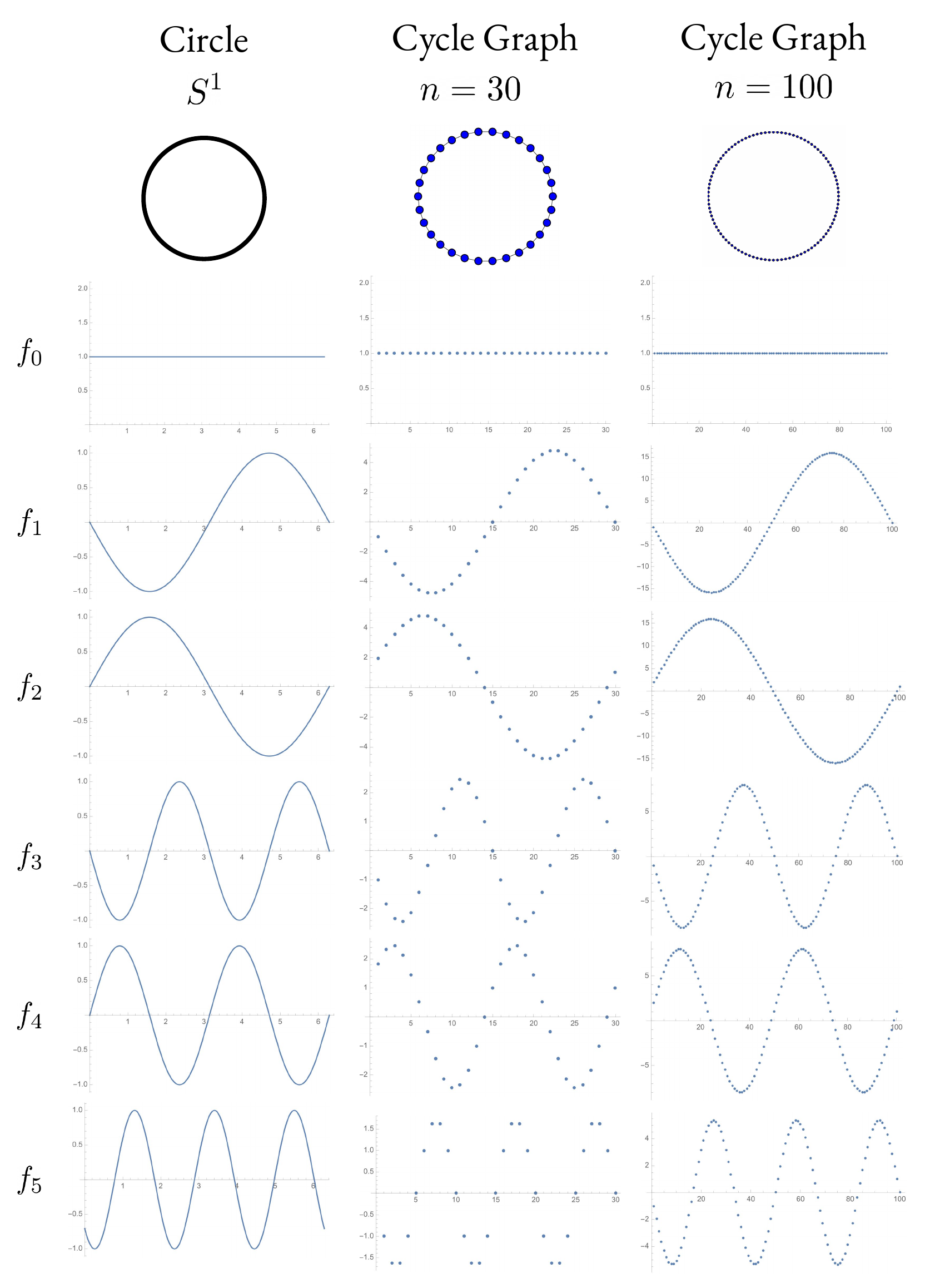}
    \caption[Eigenfunctions of the Circle and the Cycle Graph]{An illustration comparing the first six eigenfunctions of the circle and the cycle graph with $n=30,100$. }
    \label{circle_and_cycle}
\end{figure}

\paragraph{Example: Flat Torus} We saw previously that with the flat metric, the $n$-dimensional torus looks like a linearly transformed square in $\R^n$ with periodic boundary conditions. Formally, we have $\mT^n = \R^n / \Gamma$ for an $n$-dimensional lattice $\Gamma$ generated by a basis $\{e_1, \dots, e_n\}$ of $\R^n$. 

To compute its eigenvalues, let $\Gamma^*$ be the dual lattice, defined as $\{x \in \R^n : \br x, y\kt \in \mZ \, \forall y \in \Gamma \}$. Just as with the other flat manifolds ($a\R^n$ and $S^1$) above, the solutions to eigenvalue equation $\lap f = \lam f$ solve $\sum_{i=1}^{n}\ppx{^2}{x_i^2}f(x) = \lam f$, so they are complex exponentials:
\[ x \mapsto e^{2\pi i \br x,y \kt} \qtxtq{for all} y \in \Gamma \] 
The real and imaginary parts yield the eigenfunctions $1, x \mapsto \sin(2\pi i \br x,y \kt)$, $x \mapsto \cos(2\pi i \br x,y \kt)$ for $y \in \Gamma^*$, which form a basis for $L^2(\mT^n)$. The corresponding eigenvalues are $0, 4\pi^2|y|^2, 4\pi^2|y|^2$, similar to those on the circle $S^1$. 

\paragraph{Example: Embedded Torus} 
We computed the Laplcaian of the $2$-torus with the metric induced from $\R^3$, rather than the flat metric, in Equation \ref{eq:lap_emb_torus}. Its eigenvalue equation is then
\begin{align*}
\hspace*{3pt} \lap f &= - r^{-2} \left(R + r \cos\ta\right)^{-1} \ppx{}{\ta}\left(R + r \cos \ta\right)\ppx{}{\ta}f - \left(R + r \cos \ta\right)^{-2} \ppx{^2}{\phi^2} f = \lam f
\end{align*}
As this equation is separable, we consider a solution of the form $\psi(\ta, \phi) = a(\ta)e^{i k \phi}$ for $k \in \{1, 2, \dots \}$. Simplifying, we obtain 
\[ - \frac{1}{r^2} a''(\ta) + \frac{\sin \ta}{r + R \cos \ta} a'(\ta) + \frac{k^2}{(r + R \cos \ta)^2} a(\ta) = \lam a(\ta) \]
which is an ordinary differential equation in $a$ with periodic boundary conditions, solvable for given values of $r$ and $R$. Note that each non-constant eigenvalue has multiplicity at least $2$, corresponding to the real and imaginary parts of $e^{i k \phi}$, as with the flat torus and the circle.  

\paragraph{Example: More Fundamental Graphs} Recall from \autoref{fundamental_graphs} the Laplacians of the fully connected graph and the star graph on $n$ vertices.

The eigenvalues of the complete graph, apart from $\lam_1 = 0$, are $n$ with multiplicity $n-1$. As we shall see shortly, a graph's eigenvalues tell us about its connectedness, and the fully-connected graph has the largest eigenvalues.

The star graph has eigenvalues $\lam_1 = 0$, $\lam_n = n$, and $\lam_i = 1$ for $1 < i < n$. Note that the second eigenvector, $\lam_2$, is small. The graph is connected, but is ``close'' to being disconnected in the sense that if the middle vertex were removed, it would be entirely disconnected. 

A star graph is an instance of a complete bipartite graph: its vertices can be divided into two subsets such that each vertex is connected (only) to the vertices of the other subset. In general, denoting by $K_{m,n}$ the complete bipartite graph with subsets of size $m$ and $n-m$, the Laplacian $\lapg_{K_{m,n}}$ has eigenvalues $0$, $n$, $m$, and $n+m$ with multiplicies $1, m-1, n-1,$ and $1$, respectively. 

This result is a consequence of the following key lemma. 
\begin{lemma} \label{lemma:complement_of_graph}
Let $G$ be a simple graph. Let $\ol{G}$ be its complement, the graph on the same vertices as $G$ such that each edge is included in $\ol{G}$ if and only if it is not in $G$. Denote the eigenvalues of the Laplacian $\lapg_G$ of $G$ by $0 = \lam_1 \le \cdots \le \lam_n$. Then the eigenvalues of the Laplacian $\lapg_{\ol{G}}$ of $\ol{G}$ are
\[ 0, n - \lam_n, n - \lam_{n-1}, \dots, n - \lam_{2} \]
\end{lemma}
\begin{proof}
Let $v_1, \dots, v_n$ be orthonormal eigenvectors of $\lapg_G$ corresponding to $\lam_1, \dots, \lam_n$. The sum of the Laplacians of $G$ and $\ol{G}$ is
\begin{align*}
    \lapg_{G} + \lapg_{\ol{G}} &= D_{G} - A_{G} + D_{\ol{G}} - A_{\ol{G}} = (D_{G} + D_{\ol{G}}) - (A_{G} + A_{\ol{G}}) \\ &= n I - J
\end{align*}
where $J$ is the matrix of all $1s$. Now consider $\lapg_{\ol{G}}v_i$. If $v_i$ is the constant vector, $\lapg_{\ol{G}}v_i = 0$. If it is not constant, it is orthogonal to the constant vector, so $Jv_i = 0$ and 
\begin{align*}
\lapg_{\ol{G}}v_i &= (nI - J - \lapg_{G}) v_i = n v_i - 0 - \lam_i v_i = (n-\lam_i)v_i
\end{align*}
Therefore the eigenvalues of $\lapg_{\ol{G}}$ are $0, n - \lam_n, n - \lam_{n-1}, \dots, n - \lam_{2}$. Also, its set of eigenvectors is the same as that of $\lapg_{G}$. 
\end{proof}

From this lemma, it is quick to deduce the eigenvectors of the complete graph and $K_{m,n}$. The complete graph is the complement of the empty graph, which has eigenvalues $0^{(n)}$, so its eigenvalues are $0, n^{(n-1)}$. $K_{m,n}$ is the complement of the union of two complete graphs on $n$ and $m$ vertices. It is simple to show that the eigenvalues of the union of two graphs is the union of their eigenvalues, so the eigenvalues of the union are $0^{(2)}, n^{(n-1)}, m^{(m-1)}$. Then by the lemma the eigenvalues of $K_{m,n}$ are $0, n^{(m-1)}, m^{(n-1)}, n$. 

Moreover, since the eigenvalues of every graph are nonnegative, the lemma shows that $n$ is the largest that an eigenvalue of a graph with $n$ vertices can be. In this way, the complete graph has the largest eigenvalues. 

\subsection{A Note on Boundaries}

Before proceeding, we take a moment to address the concept of manifolds with boundary, as the reader likely has or will encounter such structures in the Riemannian geometry literature. We emphasize that finite graphs are analogous to \textit{closed} (i.e. compact and boundaryless) manifolds, rather than those with boundary. A number of results in this text hold for manifolds with boundary and noncompact manifolds, but we make no guarantees.  

For manifolds with boundary, the eigenfunctions of the Laplacian depends on both the underlying domain and the conditions placed on the boundary. For example, Kac's original ``shape of a drum'' question specified the boundary condition $u|_{\partial D} = 0$. This condition is the first of the two most widely-studied boundary conditions, \textit{Dirichlet boundary conditions} and \textit{Neumann boundary conditions}.\footnote{Although less common, other types of boundary conditions include Robin, Mixed, and Cauchy conditions. Each of these is different a combination of Dirichlet and Neumann boundary conditions (Robin is a linear combination, Mixed is a piecewise combination, and Cauchy imposes both at once).}

Dirichlet boundary conditions require that the function be zero on its boundary: 
\[ \lap u = \lam u \text{ on } D, \qquad u|_{\partial D} = 0 \]
Neumann boundary conditions require that the function's derivative be zero on its boundary: 
\[ \lap u = \lam u \text{ on } D, \qquad \ppx{u}{\nu}|_{\partial D} = 0  \]
where $\nu$ is the unit outward normal to $\partial D$.

To use the example of heat flow, Dirichlet boundary conditions correspond to a closed system in which no heat is allowed to enter or leave the system, whereas Neumann boundary conditions correspond to a system with a constant flow of heat at each point in the boundary.  

These two types of boundary conditions only have graph analogues in the setting of \textit{infinite} graphs. On finite graphs, fixing the value of a set of vertices determines a unique solution to $\lap f = \lam f$. 
Analogues of Dirichlet and Neumann boundary-value problems on infinite graphs is an active area of research \cite{JAVAHERI20072496,haeseler2011laplacians}.

\subsection{The Rayleigh Characterization of Eigenvalues}

There are many ways of characterizing the eigenvalues of an operator. One particularly useful characterization is the Rayleigh quotient, which enables us to express eigenvalues as the solutions to optimization problems. 

We begin in the setting of graphs. Let $\bA$ be a self-adjoint matrix with eigenvalues $\lam_1 \le \cdots \le \lam_n$. The Rayleigh quotient of a vector $x$ is the expression 
\[ R(x) = \frac{x^T \bA x}{x^T x} \]
where the denominator functions as a normalization factor. The Courant-Fischer Theorem states that $\lam_1$ minimizes this expression over all nonzero $x$, $\lam_2$ minimizes it over all $x$ orthogonal to the first eigenvector, $\lam_2$ maximizes it over $x$ orthogonal to the first two eigenvectors, and so on.
\begin{theorem}[Courant-Fischer]
The $k$-th smallest eigenvalue $\lam_k$ of the self-adjoint matrix $\bA$ is given by
\begin{equation}
    \lam_k = \min_{S \subset \R, \dim(S) = k} \max_{x\in S, x\ne 0} \frac{x^T \bA x}{x^T x}
\end{equation}
where $S$ is a subspace of $\R^n$. 
\end{theorem}
The proof of Courant-Fischer is an application of the famous Spectral Theorem (for the details, see \cite{sagt}, Chapter2).

For a Laplacian $\lapg$ of a graph $G$, the first eigenvalue $\lam_1 = 0$ corresponds to the constant vector $\textbf{1}$. We then immediately have what is known as the Rayleigh characterization of $\lam_2$. 
\begin{corollary}
The first nonzero eigenvalue $\lam_2$ of $\lapg$ is given by
\[ \lam_2 = \min_{\ns{x} = 1, x \perp \textbf{1}} x^T \lapg x  \]
\end{corollary}

In what should not be an enormous surprise at this point, the Rayleigh quotient has an analogue on manifolds: 
\[ R(f) = \frac{\int_{\fM}|\grad f|^2\,dV}{\int_{\fM}f^2\,dV} = \frac{\br \grad f,\grad f \kt}{\br f, f \kt} \]
where $dV$ is the volume form on the manifold. The eigenvalues are given by the same optimization problem:
\[ \lam_1 = 0, \qquad \lam_2 = \min \left\{ R(f) : \int_\fM f\,dV = \br f, \textbf{1} \kt = \int_\fM f\,dV = 0  \right\}    \]
The first eigenvalue is $0$, corresponding to a constant eigenfunction, and the next largest eigenvalue is the minimizer of the Rayleigh quotient over all functions orthogonal to a constant function.\footnote{Technically, this minimization is taken over all functions $f$ in the Sobolev space $H^1(\fM)$ corresponding to $\fM$.} 
Subsequent eigenvalues $\lam_3, \lam_4, \dots$ of $\fM$ may be obtained by a similar process as in the graph case. 
\[ \lam_k = \min \left\{ R(f) : \br f, f_i \kt = 0 \quad \forall \quad i < k f_i \right\}    \]
where $f_i$ denotes the eigenfunction corresponding to the $i$-th eigenvalue $\lam_i$. 

\section{Eigenvalues and Connectivity}

The Laplacian spectrum is closely related to the notion of connectedness. 

\subsection{The First Eigenvalues}

The multiplicity of the first (zero) eigenvalue of the Laplacian gives the number of connected components of its corresponding graph or manifold. 

\begin{lemma}
The number of connected components of a graph $G$ equals the multiplicity of the $0$ eigenvalue of $\lapg$. 
\end{lemma}
\begin{proof} 
First, suppose $f$ is an eigenfunction of $\lapg$ corresponding to $0$. Then $f \lapg f = \sum_{(i,j)\in E} (f(i) - f(j))^2 = 0$. In order for this sum to be $0$, if $f$ is nonzero on a vertex $v$, it must take the same value on every vertex connected to $v$. Then $f$ must be constant on each component, meaning the multiplicity of the eigenvalue $0$ is at most the number of connected components. 

Second, note that for each connected component of the graph, the characteristic function of the component is an eigenfunction, so the multiplicity of the eigenvalue $0$ is at least the number of connected components. 
\end{proof}

For simplicity, we assume from now on that the graphs/manifolds we are discussing are connected, so $\lam_1$ has multiplicity $1$. 

The second eigenvector $\lam_2$ tells us about the connectivity of the graph or manifold in a different way from $\lam_1$. Whereas $\lam_1$ tells us whether the graph is connected at all, $\lam_2$ gives us a sense of \textit{how} connected the graph is. Informally, if $\lam_2$ is small, then the graph is weakly connected, whereas if $\lam_2$ is large, the graph is strongly connected. We have already seen one example of this idea above: a graph is fully connected if and only if $\lam_2$ is as large as possible ($\lam_2 = n$). 

Graph theorists call $\lam_2$ the \textit{algebraic connectivity} of a graph. It is also sometimes referred to as \textit{Fiedler value} for Czech mathematician Miroslav Fiedler, who was among the first to give bounds on $\lam_2$.

Geometers call $\lam_2$ the \textit{fundamental tone} of a manifold. This name is derived from the fact that if we imagine a vibrating manifold, $\lam_2$ is its leading frequency of oscillation. 

\subsection{Eigenvalue Bounds}

We have seen that we can understand the structure of graphs and manifolds by looking at the eigenvalues of their Laplacians. In general, however, it is challenging to obtain analytic expressions for these eigenvalues. 

Instead, most work is dedicated to proving and tightening bounds on these eigenvalues. The Rayleigh characterization of eigenvalues is useful because it gives us a simple method of obtaining an upper bound on $\lam_2$: for any $f$, the Rayleigh quotient $\frac{\br f, \lapg f\kt}{\br f,f \kt}$ bounds $\lam_2$.

Here, we give bounds on the eigenvalues derived from simple properties of graphs and manifolds. We will build up to a proof of Cheeger's Inequality, a bound on $\lam_2$ that was first proven on manifolds, but has recently seen widespread use in graph theory. 

\begin{theorem}
Let $G$ be a simple connected graph. 
\begin{enumerate}
\item $\lam_n \le n$ with equality if and only if the complement $\ol{G}$ is disconnected.
\item $\sum_{i=1}^{n}\lam_i = \sum_{v\in V} d_v = 2 |E$
\item $\lam_2 \le \frac{n}{n-1}\min_{v\in V} d_v \qtxtq{and} \lam_n \ge \frac{n}{n-1}\max_{v\in V} d_v$
\item $\lam_n \le \max_{i \in V}(d_i + m(i))$
where $m(i)$ is the average of the degrees of vertices adjacent to vertex $i$. 
\end{enumerate}
\begin{proof}

\begin{enumerate}
    \item From Lemma \ref{lemma:complement_of_graph}, the eigenvalues of $G$ are $0, \lam_2, \dots, \lam_n$, those of $\ol{G}$ are $0, n - \lam_n, \dots, n - \lam_n$. The eigenvalues of $\ol{G}$ are nonnegative, so $\lam_i \le n$. As shown above, $0$ has multiplicity greater than $1$ in $\ol{G}$ if and only if $\ol{G}$ is disconnected, so $n$ is an eigenvalue of $G$ if and only if $\ol{G}$ is disconnected. 
    \item The sum of the eigenvalues of an operator equals its trace, and the trace of $\lapg = D - A$ is the same as the trace of $D$, which is the sum of the degree of each vertex: $\sum_{v\in V} d_v$. 
    \item This result is due to Fielder \cite{fiedler1973algebraic}. For a proof, see Appendix \ref{appendix:bounds_graph}.
    \item This result is due to Merris \cite{merris1998note}, building off a result from Anderson and Morley \cite{anderson1985eigenvalues}. For a proof, see Appendix \ref{appendix:bounds_graph}.
\end{enumerate}
\end{proof}
\end{theorem}


Another way of seeing the connection between the Laplacian spectrum and graph connectivity is to observe how they behave as one changes the graph. In particular, if one adds an edge to the graph, the eigenvalues only increase. 

\begin{theorem}[Edges Increase Eigenvalues]
Let $G$ be a non-complete graph and $(i,j)$ an edge not in $E$. Denote by $G'$ the graph $G$ with edge $(i,j)$ added. Then the eigenvalues of $G'$ interlace those of $G$:
\[ 0 = \lam_1(G) = \lam_1(G') \le \lam_2(G) \le \lam_2(G') \le \lam_3(G) \le \cdots \le \lam_n(G) \le \lam_n(G') \]
\end{theorem}
The proof of this theorem is included in Appendix \ref{appendix:cauchy}.\footnote{The proof involves background (complex analysis) beyond the expected background of the reader. Nevertheless, we encourage adventurous readers to give it a look!} It is closely related to Cauchy's Interlace Theorem and Weyl's Theorem, two corollaries of the Courant-Fischer Theorem. It also gives us another way of seeing that the complete graph has the largest eigenvalues. 

These types of interlacing results are an active area of research. The theorem above covers the case of edge addition; analagous results on vertex addition, edge subdivision, and vertex contraction may be found in \cite{porto2017eigenvalue}. 

For manifolds, bounds on the eigenvalues of $\lap$ are often more challenging to prove than their graph counterparts. A well-known result of Lichnerowicz and Obata bounds $\lam_2$ in terms of the Ricci curvature. We will not give a proof, but state it here for readers more familiar with Riemannian geometry. 

\begin{theorem}[Lichnerowicz-Obata]
Suppose $\fM$ is a compact n-dimensional Riemannian manifold with Ricci curvature satisfying the positive lower bound $Ric(\fM) \ge (n - 1)K$. Then 
\[ \lam_2(\fM) \ge nK \]
with equality if and only if $\fM$ is isometric to the sphere $S^{n}(1)$. 
\end{theorem}

Without the curvature condition of Lichnerowicz-Obata, it is possible for the second eigenvalue of a closed manifolds to be arbitrarily small. In the following example, we construct a dumbbell-shaped object with positive size and arbitrarily small $\lam_2$. 

\paragraph{Example: Cheeger's Dumbbell} Consider two spheres of volume $V$ connected by a small cylinder of radius $\ep$ and length $2L$. Let $f$ be the function that is $1$ on the first sphere, $-1$ on the second sphere, and linearly decreasing on the cylinder. The gradient of $f$ has norm $1/L$ and is $0$ otherwise. 
Note that $\int_\fM f dV = 0$. The Rayleigh quotient of $f$ is then 
\[ \int_\fM |\grad f|^2 dV = \frac{L^2}{2V}\vol(C) \]
which goes to $0$ as $\ep \to 0$. This quantity upper bounds $\lam_2$, so $\lam_2$ may be made arbitrarily small on a manifold of volume at least $2V$. 

\subsection{Bounds and Boundaries}

The Laplacian and its eigenvalues are intimately connected to the boundaries of subsets of the graph. To express this connection, we need a few more definitions.

Let $G$ be a graph and $S \subset V$ be a subset of the vertices of $G$. We say that the size of the boundary of $S$ is the number of edges between vertices in $S$ and those in $G \setminus S$. 

Define the \textit{conductance} of a subset $S \subset V$ of vertices to be the size of its boundary $\partial S$ relative to the size of the subset (or the size of its complement, whichever is smaller):
\[ h_G(S) = \frac{|\partial S}{\min(|S|, |G \setminus S|)} \]
Define the conductance of a graph, also called the \textit{Cheeger constant} of $G$, to be the minimum conductance of any subset:
\[ h(G) = \min_{S \subset V} h_G(S) \]

Switching to the manifold case, let $\fM$ be a closed $n$-dimensional manifold. The boundary of an $n$-dimensional submanifold $S \subset \fM$ is $(n-1)$-dimensional. For ease of notation, we write $\vol(\cdot)$ to denote the volume of an $n$-dimensional submanifold and $\area(\cdot)$ denote the volume of an $(n-1)$-dimensional region.

Consider a smooth $(n-1)$-dimensional submanifold $B \subset \fM$ that divides $\fM$ into two disjoint submanifolds $S$ and $T$. Let
\[ h_\fM(B) = \frac{\area(B)}{\min(\vol(S), \vol(T))} = \min_{S \subset \fM: 0 \le \vol(S)} \frac{\area(\partial S)}{\min(\vol(S), \vol(M \setminus S))}\]
analogous to $h_G$ above. Also let  
\[ h(\fM) = \min_{S \subset \fM} h_\fM(S) \]
where the minimum is taken over submanifolds $S$ of the form above. We call $h(\fM)$ the \textit{Cheeger isoperimetric constant} or simply the Cheeger constant of $\fM$. 

\subsubsection{Cheeger's Inequality}

Cheeger's inequality is a celebrated result that bounds the conductance of a graph or manifold in terms of $\lam_2$. It is named for geometer Jeff Cheeger, who formulated and proved the result for manifolds. 

\begin{theorem}[Cheeger's Inequality for Graphs] \label{eq:cheeger_g}
For an unweighted $d$-regular graph, 
\[ h(G) \le \sqrt{2 d \lam_2} \]
\end{theorem}

\begin{theorem}[Cheeger's Inequality for Manifolds] \label{eq:cheeger_m}
For a closed manifold $\fM$,
\[ h(\fM) \le \sqrt{2 \lam_2} \]
\end{theorem}

The most remarkable thing about these two theorems is how similar their proofs are --- the proofs are essentially identical! I have included them, as adapted from a brilliant blog post by Luca Trevisan \cite{cheeger_blog}, in Appendix \ref{appendix:cheeger}. 

\subsubsection{Measuring Boundaries}

We now explore how the Laplacian can be used to measure the size of boundaries. 

Starting with the graph case, let $\ind_S$ be the characteristic function (i.e. indicator) of a subset $S \subset V$:  
\[ \ind_S(v) = \begin{cases} 1 & v \in S \\ 0 & v \not \in S \end{cases} \]
Observe that the size of the boundary may be measured by 
\begin{equation} \label{eq:graph_measure_boundary}
    |\partial S| =  \sum_{(i,j)\in E} |\ind_S(i) - \ind_S(j)| 
\end{equation}
because this sum simply counts edges between $S$ and $G \setminus S$. 

Turning to the manifold case, let $S \subset \fM$ be a $n$-dimensional submanifold and let $\ind_S$ be its characteristic function. The analogous statement to \ref{eq:graph_measure_boundary} above would be
\begin{equation} \label{eq:mani_measure_boundary}
    |\partial S| = \int_\fM |\grad \ind_S|\, dV
\end{equation}
but the indicator function is not differentiable on $\partial S \subset \fM$, so this expression does not make sense! 

If it \textit{did} make sense, we see that it would be consistent with the well-known coarea formula. This formula states that for a Lipschitz function $u$ and an $L^1$ function $g$, 
\begin{equation} \label{eq:coarea_formula}
    \int_\fM g(x) |\grad u(x)|\, dx = \int_{\R} \left(\int_{u^{-1}(t)}g(x)\,dV_{n-1}(x)\right)\,dt
\end{equation}
Naively substituting $u = \ind_S$ and $g = 1$ into this formula gives Equation \ref{eq:mani_measure_boundary}. Of course, $\ind_S$ is not Lipschitz, so this substitution is not justified. 

It turns out that it \textit{is} possible to formally justify Equation \ref{eq:mani_measure_boundary}, but doing so requires the machinery of distribution functions. We informally discuss how this is done in the following section on the Laplacian of the indicator. 

\subsubsection{The Laplacian of the Indicator}

The Laplacian of the indicator function, written $\lap\ind_S$, is a generalization of the derivative of the Dirac delta function. 
Intuitively, $\lap\ind_S$ is infinitely positive on the inside of the boundary of $S$, infinitely negative on the outside of the boundary of $S$, and zero on $S \setminus \partial S$. Formally, it is a distribution function, which is to say that it is only defined in the integrand of an integral, where it integrates to a (generalized) Dirac delta function.

For a function $f: \fM \to \R$, integrating $\lap 1_{S} f(x)$ gives: 
\begin{align*}
\int_{\fM} \lap 1_{S} f(x) dV 
&= \int_{\fM} 1_{S} \lap f(x) dV \\
&= \int_{S} \lap f(x) dV = \int_{S} - \text{div}\, \grad f(x) dV \\
&= \int_{\partial S} (- n \cdot \grad f)(x) dS
\end{align*}
where the first inequality follows from the properties of the Laplacian and the second inequality follows from the divergence theorem. This last integral is called the \textit{surface delta function}, as it generalizes the Dirac delta function. For this reason, the Laplacian of the indicator is also sometimes called the \textit{surface delta prime function}.

In practice, the Dirac delta function is often approximated as the limit of smooth bump functions. In the same way, the Laplacian of the indicator is approximated as the limit of the Laplacian of smooth step functions converging to the indicator function on $S$. 

\paragraph{Example: Smooth Approximation of $\lap \ind$ on $S^1$} Since the last two sections were relatively abstract, at this point it may be useful to give a concrete example. 

Consider the manifold $S^1$, viewed as the unit interval $[0,1]$ with periodic boundary conditions and the canonical metric. Suppose we are interested in calculating the size $|\partial D|$ of a segment $D$ whose length is four-fifths of that of the circle. That is, let $D$ be the region $[0.1, 0.9]$, so $S^1 \setminus D =$ $[0,0.1) \cup (0.9, 1]$. \autoref{indicator_on_circle} (top) shows a diagram of our region. 

We will create a family of smooth approximations $\psi_t$, indexed by a parameter $t$, to the indicator function $\ind_D$. We create $\psi_t$ using the sigmoid function
\[ \si_t(x) = (1 + e^{-x \cdot t})^{-1} \]
which converges to $\ind_{x\ge0}$ as $t \to \infty$. Adding two copies of $\si_t$ and reflecting over the line $0.5$ to ensure periodicity, we have
\[ \psi_t(x) = \sigma_t(5 (1-x)-0.5)+\sigma_t(5 x-0.5) \]
Plots of $\psi$ for different values of $t$ are shown in \autoref{indicator_on_circle} (bottom). As $t\to\infty$, $\psi_i$ becomes $\ind_D$. 

To measure $\partial D$, we can now compute 
 \[ \int_{0}^1 |\partial_x \psi_t(x)| \, dx \]
Results of numerical integration using Mathematica for different value of $t$ are displayed in \autoref{indicator_on_circle} (bottom). As $t \to \infty$, this quantity approaches $2$, which is correct as $|\partial D| = |\{0.1,0.9\}| = 2$. 

\begin{figure}[]
    \centering
    \includegraphics[width=\textwidth]{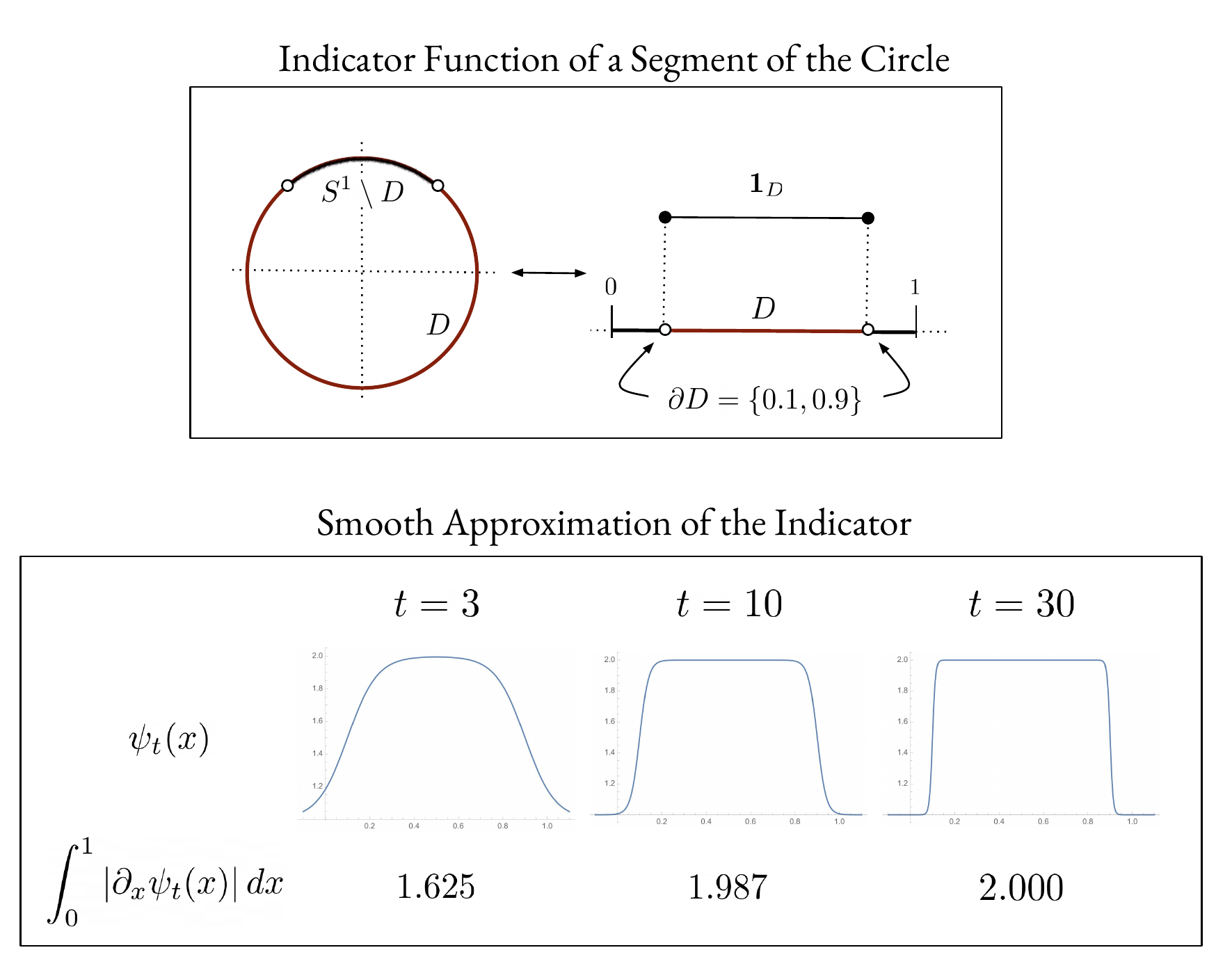}
    \caption[Smooth Approximation to the Indicator]{Above, an illustration of the indicator function of a segment of a circle. Below, graphs of smooth approximations $\psi_t$ to the indicator for $t = 3,10,30$. As $t$ grows large, the integral of $|\partial_x \psi_t(x)|$ approaches $|\partial D| = 2$.}
    \label{indicator_on_circle}
\end{figure}
    
\section{The Heat Kernel} \label{sec:heat_kernel}
We finish this chapter with a short discussion of the heat equation, the classical motivation for the study of the Laplacian. The heat kernel is the key tool of our main proof in \autoref{thm:converge_of_graph_lap}.

We begin with the manifold variant of the heat equation and then discuss the graph variant. 

\subsection{Manifolds}

Let $\fM$ be a closed manifold with measure $\mu$. Define the \textit{heat operator} $L: C^2(\fM) \times C^1((0,\infty))$ by 
\[ L = \lap + \partial_t \]
Let $F(x,t)$ and $f(x)$ be functions on $\fM \times (0,\infty)$ and $\fM$, respectively. The \textit{heat equation} is the partial differential equation 
\begin{align*}
Lu(x,t) &= F(x,t) \\
u(x, 0) &= f(x)
\end{align*}
If $F(x, t) = 0$, we have the \textit{homogenous heat equation}
\begin{align*}
Lu(x,t) &= 0 \\
u(x, 0) &= f(x)
\end{align*}
\begin{theorem}
    A solution to the homogeneous heat equation is unique. 
\end{theorem}
See Appendix \ref{appendix:heat_eq} for the proof. 

A \textit{fundamental solution} to the heat equation is a function $p: \fM \times \fM \times (0,\infty) \to \R$ that is $C^2$ on $\fM \times \fM$ and $C^1$ on $(0,\infty)$ such that 
\[ L_y p = 0, \qquad \lim_{t\to0} p(\cdot, y, t) = \da_y \]
where $\da_y$ is the Dirac delta function. Fundamental solutions may be shown to be unique and symmetric in $x$ and $y$. 

For $t > 0$, define the \textit{heat propagator} operator $e^{-t\lap}: L^2(\fM) \to L^2(\fM)$ as 
\[ e^{-t\lap}f(x) = \int_\fM p(x,y,t)f(y)\,d\mu(x) \]
The heat propagator may be thought of as the solution to the heat equation with initial condition $f(x)$. The following theorems state some of its properties; in essence, $e^{-t\lap}$ behaves as if it were simply an exponentiated function. 

\begin{theorem} The heat propagator satisfies: 
\begin{enumerate}
    \item $e^{-t\lap} \circ e^{-s\lap} = e^{-(s+t)\lap}$
    \item $\left(e^{-\lap}\right)^t = e^{-t \lap}$
    \item $e^{-t\lap}$ is a positive, self-adjoint operator.
    \item $e^{-t\lap}$ is compact.
\end{enumerate}
\end{theorem}

\begin{theorem} \label{thm:heat_identity}
    As $t \to 0$, $e^{-t\lap} \mapsto \text{Id}_{L^2}$, the identity operator in $L_2(\fM)$. 
\end{theorem}

\begin{theorem} \label{thm:heat_constant}
    As $t \to \infty$, $e^{-t\lap}$ converges uniformly in $L^2$ to a constant function (a harmonic function if $\fM$ is not closed). 
\end{theorem}

The next theorem reveals the fundamental connection between the heat equation and the Laplacian spectrum. 

\begin{theorem}[Sturm-Liouville decomposition]
Denote the eigenvalues and eigenfunctions of the Laplacian $\lap$ by $\lam_1 \le \lam_2 \le \cdots$ and $\phi_1, \phi_2, \dots$, respectively. Then 
\[ p(x,y,t) = \sum_{i=0}^\infty e^{-\lam_i t}\phi_i(x)\phi_i(y) \] 
\end{theorem}
See Appendix \ref{appendix:heat_eq} for the proof. 

\subsection{Graphs}
Having developed our heat operator toolkit on manifolds, we now look at the heat kernel on graphs. In what follows, for ease of notation, we work with the normalized Laplacian $\nlap = D^{-1/2}\lapg D^{-1/2}$ rather than $\lapg$. 

For a graph $G$, we define the heat kernel $H_t$ to match the form $e^{-t\lap}$:
\[ H_t = e^{-t\nlap}\]
analogously to the Sturm-Liouville decomposition, it may also be written as a sum of outer products,
\[ H_t = \phi e^{-t\Lambda} \phi^T = \sum_{i=1}^{|V|} e^{-t \lam_i}\phi_i\phi_i^T  \]
where $\lam_i$ and $\phi_i$ are the eigenvalues and eigenvectors of the $\nlap$. 

For $t$ near $0$, $H_t \approx I - \nlap t$ by a Taylor series expansion; the heat kernel depends only on the graph's local structure. In the limit $t \to 0$, it converges to the identity function, as in Theorem \ref{thm:heat_identity} on manifolds. 

Another way of understanding the heat kernel on graphs is to see it as defining a continuous-time random walk. A standard (discrete-time) random walk on $G$ is defined by the random walk matrix $P$,
\[ P_{ij} = \begin{cases}
    1/d_i & (i,j) \in E \\ 0 & \text{otherwise} 
\end{cases} \]
The entries $P_{ij}$ of $P$ may be regarded as the probability of moving $i \to j$ at any time step. For a distribution $v$ over vertices at time $t=0$, the entries of $(P^t v)_i$ may be regarded as the probability of being at vertex $i$ after time $t$. 

By a Taylor expansion, the heat kernel $H_t$ may be written as 
\begin{align*}
H_t &= e^{-t\nlap} = e^{-t}\left(I + IP + \frac{IP}{2!} + \cdots \right) \\
    &= \sum_{k=0}^\infty \frac{ t^k e^{-t}}{k!} P^k
\end{align*}
In this way, it describes a random walk with $Pois(1)$ distributed waiting times.

%% file: chapters/ch5-manireg.tex
\chapter{Manifold Regularization} \label{chap:manireg}

This chapter presents manifold regularization, a regularization technique that unites the ideas introduced in the previous four chapters. 





\subsection{Background}

Introduced by \cite{belkin2006manifold} in 2004, manifold regularization gained attention from machine learning practitioners and theoreticians throughout the mid-late 2000s and early 2010s. It was first grounded in a rigorous theory by \cite{belkin2005towards}, who justified the use of the data graph Laplacians by proving that, in the limit of infinite data, they converge to data manifold Laplacians. One of the primary objectives of this chapter is to give a clear exposition of this proof using the tools of heat kernels. 

A large body of work has emerged around manifold regularization applications and theory in the last decade. Applications include web image annotation, face recognition, human action recognition, and multitask learning \cite{ma2018recent}.
Theoretical analyses have investigated the extent to which the discrete approximations used in manifold regularization (i.e. operators on graphs) conform with the continuous objects that motivate them (i.e. operators on manifolds). 

\subsection{Organization}

This chapter is organized as follows. First, we motivate manifold regularization using a toy example and give its formal definition. 
Second, we present the two representer theorems due to \cite{belkin2006manifold} that characterize the solutions to manifold-regularized learning problems. Third, we give examples of two manifold-regularized learning algorithms (Laplacian RLS, Laplacian SVM). Fourth, we discuss the convergence of the graph Laplacian, which provides a theoretical underpinning to manifold-regularized learning. Finally, we give an overview of recent research in the field and discuss potential directions for future work. 


\section{Manifold Regularization}


Consider the toy example presented in \autoref{toy_example}. It consists of points in the $2D$ plane, two of which have labels (shown as red and blue). Suppose we wish to perform binary classification, which is to say separate the plane into two regions corresponding to the two classes. 


\begin{figure}[]
    \centering
    \includegraphics[width=0.73\textwidth]{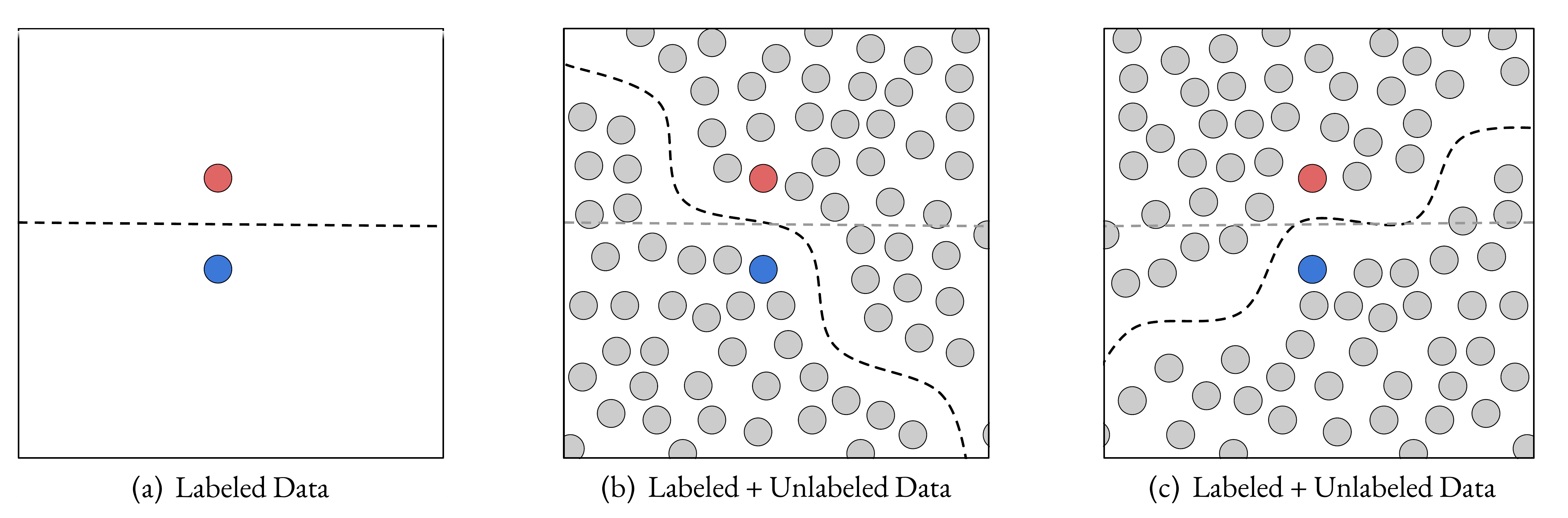}
    \caption[Motivation for Manifold Learning]{A toy example illustrating how the shape of our data can change how we see a binary classification problem.}
    \label{toy_example}
\end{figure}

If we only consider the labeled data (2 points), our notion of a natural classification function (Figure \textit{a}) is a straight line, a smooth function in the extrinsic space ($\R^2$). However, if we add unlabeled data (Figures \textit{b} and \textit{c}), our notion of a natural classification function changes to one that is smooth in the intrinsic space (the data manifold). The shape of our data (\textit{b} vs \textit{c}) determines the natural classification function. 

More generally, suppose we have a learning problem with $N_L$ labeled examples and $N_U$ unlabeled examples: $S = \{(x_i, y_i)\}_{i=1}^{N_L} \cup \{x_i\}_{i=N_L}^{N_L + N_U}$, for $x_i \in X$ and $y_i \in Y$. We assume the data $x_i$ are drawn independently from a probability distribution $\rho_X$ supported on a Riemannian manifold $\fM$. 

Manifold regularization adds a term to the loss function that penalizes functions which are more complex with respect to the intrinsic geometry of the data manifold $\fM$: 
\begin{equation} \label{eq:original_manireg}
    L(f, x, y) = L_{sup}(y, f(x)) + \ga_{\fK} \norm{f}_{\fK}^2 + \ga_{\fI} R_{\fI}(f) 
\end{equation}
where $\norm{f}_{\fK}^2$ is standard (extrinsic) Tikhonov regularization term, $R_{\fI}(f)$ is a new \textit{intrinsic} regularization term. This intrinsic term captures the intuition that our functions should be smooth \textit{on the manifold}, not just smooth in the extrinsic space. 

The constants $\ga_{\fK}$ and $\ga_{\fI}$ determine the strength of extrinsic and intrinsic regularization, respectively. Note that whereas the extrinsic term is data-independent (i.e. it depends only on $f$), the intrinsic term depends the data ($x$) by means of the data manifold. 

As seen throughout the last chapter, we can measure the smoothness of a function $f$ on $\fM$ by the Dirichlet energy, the integral of the Laplacian quadratic form: 
\[ R_{\fI}(f) = \int_{\fM} \norm{f}_{\fI}^2 d\rho_X \]
Our objective is then:
\begin{equation} \label{eq:maniregeq}
    L(f, x, y) = L_{sup}(y, f(x)) + \ga_{\fK} \norm{f}_{\fK}^2 + \ga_{\fI} \int_{\fM} \norm{\grad f(x)} \, d\rho_X(x)
\end{equation}
Clearly, given only finite data, we cannot compute the intrinsic term exactly. The key idea of manifold regularization is to approximate this term by replacing the manifold with a graph approximation. 

Suppose we construct a graph $G$, called a \textit{data graph}, that approximates the data manifold $\fM$. For example, we may take $G$ to be the $k$-nearest neighbors graph (\autoref{ssec:data_graphs}), where each data point $x_i$ is connected by an edge to its $k$ nearest neighbors. 

Substituting the Laplacian $\lapg$ of $G$ for the Laplacian $\lapm$ of $\fM$, the intrinsic term becomes computable: 
\begin{equation} \label{eq:graphregterm}
    R_{\fI}(f) = \frac{1}{(N_U + N_L)^2} \, \bf(x)^{T}\bL\bf(x) \approx \int_{\fM} f(x) \lapm f(x) \, d\rho_X(x)
\end{equation}
where $\bf(x)$ denotes the vector $(f(x_1), \dots, f(x_n))$. Alternatively, expressed in summation notation, we have: 
\begin{equation}
    R_{\fI}(f) = \frac{1}{(N_U + N_L)^2} \, \sum_{i = 0}^{N} \sum_{r \in N(x_i)} w_{ir} (f(x_i) - f(x_r))^{2}
\end{equation}
where $w_{ij}$ is the weight on edge $(i,j)$ if the data graph is weighted, and $w_{ij}=1$ if the data graph is unweighted. 

Substituting $R_{\fI}(f)$ back into Equation \ref{eq:original_manireg} gives the final loss function 
\begin{equation} \label{eq:graphregeq}
\hspace*{-9mm} 
    \sum_{i = 0}^{N_L} L_{sup}(y_i, f(x_i))
    + \ga_{\fK} \norm{f}_{\fK}^2
    + \frac{1}{(N_U + N_L)^2} \, \bf(x)^{T}\bL\bf(x) 
\end{equation}
for an arbitrary supervised loss function $L_{sup}$. 

In summary, the manifold regularization framework has three steps: 
\begin{enumerate}
    \item Construct a graph from one's data (\autoref{ssec:data_graphs}) 
    \item Calculate the Laplacian $\bL$ of the data graph: $\bL = \bD - \bW$
    \item Optimize the regularized objective function:
    \[ \hspace*{-15mm} \hat{f} =  \arg\min_{f\in\fH_K} 
        \sum_{i = 0}^{N_L} L_{sup}(y_i, f(x_i))
    + \ga_{\fK} \norm{f}_{\fK}^2
    + \frac{1}{(N_U + N_L)^2} \, \bf(x)^{T}\bL\bf(x) \]
\end{enumerate}

\section{Representer Theorems} \label{sec:representer_theorems_sec}

Now that we can compute our loss function, we are left with the task of optimizing it. Fortunately, as in the case of Tikhonov regularization, we can characterize the form of the optimal solution $f^*$. 

In this section we state and prove two representer theorems: one for the manifold case of Equation \ref{eq:maniregeq} and one for the graph case of Equation \ref{eq:graphregeq}. We follow the original proofs given in \cite{belkin2006manifold}.

The standard Representer Theorem (\autoref{thm:representer_theorem}) expresses the minimizer of a Tikhonov-regularized loss function in terms of the kernel functions evaluated at the data points $x$. The following manifold regularized extensions are due to \cite{belkin2006manifold}.

\begin{theorem}[Manifold Regularization Representer Theorem] \label{thm:repthmcont}
    Assuming the intrinsic norm $\norm{\cdot}_I$ satisfies a smoothness condition (Equation \ref{eq:rthm_smoothness}), the minimizer $f^*$ of Equation \ref{eq:maniregeq} takes the form: 
    \begin{equation} \label{eq:repthm_cont} 
        f^*(x) = \sum_{i=1}^{N_L} a_i K(x_i, x) + \int_\fM a(y) K(x,y)\,d\rho_X(y) 
    \end{equation}
    \end{theorem}

\begin{theorem}[Graph Regularization Representer Theorem] \label{thm:repthmdisc}
    The minimizer $f^*$ of Equation \ref{eq:graphregeq} takes the form: 
    \begin{equation} \label{eq:repthm_disc} 
        f^*(x) = \sum_{i=1}^{N_L + N_U} a_i K(x_i, x) 
    \end{equation}
    \end{theorem}
    
The remainder of this section is dedicated to proving these theorems, beginning with the manifold case. 

\textit{Idea:} The proof is structured as follows. We use an orthogonality argument to show that we can write $f^*$ as the sum of two quantities. The first, corresponding to the first two terms in Equation \ref{eq:maniregeq}, will be a weighted sum of the kernel function at the data points: 
\[ \sum_{i=1}^{N_L} a_i K(x_i, x) \]
The second, corresponding to the intrinsic term in Equation \ref{eq:maniregeq}, will take the form of a sum $\sum_i a_i e_i$ over basis vectors $e_i$, where the $a_i$ depend on a differential operator $D$. A series of lemmas will show that if $D$ is bounded, this sum lies in the span of the integral operator $I_K$, and so it may be written in the form: 
\[ \int_\fM a(y) K(x,y)\,d\rho_X(y)  \]
Finally, we will show that $D$ is bounded to complete the proof. 

To begin, let $\fH_K$ be a RKHS with kernel $K$ and $\rho$ be a distribution supported on a compact manifold $\fM \subset X$. Consider the $L^2_\rho$ inner product 
\[ \br f, g \kt_\rho = \int_X f(x) g(x) \, d\rho(x)  \]
and let $I_K$ denote the corresponding integral operator
\[ (I_K f)(x) = \br f, k_x \kt = \int f(y) K(x,y)\, d\rho(y) \]
As noted in \ref{ssec:integral_ops}, $I_K$ is a compact self-adjoint operator. Denote its eigenfunctions and eigenvalues by $e_1, e_2, \dots$ and $\lam_1, \lam_2, \dots$, respectively. 

The following properties of $I_K$ will prove helpful shortly. 
\begin{lemma} \label{lemma:orthonormal_basis}
    The functions $\sqrt{\lam_i} e_i$ form an orthonormal basis for $\fH_K$. 
\end{lemma}

\begin{corollary}
    Any $g \in \fH_K$ may be written as $g = \sum_{i=1}^\infty b_i e_i$. 
\end{corollary}

\begin{lemma} \label{lemma:image_of_op}
    A function $f = \sum_{i=1}^\infty a_i e_i$ lies in the image of $I_K$ if and only if
    \begin{equation} \label{eq:cond_1}
        \sum_{i = 1}^{\infty} b_i^{2} < \infty
    \end{equation}
    where $b_i = a_i / \lam_i$. 
\end{lemma}

Proofs of both lemmas are included in Appendix \ref{appendix:integral_ops}.  

Next, consider the closure of the span of the kernels of points $x \in \fM$, denoted $\fS$: 
\[ \fS = \overline{\text{span}\{k_x : x \in \fM\}}\]
Note that $S$ with the induced inner product from $\fH_K$ is a Hilbert space. Let $\fH_{K_\fM}$ and $\fS_\fM$ denote restrictions to $\fM$ of $\fH_K$ and $\fS$, each of which can be seen as Hilbert spaces (with the induced kernel $K$).

We need two properties of $\fS$ and $\fS_\fM$. 
\begin{lemma} \label{lemma:first_prop_of_s}
    $\fH_{K_\fM} = \fS_\fM$
\end{lemma}

\begin{lemma} \label{lemma:repthm_2}
    The complement of $\fS$ is $\fS^\perp = \{f \in \fH : f(\fM) = 0\}$. 
\end{lemma}

Proofs are included in Appendix \ref{appendix:s_properties}.  

We now return to our learning problem
\begin{equation} \label{eq:rep_thm_objective}
\arg\min_{f\in\fH_K} L_{sup}(y, f(x)) + \ga_{\fK} \norm{f}_{\fK}^2 + \ga_{\fI} \int_{\fM} \norm{\grad f(x)} \, d\rho_X(x)
\end{equation}
We proceed in three steps: (1) we show a solution $f$ exists, (2) we show $f \in \fS$, and (3) we show that $f$ has the desired form. 

For ease of notation, let $H$ denote the loss we aim to minimize in \ref{eq:rep_thm_objective}. 
\[ H(f) = L_{sup}(y, f(x)) + \ga_{\fK} \norm{f}_{\fK}^2 + \ga_{\fI} \norm{f}_\fI^2 \] 
where we write $\norm{f}_\fI^2$ in place of $\int_{\fM} \norm{\grad f(x)} \, d\rho_X(x)$. 
\begin{lemma} \label{lemma:minimizer_exists}
    A minimizer $f^*$ of Equation \ref{eq:rep_thm_objective} exists. 
\end{lemma}
\begin{proof}
    Consider a ball $\fB_r \subset \fH_k$ of radius $r$: $\fB_r = \{f \in \fS : \norm{f}_K \le r \}$. Since this ball is compact in $L^\infty$, there must exist a minimizer $f_r^* \in \fB_r$ of Equation \ref{eq:rep_thm_objective} in this ball. 

    The zero function gives us a lower bound on $H(f_r^*)$: 
    \[ H(f_r^*) \le H(0) = \frac{1}{N_L} \sum_{i=1}^{N_L} L_{sup}(x_i, y_i, 0) \]
    If the zero function is a solution, we are done. Otherwise, we obtain a bound on the $\norm{\cdot}_K$ term: 
    \[  \norm{f}_{\fK}^2 \le \frac{1}{\ga_{\fK} } \l( L_{sup}(y, f(x)) + \ga_{\fI} \norm{f}_\fI^2 \r) < \frac{1}{N_L \ga_{\fK}} \sum_{i=1}^{N_L} L_{sup}(x_i, y_i, 0) \]
    If we keep increasing the radius $r$ of our ball, $H(f)$ must be lower bounded (because the right hand side is fixed). Specifically, the minimizer cannot be found outside the ball of radius $r = \sqrt{\frac{1}{N_L \ga_{\fK}} \sum_{i=1}^{N_L} L_{sup}(x_i, y_i, 0)}$. 

    Therefore there exists a solution $f^*$. 
    
    Also, if $V$ is convex then the full objective is convex and the solution is unique. 
\end{proof}

\begin{lemma} \label{lemma:minimizer_in_s}
    If the intrinsic norm $\norm{\cdot}_\fI$ satisfies the following smoothness condition: 
    \begin{equation} \label{eq:rthm_smoothness}
    f|_\fM = g|_\fM \implies \norm{f}_I = \norm{g}_I \qquad \forall f,g \in \fH_K
    \end{equation}
    Then the solution $f^*$ of Equation \ref{eq:rep_thm_objective} lies in $\fS$. 
\end{lemma}
\begin{proof}
    Let $f \in \fH_K$. Decompose $f$ into the orthogonal projections $f = f_\fS + f_{\fS^{\perp}}$. By \autoref{lemma:repthm_2},  $f_{\fS^{\perp}} = 0$ on $\fM$. Then $(f - f_\fS) = 0$ on $\fM$, so for the intrinsic norm:
    \[ \norm{f}^2_I = \norm{f_{\fS}}^2_I \]
    For the extrinsic norm, we have
    \[ \norm{f}^2_K = \norm{f_{\fS}}^2_K + \norm{f_{\fS^{\perp}}}^2_K \]
    which implies
    \[ \norm{f}^2_K \ge \norm{f_{\fS}}^2_K \]
    This shows that $f^{*} \in \fS$, because if $f^{*}$ had any component orthogonal to $\fS$, this component would contribute strictly positively to the expression in Equation \ref{eq:rep_thm_objective}. 
\end{proof}
From now on, we will assume that $\norm{\cdot}_I$ satisfies the smoothness condition (\ref{eq:rthm_smoothness}). 

We have finally built up to the main result. 

\begin{theorem}
The minimizer $f^*$ of 
\begin{equation} \label{eq:repthm_cont_obj} 
    H(f) = \frac{1}{N_L} \sum_{j=1}^{N_L} L_{sup}(y_j, f(x_j)) + \ga_{\fK} \norm{f}_{\fK}^2 + \ga_{\fI} \norm{f}_{\fI}^2
\end{equation}
takes the form: 
\begin{equation} \label{eq:repthm_cont_thm} 
    f^*(x) = \sum_{i=1}^{N_L} a_i K(x_i, x) + \int_\fM a(y) K(x,y)\,d\rho_X(y) 
\end{equation}
\end{theorem}
\begin{proof}
By \autoref{lemma:minimizer_exists}, a minimizer $f^*(x)$ exists. By \autoref{lemma:minimizer_in_s}, $f^*(x) \in \fS$, the closure of kernel functions centered at points in $\fM$. By \autoref{lemma:orthonormal_basis}, we can write $f^{*} = \sum_{i=1}^{\infty} a_i e_i$, where $\{e_i\}$ are the basis formed by the eigenvectors of the integral operator $I_K$, which we defined above as $I_K(f) = \sum_{\fM} f(y) K(x,y)\,d\rho_X(y)$.

We will show that $f^{*}$ decomposes into two terms, the first of which is a finite sum of kernel functions at the data points $x_i$, and the second of which is lies in the image of $I_K$ and so may be written as $\sum_{\fM} a(y) K(x,y)\,d\rho_X(y)$ for some function $a$. 

To begin, we plug $f^{*} = \sum_{i=1}^{\infty} a_i e_i$ into $H$: 
\[ \hspace*{-1pt} H(f^*) = \frac{1}{N_L} \sum_{j=1}^{N_L} L_{sup}\l(y_j, \sum_{i=1}^{\infty} a_i e_i(x_j) \r) 
            + \ga_{\fK} \norm{f\sum_{i=1}^{\infty} a_i e_i}_{\fK}^2 
            + \ga_{\fI} \norm{f\sum_{i=1}^{\infty} a_i e_i }_{\fI}^2 \]
We differentiate with respect to $a_k$ and set the result to $0$:
\begin{align*}
\hspace*{-6pt} 0 &= 
    \ppx{H(f^{*})}{a_k} = \frac{1}{N_L} \sum_{j=1}^{N_L} e_k(x_j) \pl_{(2)} L_{sup}(y_j, f^*(x_j))
    + 2 \ga_{\fK} \frac{a_k}{\lam_k}
    + \ga_{\fI} \br (D + D^*) f, e_k \kt 
\end{align*}
where $D$ is a differential operator, $D^*$ is its adjoint, and $\pl_{(2)}$ is the partial with respect to the second input of $L_{sup}$. Note the two terms above corresponding to the norms hold because
\[ \ppx{H(f^{*})}{a_k} \norm{f\sum_{i=1}^{\infty} a_i e_i}_{\fK}^2 = \ppx{H(f^{*})}{a_k} \sum_{i=1}^{\infty} \frac{a_i^2}{\lam_i} = 2 \frac{a_k}{\lam_k} \qquad\text{and} \]
\[ \ppx{H(f^{*})}{a_k} \norm{f\sum_{i=1}^{\infty} a_i e_i}_{\fK}^2 = \br Df, e_k \kt + \br f, D e_k \kt = \br (D + D^*) f, e_k \kt \]
Solving the equation above for $a_k$ yields
\[ a_k = -\frac{\lam_k}{2 \ga_\fK} \left(   \ga_{\fI} \frac{1}{N_L} \sum_{j=1}^{N_L} e_k(x_j) \pl_{(2)} L_{sup}(y_j, f^*(x_j)) + \br (D + D^*) f, e_k \kt  \right)   \]
We can plug this expression back into $f^{*} = \sum_{i=1}^{\infty} a_i e_i$ to give
\begin{align*}
\hspace*{-30pt} f^*(x) 
&= - \frac{1}{ 2 \ga_\fK N_L} \sum_{j=1}^{N_L} \sum_{k=1}^{\infty} \lam_k e_k(x_j) e_k(x) \pl_{(2)} L_{sup}(y_j, f^*(x_j)) -\frac{\lam_k}{2 \ga_\fK} \sum_{k=1}^{\infty} \lam_k \br (D + D^*) f, e_k \kt e_k
\end{align*}
Using the fact that $K(x,y) = \sum_{i=1}^{\infty} \lam_i e_i(x) e_i(y)$, we have: 
\begin{align*}
\hspace*{-20pt} f^*(x) 
&=  - \underbrace{\frac{1}{ 2 \ga_\fK N_L} \sum_{j=1}^{N_L} K(x, x_j) \pl_{(2)} L_{sup}(y_j, f^*(x_j))}_{\text{sum of kernels at data $x_j$}} 
    - \underbrace{\frac{\lam_k}{2 \ga_\fK} \sum_{k=1}^{\infty} \lam_k \br (D + D^*) f, e_k \kt e_k}_{\text{this is in the image of $I_K$}}
\end{align*}
The first term above takes our desired form. By \autoref{lemma:image_of_op}, the second term above is in the image of $I_K$ if and only if: 
\[ \sum_{k=1}^{\infty} \frac{( \lam_k \br (D + D^*) f, e_k \kt)^2}{\lam_k^2} = \sum_{k=1}^{\infty}\br (D + D^*) f, e_k \kt^2 \]
is bounded. \autoref{lemma:d_is_bounded} below shows that $D$ is bounded, implying that $D + D^*$ is bounded and so the expression above is bounded. Given this result, the second term above is in the image of $I_K$, and so takes the form $\int_\fM g(y) K(x,y)\,d\rho_X(y)$. 

Therefore 
\[ f^*(x) = \sum_{i=1}^{N_L} a_i K(x_i, x) + \int_\fM a(y) K(x,y)\,d\rho_X(y)  \]
for some real numbers $a_i$ and some function $a$.
\end{proof}

To complete the proof, all that remains is to show that $D$ is bounded. To do so, we have to be a bit more specific about the geometry of our manifold. Let $\fM$ be a boundaryless manifold with measure $\rho$, $D \in C^{\infty}$ a differential operator, and $K(x,y)$ a kernel with at least $2k$ derivatives. 

\begin{lemma} \label{lemma:d_is_bounded}
    $D: \fS \to L_\rho^2$ is a bounded operator.
\end{lemma}
\begin{proof}
    We show $D$ is bounded on $\fH_K$. Note that the integral operator $I_K$ defined above is compact (and so bounded). As a result, $I_K D$ is bounded, and (by taking the adjoint and composing with $D^{*}$) we have that $D I_K D^*: L_\rho^2 \to L_\rho^2$ is bounded. 

    Consider the square root $I_{K}^{1/2}$ of $I_K$. As seen by the eigenvalues of $I_{K}^{1/2}$ or the relation $I_{K}^{1/2} \circ I_{K}^{1/2} = I_K$, this operator is positive and adjoint. As seen above, $I_{K}^{1/2}: \fH_K \to L_\rho^2$ is an isometry, so any $g \in \fH_K$ may be written as $I_{K}^{1/2}f$ for some $f \in L_\rho^2$. Then $\norm{f}_{L_\rho^2} = \norm{g}_{K}$. We now have
    \begin{equation} \label{eq:d_is_bnd}
        \norm{Dg}_{L_\rho^2} = \norm{DI_K^{1/2}f}_{L_\rho^2} \le \norm{DI_K^{1/2}f}_{L_\rho^2} \norm{f}_{L_\rho^2} = \norm{DI_K^{1/2}f}_{L_\rho^2} \norm{g}_{K}
        \end{equation}
    Finally, we bound $D I_{K}^{1/2}$. Let $\ep > 0$ be arbitrary. There exists $f \in L_\rho^2$ such that $\norm{f}_{L_\rho^2}$ and 
    \begin{align*}
    \norm{DI_{K}^{1/2}}^2_{L_\rho^2}
    = \norm{I_{K}^{1/2} D^*}^2_{L_\rho^2} 
    \le \br I_{K}^{1/2} D^* f, I_{K}^{1/2} D^*  f \kt_{L_\rho^2} = \br D I_{K} D^*, f \kt_{L_\rho^2} \le \norm{D I_{K} D^*} \norm{f}^2
    \end{align*}
    Now $\norm{f}^2 \le (1 + \ep)^2$ and $\norm{D I_{K} D^*}$ is bounded, so $\norm{DI_{K}^{1/2}}^2_{L_\rho^2}$ is bounded. 

    Returning to Equation \ref{eq:d_is_bnd}, we see:
    \[ \norm{Dg}_{L_\rho^2} \le \norm{DI_K^{1/2}f}_{L_\rho^2} \norm{g}_{K} \le C \cdot \norm{g}_{K} \]
    for some constant $C$. Therefore $D$ is a bounded operator $\fS \to L_\rho^2$. 
\end{proof}

With this result, our proof of Theorem \ref{thm:repthmdisc} is complete. 

Fortunately, the proof of the discrete manifold regularization theorem is significantly simpler. It parallels the orthogonality argument from the original representer theorem. 

\begin{theorem}[Theorem \ref{thm:repthmdisc}]
The minimizer $f^*$ of 
\begin{equation} \label{eq:repthm_disc_obj} 
    H(f) = \frac{1}{N_L} \sum_{j=1}^{N_L} L_{sup}(y_j, f(x_j)) + \ga_{\fK} \norm{f}_{\fK}^2 + \frac{\ga_{\fI}}{(N_L + N_U)^2} \bf^T \bL \bf
\end{equation}
takes the form: 
\begin{equation} \label{eq:repthm_disc_thm} 
    f^*(x) = \sum_{i=1}^{N_L + N_U} a_i K(x_i, x) 
\end{equation}
\end{theorem}
\begin{proof}
    Suppose $f$ is a minimizer of Equation \ref{eq:repthm_disc_obj}. Let $S$ be the subspace spanned by the kernel functions $K_{x_i}$ on the data $\{x_i\}_{i=1}^{N_L + N_U}$, in other words the functions that may be written in the form $\sum_{i=1}^{N_L + N_U} a_i K(x_i, x)$ for some coefficients $a_i$. 
    
    Write $f = f_{S} + f_{S^\perp}$, where $f_{S}$ and $f_{S^\perp}$ are orthogonal projections onto $S$ and $S^{\perp}$. Our goal is to show that $f_{S^\perp} = 0$, as then $f$ takes the form $f(x) = f_{S}(x) = \sum_{i=1}^{N_L + N_U} a_i K(x_i, x)$.

    By the reproducing property, we see that the value of $f$ on a data point $x_i$ does not depend on $f_{S^\perp}$:
    \[ f(x_i) = \br f, K_{x_i} \kt = \br f_{S}, K_{x_i} \kt + \br f_{S^\perp}, K_{x_i} \kt = \br f_{S}, K_{x_i} \kt \] 
    Examining Equation \ref{eq:repthm_disc_thm}, the first and third components of $H(f)$ only depend on $f$ evaluated at the data points. Therefore $H(f)$ and $H(f_S)$ differ only on the second component: 
    \[ H(f) - H(f_{S}) = \norm{f}^2_K - \norm{f_S}^2_K = \norm{f_{S^\perp}}^2_K \]
    If $f$ is a minimizer of $H$, this difference cannot be positive, so:
    \[ \norm{f_{S^\perp}}^2_K \le 0 \implies \norm{f_{S^\perp}}^2_K = 0 \implies f_{S^\perp} = 0 \]
    Therefore $f = f_S \in S$ and $f$ takes the form 
    \[ f(x) = \sum_{i=1}^{N_L + N_U} a_i K(x_i, x) \]
\end{proof}

Whereas the manifold-based representer theorem is exclusively of theoretical interest, this graph-based version enables us to compute solutions to manifold regularized learning problems. We give two examples of such algorithms below. 

\section{Algorithms}

In general, to solve a manifold regularized learning problem, we solve for a function in the form given by the representer theorem
\[ f(x) = \sum_{i=1}^{N_L + N_U} a_i K(x_i, x)   \]
by optimizing the parameters $a_i$, usually using gradient-based optimization methods. 

\paragraph{Laplacian Regularized Least Squares (Lap-RLS)} Lap-RLS corresponds to a least squares loss function on the supervised data, $L_{sup}(f(x), y) = (f(x) - y)^2$. Our objective is then
\[ f^* = \arg\min_{f\in\fH} \frac{1}{N_L} \sum_{j=1}^{N_L} (f(x_i) - y_i)^2
                      + \ga_{\fK} \norm{f}_{\fK}^2 
                      + \frac{\ga_{\fI}}{(N_L + N_U)^2} \bf^T \bL \bf   \]
By the representer theorem, our minimizer takes the form $f^* = \sum_{i=1}^{N_L + N_U} a_i K(x_i, x)$. At this point, we would usually use gradient descent on the $a_i$, but in this case we are able to give a closed form. 

To simplify notation, define:
\begin{itemize}
    \item $\ba = (a_1, \dots, a_{N_L+N_U}) \in \R^{(N_L+N_U)}$ to be the vector of coefficients $a_i$
    \item $K = \l K(x_i,x_j)\r_{i,j=1}^{N_L+N_U} \in \R^{(N_L+N_U)\times(N_L+N_U)}$ to be the kernel matrix (or Gram matrix) on the labeled and unlabeled data 
    \item $Y = (y_1, \dots, y_{N_L}, 0, \dots, 0) \in \R^{N_L + N_U}$ to be the label vector on the labeled data and $0$ on the unlabeled data
    \item $J = \text{diag}(1, \dots, 1, 0, \dots, 0) \in \R^{(N_L+N_U)\times(N_L+N_U)}$ to be the matrix with $1$s on the diagonal entries corresponding to the labeled data and $0$ elsewhere. 
\end{itemize} 
Plugging in $f^* = \sum_{i=1}^{N_L + N_U} a_i K(x_i, x)$, our objective is: 
\[ \ba^* = \arg\min_{\ba\in\R^{N_L+N_U}} \frac{1}{N_L}(Y - JK\ba)^T(Y - JK\ba) + \ga_{\fK} \ba^T K \ba + \frac{\ga_\fI}{(u+l)^2} \bf^T \bL \bf \]
Taking a derivative and solving for $\ba^*$ gives: 
\[ \ba^* = \l JK + \ga_\fK I + \tfrac{\ga_\fI}{(N_L + N_U)^2} LK  \r^{-1} Y \]
This is the same as the well-known solution $w^* = \l K + \ga_\fK I \r^{-1} Y$ of the standard RLS problem, with an added term corresponding to the intrinsic norm. 

\paragraph{Laplacian Support Vector Machines (Lap-SVM)} Lap-SVM corresponds to a hinge loss on the supervised data, 
$L_{sup}(f(x), y) = \max(0, 1 - yf(x)) = (1 - y f(x))_{+}$ where $y \in \{-1,1\}$. Our objective is then
\[ f^* = \arg\min_{f\in\fH} \frac{1}{N_L} \sum_{j=1}^{N_L} \max(0, 1 - y_i f(x_i))
                      + \ga_{\fK} \norm{f}_{\fK}^2 
                      + \frac{\ga_{\fI}}{(N_L + N_U)^2} \bf^T \bL \bf \]
Again by the representer theorem, $f^* = \sum_{i=1}^{N_L + N_U} a_i K(x_i, x)$ and we are looking for: 
\begin{align*}
\ba^* = \arg\min_{\ba\in\R^{N_L+N_U}} & \Bigg(
    \frac{1}{N_L} \sum_{j=1}^{N_L} \max\left(0, 1 - y_i \l \textstyle\sum_{i=1}^{N_L + N_U} a_i K(x_i, x) \r\right)
    \\ & \quad + \ga_{\fK} \ba^T K \ba
    + \frac{\ga_{\fI}}{(N_L + N_U)^2} \ba^T K L K \ba \Bigg)
\end{align*}

\paragraph{A Note on Complexity} The primary difficulty with using Lap-RLS, Lap-SVM and similar algorithms in practice is the computational complexity of working with the kernel matrix $K$, a dense $(N_L + N_U) \times (N_L + N_U)$ matrix. In Lap-RLS, for example, the matrix inversion takes $O((N_L + N_U)^3)$ time, which is infeasible for datasets containing millions of unlabeled examples. 

Developing sparse and computationally tractable approximations for the types of objective functions seen above is an active area of research. In fact, it is most active in the Gaussian processes research community, which faces the same challenge of inverting large kernel matrices in Gaussian process regression. 

\paragraph{An Note on the Hessian}
Thus far, almost all our work has been based on the Laplacian operator. A somewhat less popular but still notable theory has arisen in parallel that substitutes the Hessian for the Laplacian. Changing from a Laplacian-regularized loss function to a Hessian-regularized is as simple as changing the quadratic form $f^T \fL f$ to $f^T \fH f$. 

Theoretically, whereas the Laplacian corresponds to the Dirichlet Energy, the Hessian corresponds to the Eells Energy: 
\[ E_{Eells}(f) = \int_{\fM} \norm{\grad_a\grad_b f}^{2}_{T_x^{*}\fM \otimes T_x^{*}\fM} dV(x) \]
Manipulating this expression into normal coordinates yields the Frobenius norm of the Hessian of $f$:
\[ \R(f) = \sum_{i = 1}^{N} \sum_{r,s = 1}^{m} \l \ppx{^{2} f}{x_r \partial x_s} (x_i) \r^{2} \]
However, the second-order nature of the Hessian is a double-edged sword. While it gives the operator the desirable properties mentioned above, it makes the Hessian difficult to compute. To get around this, \cite{kim2009semi} introduce a sparse matrix approximation $\bB$ by fitting a quadratic function to the data points. This approximation yields an objective function almost identical to that of Laplacian-based manifold regularization:
\[ L(f, x, y) = L_{sup}(y, f(x)) + \ga_{\fK} \norm{f}_{\fK}^2 + \ga_{\fI} \b{f}^{T}\b{B}\b{f} \] 
where $\bB$ is analogous to $\bF$ in Equation \ref{eq:repthm_cont_obj}. 

\subsection{Data Graphs} \label{ssec:data_graphs}

Thus far, we have glossed over the first step of manifold learning algorithms: constructing a graph from the data. Here, we briefly give a summary of different types of data graphs. In all cases, the data graph $G=(V,E)$ is undirected and its vertices $V$ correspond to the observed data $\{x_i\}_{i=1}^N$.

Common data graphs include:
\begin{itemize}
    \item \textit{$k$-Nearest-Neighbors Graph}: An edge is created between each data point $x$ and the $k$ other points closest to $x$ (nearest neighbors) according to some distance function $d$. This graph is sparse and connected.
    \item \textit{$\ep$-Neighbors Graph}: An edge is created between all pairs $(x, x')$ of data points with distance less than $\ep$ according to a distance function $d$. Each edge has weight $1$. This graph is sparse, but it may be disconnected. 
    \item \textit{Gaussian-Weighted Graph}: A fully-connected weighted graph is constructed using Gaussian edge weights: $w_{ij} = e^{-\frac{(x_i-x_j)^2}{\si^2}}$ for some $\si^2 > 0$. This graph turns out to have attractive theoretical properties, but unlike the other graphs here it is dense, so it is computationally difficult to work with. 
    \item \textit{$b$-Matching Graph}: A graph is obtained by solving a maximum weight matching problem: $\min_{w} \sum_{w_{ij}d(x_i,x_j)}$ subject to the constraints that $w_{i,j}$ is binary, symmetric, and $b$-regular (i.e. every node has exactly $b$ edges). The solution is sparse, connected, and $b$-regular by construction. It has been found to perform well on small to medium-sized datasets, but solving the matching problem can take $O(dn^3)$ time.
\end{itemize}

\begin{table}[]
\centering
\def\arraystretch{1.4}
\arrayrulecolor{gray}
\rowcolors{1}{white}{gray!15}
\begin{tabular}{ |l|c|c|c|  }
    \hline
    \textit{Type} & \textit{Sparse} & \textit{Connected} & \textit{Construction Time} \\
    \hline
    $k$-Nearest Neighbors   & \cmark    & \cmark  & \textit{Varies}\\
    $\ep$-Neighbors         & \cmark    & \xmark  & \textit{Varies}\\
    Gaussian                & \xmark     & \cmark  & $O(n^2)$ \\
    $b$-Matching            & \cmark    & \cmark  & $O(dn^3)$ \\
    \hline
\end{tabular}
\caption[Data Graph Methods Comparison]{A comparison of different graph construction methods. Note that the running time for $k$-nearest neighbors and $\ep$-neighbors methods depends on the neighbor-finding algorithm chosen. Usually, a fast, approximate algorithm is chosen rather than an exact algorithm. It is also possible to improve the speed of $b$-matching graph construction with loopy belief propagation.}
\label{data_graph_table}
\end{table}


Recent research on graph construction includes methods based on random walks \cite{rao2008affinity}, adaptive coding \cite{weng2016graph}, signal representation \cite{Dong_2019}, and ensembles of different types of graphs \cite{argyriou2006combining}. 

\section{Convergence of the Graph Laplacian}

In the exposition above, we left one final piece of the manifold learning approach without theoretical justification: our approximation of the data manifold Laplacian with the analagous data graph Laplacian. Given the deep connection between manifold and graph Laplacians seen in Chapter \ref{chap:graphsandgeo}, this approximation should hopefully feel natural.

A significant amount of work has gone into proving variants of this convergent result under different sets of assumptions about the distribution of data on the manifold and different constructions of the data graph.

The key result in this area, from \cite{belkin2005towards}, shows that for data that is uniformly distributed on a compact manifold, the Laplacian matrix $\lapg$ of a graph with exponentially-weighted edges converges pointwise to the Laplacian $\lapm$ of the manifold, as the number of data points goes to infinity. 

Since this result was published, numerous follow-up works have relaxed the assumptions required for convergence to hold. 

\cite{hein2007graph} extends the results to the setting of random neighborhood graphs, including the classical random walk graph. \cite{ting2011analysis} relaxes constraints on the smoothness of the kernel function and extends the analysis to include additional types of graphs, including kNN-graphs. \cite{belkin2012toward} argues that singularities and boundaries are an important aspect of realistic data manifolds, and investigates the behavior of the Laplacian near these points. 

\cite{trillos2018variational} proposes a variational approach to investigate the spectral (as opposed to pointwise) convergence of the graph Laplacian, in the case that the data is sampled from an open, bounded, connected set. \cite{wang2015spectral} extends these results to the case of (non-open) manifolds embedded in a high-dimensional ambient space. \cite{wang2015spectral} finds that when the data is sampled without noise, the convergence rate depends on only the intrinsic dimension of the manifold, whereas when it is sampled with noise, the convergence rate also depends on the dimension of the ambient space. Very recently, \cite{trillos2020error} gave error estimates for the spectral convergence rate of the Laplacian of a wide range of graphs. 

We now state the key result from \cite{belkin2005towards} and give an outline of the proof. We encourage the interested reader to read the paper for the full details. 

Let $\fM$ be a compact $k$-dimensional manifold embedded in $\R^N$. Let $S = \{x_i\}_{i=1}^n$ for $x_i$ sampled i.i.d. from the uniform distribution on $\fM$ (that is, the distribution $\rho(x) = 1/\vol(\fM)$ for $x \in \fM$). For notation's sake, let $n$ denote the number of data points ($n$ here corresponds to $N_L + N_U$ above). 

\begin{theorem}[Convergence of the Graph Laplacian] \label{thm:converge_of_graph_lap}
    Fix a function $f \in C^\infty(\fM)$, a point $z \in \fM$, and a constant $a > 0$. Set $t_n = n^{1/(k+2+a)}$. Then 
    \[ \lim_{n\to\infty} \frac{1}{t_n (4 \pi t_n)^{\frac{k}{2}}} \lapg_n^{t_n} f(z) = \frac{1}{\vol(\fM)}\lapm f(z) \]
    where the limit holds in probability. 
\end{theorem}

\textit{Proof Outline:}

The proof has three steps. The first two steps show that $\lapg^t$ converges to $\lapm$ as $t \to 0$ using the heat operator. The final step shows that $\tfrac{1}{n}\lapg_n^t$ converges to $\lapg^t$ as $n \to \infty$ using Hoeffding’s inequality. 

The key idea of the proof is that if one constructs a weighted graph from the data points $\{x_i\}_{i=1}^n$ with Gaussian edge weights, one can associate its (discrete) Laplacian with the (continuous) heat kernel on $\fM$. 

Formally, let $G = (V, E)$ be a fully-connected weighted graph on $|V| = n$ vertices, with each vertex corresponding to a data point $x_i \in S$. Assign to each edge $(i,j) \in G$ the weight
\[ w_{ij} = e^{\frac{\norm{x_i - x_j}^2}{4t}} \]
where $t > 0$. Note that $G$ varies with the number of nodes $n$ and the parameter $t$. 

Consider the Laplacian matrix of $G$, which we write as $\lapt_n$:
\begin{align*}
\lapt_n f(x_i) 
    &= f(x_i) \sum_{j=1}^n w_{ij} - \sum_{j=1}^n f(x_j) w_{ij} \\
    &= f(x_i) \sum_{j=1}^{n} e^{\frac{\norm{x_i - x_j}^2}{4t}}
        - \sum_{j=1}^{n} f(x_j) e^{\frac{\norm{x_i - x_j}^2}{4t}}
\end{align*}
We may extend $\lapt_n$ to a linear operator on functions defined on the ambient space of points $x \in \R^N$:
\begin{align*}
\lapt_n f(x) 
    &= f(x) \sum_{j=1}^{n} e^{-\frac{\norm{x - x_j}^2}{4t}}
        - \sum_{j=1}^{n} f(x_j) e^{-\frac{\norm{x - x_j}^2}{4t}}
\end{align*}
The continuous analogue of this operator, which we denote $\lapt$, generalizes the expression from a discrete set of points $x_j$ to a measure $\rho$:
\begin{align*}
\lapt f(x) 
&= f(x) \int_\fM e^{-\frac{\norm{x - y}^2}{4t}}\,d\rho(y) - \int_\fM f(y) e^{-\frac{\norm{x - y}^2}{4t}}\,d\rho(y) \\
&=  \int_\fM (f(x) - f(y)) e^{-\frac{\norm{x - y}^2}{4t}}\,d\rho(y)
\end{align*}
The first two steps of the proof show that as $t \to 0$, after appropriate scaling, $\lapt$ converges to $\lapm$:
\begin{lemma}
Fix $z \in \fM$. Then:
\[ \lim_{t\to0} \frac{1}{t(4\pi t)^{k/2}} \lapt f(z) = \frac{1}{\vol(\fM)} \lapm f(z) \]
\end{lemma}
In the first step, we restrict our attention to an open ball $B$ around $z \in \fM$ and perform an exponential change of coordinates. This coordinate transformation reduces our computations to computations in $\R^k$. 

In the second step, we show that our (transformed) integral involving $\lapt$ converges to the Laplacian in $\R^N$. The high-level idea is that since the manifold is locally Euclidean, we can restrict our attention to a local space and then prove our result using properties of Gaussians integrals in $\R^N$. 

The third and final step is a straightforward application of Hoeffding’s Inequality to obtain the convergence of $\lapt_n$. 

For the full details, we direct the reader to \cite{belkin2005towards}, which we remark is very well-written. 

\section{Active Areas of Research}

Manifold-regularized learning---both in the theoretical and empirical domains---continues to be a vibrant area of research in machine learning community. 

In the theoretical domain, discussed in the last section, progress continues to be made on generalizing convergence results for the graph Laplacian.

In the empirical domain, manifold regularization is being applied to improve the performance of learning systems on a range of tasks, such as point set registration \cite{ma2019pointset} and zero-shot learning \cite{meng2020zeroshot}. Another line of research is trying to address the primary drawback of manifold-regularized learning methods relative to other popular machine learning approaches, its relatively high computational cost (due to the need to compute $\bf(x)^{T}\bL\bf(x)$ during optimization). This research tries to scale manifold regularization to modern big regimes, where it is not uncommon to deal with millions of data points. \cite{lecouat2018semi} models the data distribution with a neural network and uses it to obtain a Monte-Carlo approximation to the Laplacian term, enabling them to scale to large datasets. Toward the same goal, \cite{li2019approximate} develops an approach based on Nystrom subsampling and preconditioned conjugate gradient descent. 

\section{Conclusion}

The field of manifold learning lies at the intersection of many branches of mathematics. This thesis has sought to elucidate the connections between these branches, with a particular emphasis on the remarkable interplay between graphs and manifolds. 

These connections remain a central topic of study both within and beyond machine learning. Within machine learning, their theoretical and algorithmic implications drive the development of new proofs and algorithms. Beyond machine learning, they provide insight into physics, chemistry, and a host of other domains. 

Finally, the mathematics is beautiful in and of itself. The author hopes that this thesis managed to convey, if nothing else, some small fraction of that beauty to the reader. 

%% file: chapters/appendix.tex
\appendix
\chapter{Appendix}
\section{Supplementary Proofs}

\subsection{Appendix: Eigenvalue Bounds (Manifolds)} \label{appendix:bounds_mani}

\begin{theorem}[Faber-Krahn Inequality]
Let $\Om \subset \R^n$ be a bounded domain with smooth boundary. Let $B \subset \R^n$ be the ball with the same volume as $\Om$. Denote by $\lam_2(\Om)$ the first nonzero eigenvalue of the Laplacian of $\Om$ under Dirichlet boundary conditions ($\partial \Om = 0$). Then 
\[ \lam_2(\Om) \ge \lam_2(B) \]
with equality if and only if $\Om = B$. 
\end{theorem}
The following proof is due to \cite{kwong_2017}. 
\begin{proof}[Proof of Faber-Krahn]
Denote by $f$ the eigenfunction corresponding to $\lam_2(\Om)$. We will construct a radial function $g$ on the ball $B$ that resembles $f$. Define $g: B \to \R^+$ to be the radial function such that 
\[ \vol(f \ge t) = \vol(g \ge t) \]
That is, 
\[ g(x) = \sup\left\{ t \ge 0: \vol(f \ge t) \ge \vol(B_{\norm{x}}) \right\} \]
We have constructed $g$ in this manner so that integrating over $t$ gives: 
\[ \int_\Om f^2\,dV = \int_0^\infty\vol(f^2 \ge t)\,dV = \int_0^\infty\vol(g^2 \ge t)\,dV = \int_B g^2\,dV\]
Using the Rayleigh quotient characterization of the eigenvalue $\lam_2$, we have 
\[ \lam_2(\Om) = \frac{ \int_\Om |\grad f|^2 }{ \int_\Om f^2 } \qtxtq{and} \lam_2(B) = \frac{ \int_B |\grad g|^2 }{ \int_B g^2 }   \]
We have shown that the denominators are equal, so it remains to be shown that $\int_\Om |\grad f|^2 \ge \int_\Om |\grad g|^2$. 

Consider the area of a level set $\{g = t\}$. Since $g$ is radial, it is constant on its own level sets:
\[ \text{Area}\{g = t\} = \int_{\{g = t\}}\,dS = \sqrt{ \int_{\{g = t\}} |\grad g|\, dS \int_{\{g = t\}} \frac{1}{|\grad g|}\, dS } \]
For $f$, by Cauchy-Schwartz: 
\[ \text{Area}\{f = t\} = \int_{\{f = t\}}\,dS \le \sqrt{ \int_{\{f = t\}} |\grad f|\, dS \int_{\{f = t\}} \frac{1}{|\grad f|}\, dS } \]
The key step of the proof is to use the isoperimetric inequality, which states that the ball is the surface with maximal ratio of volume to surface area.  
\begin{align} 
\hspace*{-12pt} \sqrt{ \int_{\{f = t\}} |\grad f|\, dS \int_{\{f = t\}} \frac{1}{|\grad f|}\, dS } &\ge \text{Area}\{f = t\} \ge \text{Area}\{g = t\} \nonumber \\
&= \sqrt{ \int_{\{g = t\}} |\grad g|\, dS \int_{\{g = t\}} \frac{1}{|\grad g|}\, dS } \label{eq:isoper_line}
\end{align} 
Next, the co-area formula states 
\[ \vol(\Om') = \int_{\Om'}dV = \int_{-\infty}^{\infty}\frac{1}{|\grad f|}\text{Area}(f^{-1}(t))dt  \]
which applied to $f$ on $\Om$ and $g$ on $B$ gives: 
\begin{equation} \label{eq:isoper_line_2}
\int_{\{f=t\}}\frac{1}{|\grad f|}\,dS = - \ddx{}{t} \vol(f \ge t) = - \ddx{}{t} \vol(g \ge t) = \int_{\{g=t\}}\frac{1}{|\grad g|}\,dS
\end{equation} 
where the middle equality holds because $\vol(f \ge t) = \vol(g \ge t)$. 

From Equations \ref{eq:isoper_line} and \ref{eq:isoper_line_2}, we see
\[ \int_{\{f = t\}} |\grad f|\, dS \ge \int_{\{g = t\}} |\grad g|\, dS \]
and so 
\[ \int_\Om |\grad f|^2 
= \int_0^\infty \left( \int_{\{f = t\}} |\grad f|\, dS \right)\,dt 
\ge \int_0^\infty \left( \int_{\{g = t\}} |\grad g|\, dS \right)\,dt 
=  \int_\Om |\grad g|^2  \]
This result completes the proof.
\end{proof}

\subsection{Appendix: Eigenvalue Bounds (Graphs)} \label{appendix:bounds_graph}

The following theorem was proven by Miroslav Fiedler in 1973 \cite{fiedler1973algebraic} and is the origin of the term ``Fielder value''.
\begin{theorem}[Fielder]
    \begin{equation} \label{eq:fielder}
        \lam_2 \le \frac{n}{n-1}\min_{v\in V} d_v \qtxtq{and} \lam_n \ge \frac{n}{n-1}\max_{v\in V} d_v 
    \end{equation}

\end{theorem}
\begin{proof}
Define the matrix $M$ by
\[ M = \lapg - \lam_2 (I - J/n) \]
Note that $M\textbf{1} = 0$ for the constant vector $\textbf{1}$ because $(I - J/n)\textbf{1} = 0$. 

Any vector $y$ may be decomposed into its orthogonal components $y = c_1 \textbf{1} + c_2 x$, where $x$ is a unit-length vector orthogonal to $\textbf{1}$. Then we have 
\[ y^T M y = c_2^2 x^T M x = c_2^2(x^T \lapg x - \lam_2)  \]
Since $\lam_2 = \min_{x \perp \textbf{1}, \ns{x} = 1} x^T \lapg x$, the quantity above is always positive, so that $M$ is positive semidefinite. 

Let $M_{ii}$ denote the $i$-th diagonal element of $M$. Note that $M_{ii} \ge 0$ (as it equals $e_i^T M e_i$). We then have
\[ \min_i M_{ii} = \min_i L_{ii} - \lam_2 (1 - 1/n) \ge 0  \]
and rearranging gives \ref{eq:fielder}.
\end{proof}

A bound on $\lam_n$ was proven by Anderson and Morley in 1985 \cite{anderson1985eigenvalues}.
\begin{theorem}[Anderson and Morley] 
    \[ \lam_n \le \max_{(i,j) \in E} (d_i + d_j) \]
\end{theorem}
This bound was strengthened by Merris \cite{merris1998note}, who also provided a simple proof based on Gershgorin's circle theorem. 
\begin{theorem}[Merris]
Let $m(i)$ be the average of the degrees of vertices adjacent to vertex $i$. That is, $m(i) = \tfrac{1}{|N(i)|}\sum_{j\in N(i)} d_j$ where $N(i)$ denotes the neighbors of $i$. Then 
\begin{equation} \label{eq:merris}
    \lam_n \le \max_{i \in V}(d_i + m(i))
\end{equation}
\end{theorem}
\begin{lemma}[Gershgorin's circle theorem]
Let $M$ be an $n\times n$ matrix with entries $m_{ij}$. Let $r_i = \sum_{j\ne i}|m_{ij}|$ be the sum of the non-diagonal elements of the $i$-th row of $M$. Let $D_i = D(m_{ii}, r_i) \subset \mC$ be the closed disk in the complex plane with radius $r_i$ and center $m_{ii}$. Then every eigenvalue of $M$ is contained in some $D_i$.
\end{lemma}
\begin{proof}[Proof of Lemma]
Let $\lam$ be an eigenvalue of $M$ with corresponding eigenvector $v$. Without loss of generality, let the component $v_i$ of $v$ with largest magnitude be $1$. We have 
\[ (Mv)_i = (\lam v)_i = \lam \]
and 
\[ (Mv)_i = \sum_{j} m_{ij} v_j = \sum_{j \ne i} m_{ij} v_j + m_{ii} \]
so then 
\[ |lam - m_{ii}| = \left| \sum_{j \ne i} m_{ij} v_j \right| \le \sum_{j \ne i} |m_{ij}| |v_j| \le \sum_{j \ne i} |m_{ij}| = r_i  \]
showing that $\lam \in D_i$. 
\end{proof}
\begin{proof}[Proof of Merris' Bound]
Consider $\ol{L} = D^{-1} \lapg D$, where $D$ is the diagonal matrix of degrees of vertices. 
\[ \ol{L}_{ij} = \begin{cases}
d_i & i = j \\ -d_j / d_i & (i,j) \in E \\ 0 & \text{otherwise}
\end{cases} \]
Applying Gershgorin's circle theorem gives that every eigenvalue $\lam$ of $\ol{L}$ is bounded by 
\begin{align*}
\hspace*{-20pt} \max_i \ol{L}_{ii} + r_i = \max_i \ol{L}_{ii} + \sum_{j \in N(i)} |-d_j / d_i| = \max_i (d_i + \frac{1}{N(i)} \sum_{j \in N(i)} d_j) = \max_i (d_i + m(i))
\end{align*}
Since $\lapg$ is similar to $D^{-1} \lapg D$, they share the same eigenvalues, and \ref{eq:merris} holds. 
\end{proof}

A simple bound relates $\lam_2$ on a graph to $\lam_2$ on a subset of the vertices. 

\begin{theorem}
For a subset $S \subset V$ of the vertices of $G$, let $G \setminus S$ denote the graph with all vertices in $S$ and edges connecting to $S$ removed. Then
\[\lam_2(G) \le \lam_2(G \setminus S) + |S| \]
\end{theorem}
\begin{proof}
Let $v$ be an eigenvector of the Laplacian of $G \setminus S$ corresponding to the eigenvalue $\lam_2(G \setminus S)$. Consider $v$ as a vector on all of $G$ by adding $0$s in the entries corresponding to $0$. By the Rayleigh characterization of $\lam_2$, 
\[ \lam_2 \le \sum_{(i,j) \in E(G)} (v_i - v_j)^2 \]
Each of these edges has $0,1$, or $2$ vertices in $S$, so 
\[ \lam_2 \le \sum_{(i,j) \in E(G\setminus S)} (v_i - v_j)^2 + \sum_{i \in S}\sum_{j \in N(i)}v_j^2 + 0 \le \lam_2(G \setminus S) + |S| \]
\end{proof}

\subsection{Appendix: Cauchy's Interlacing Theorem} \label{appendix:cauchy}

Cauchy's Interlacing Theorem is a satisfying result relating the eigenvalues of a matrix to those of a principal submatrix of dimension $(n-1)$ (i.e. a submatrix obtained by deleting the same row and column). As one might intuitively expect, these set of eigenvalues cannot differ greatly. 

We prove two versions of this result, the second of which is sometimes called Weyl's Theorem or Weyl's Perturbation Inequality. 

\begin{theorem}[Cauchy's Interlacing Theorem]
Let $A$ be a self-adjoint $n\times n$ matrix. Let $B$ be a principal submatrix of $A$ of dimension $n-1$. Denote the eigenvalues of $A$ and $B$ by $\al_1 \le \cdots \le \al_n$ and $\be_1 \le \cdots \le \be_n$, respectively. Then 
\[ \al_1 \le \be_1 \le \al_2 \le \cdots \le \al_{n-1} \le \be_n \le \al_n \]
\end{theorem}
\begin{proof}
Without loss of generality, let the first row and column of $A$ be deleted. By the Courant-Fischer Theorem applied to $A$,
\[ \al_{k+1} = \max_{S \subset \R, \dim(S) = n-k} \min_{x\in S, x \ne 0} \frac{x^T A x}{x^T x} \]
and by the Courant-Fischer Theorem applied to $B$,
\[ \hspace*{-20pt} \be_k = \max_{S \subset \R^{n-1}, \dim(S) = n-k-1} \min_{x\in S, x \ne 0} \frac{x^T B x}{x^T x} = \max_{S \subset \R^{n-1}, \dim(S) = n-k-1} \min_{x\in S, x \ne 0} \frac{(0 \enskip x)^T A (0 \enskip x)}{x^T x} \]
where $(0 \enskip x)$ is the $n$-dimensional vector with $0$ in its first component and the entries of $x$ in its $(n-1)$ other components. Comparing these expressions, we see $\al_{k+1} \ge \be_k$ because the expression for $\be_k$ is the same as that for $\al_k$, but taken over a smaller space. The other direction ($\al_{k} \le \be_k$) is obtained by the same method applied to $\al_{k}$ and $\be_{k}$. 
\end{proof}

\begin{corollary}
Let $B$ be a principal submatrix of $A$ of dimension $r$. Then 
\[ \al_i \le \be_i \le \al_{i+n-r} \]
\end{corollary}
\begin{proof}
    Apply Cauchy's Interlacing Theorem $r$ times.
\end{proof}

An application of these ideas is that removing an edge from a graph decreases its eigenvalues. The proof here, due to \cite{godsil2013algebraic}, is the simplest proof of which I am aware. It uses heavy machinery from complex analysis, so  

\begin{theorem}[Edges Increase Eigenvalues]
Let $G$ be a non-complete graph and $(i,j)$ an edge not in $E$. Denote by $G'$ the graph $G$ with edge $(i,j)$ added. Then the eigenvalues of $G'$ interlace those of $G$:
\[ 0 = \lam_1(G) = \lam_1(G') \le \lam_2(G) \le \lam_2(G') \le \lam_3(G) \le \cdots \le \lam_n(G) \le \lam_n(G') \]
\end{theorem}
\begin{proof}
Let $L$ and $L'$ be the Laplacians of $G$ and $G'$, respectively. Let $z$ be the vector that is $1$ in the entry corresponding to vertex $i$, $-1$ in the entry corresponding to vertex $j$, and $0$ elsewhere. Then $L' = L - zz^T$. 

For a real number $t$, consider the quantity $tI - L'$. We have 
\begin{align*}
tI - L' = tI - L - zz^T = (tI - L)(I - (tI - L)^{-1} zz^T)
\end{align*}
Taking determinants gives: 
\begin{align*}
\det(tI - L') = \det(tI - L)\det(I - (tI - L)^{-1} zz^T)
\end{align*}
The determinant has the property that $\det(I - CD) = \det(I - DC)$, so 
\[ \det(I - (tI - L)^{-1} zz^T) = 1 - z^T(tI - L)^{-1}z \]
and
\[ \frac{\det(tI - L')}{\det(tI - L)} = 1 - z^T(tI - L)^{-1}z \]
Denote this expression as a function of $t$ by $\psi(t)$.

We now prove a lemma about rational functions of this form. 
\begin{lemma}
    Let $\psi$ be a rational function of the form $\psi(t) = z^T(tI - L)^{-1}z$ for a real self-adjoint matrix $L$. Then 
    \begin{enumerate}
        \item $\psi$ has simple zeros and poles
        \item $\psi' < 0$ where it is defined. 
        \item Consecutive poles of $\psi$ are separated by no more than $1$ zero of $\psi$.
    \end{enumerate}
\end{lemma}
\begin{proof}[Proof of Lemma]
Write 
\[ \psi(t) = \sum_{\lam \in \text{eval($L$)}} \frac{z^T v_\lam z}{t - \lam}\]
where $\text{eval($L$)}$ denotes the set of eigenvalues of $L$ with corresponding eigenvectors $v_\lam$. Note that the poles of this expression are simple. 

Differentiating gives
\[ \psi'(t) = - \sum_{\lam \in \text{eval($L$)}} \frac{z^T v_\lam z}{t - \lam}^2 = -z^T(tI - L)^{-2}z \]
which is negative as $z^T(tI - L)^{-2}z = \ns{(tI - L)^{-1}z}$. Then each zero of $\psi$ is simple. 

Now consider consecutive poles $a$ and $b$ of $\psi$. As they are simple and $\psi' < 0$, $\psi$ is strictly decreasing on $[a,b]$. Since $t$ is positive near $a$ in this interval and negative near $b$ in this interval, it follows that $\psi$ has exactly one zero in $[a,b]$. This result completes the lemma. 
\end{proof}

We now complete the main proof. Applying the lemma with $\psi(t)$ defined as above, we see that $\psi$ has simple zeros and poles, with consecutive poles separated by a single zero. Its poles are the zeros of $\det(tI - L)$ and its zeros are the zeros of $\det(tI - L')$. In other words, its poles are the eigenvalues of $L$ and its zeros are the eigenvalues of $L'$.  It follows from the lemma that the $n$ zeros and poles of $\psi$ interlace. 

It remains to be shown that this interlacing begins with an eigenvalues of $L$ (and not $L'$), but this is clear because the trace of $L'$ (the sum of the eigenvalues) is $2$ greater than the trace of $L$. 
\end{proof}

\subsection{Appendix: Cheeger's Inequality} \label{appendix:cheeger}
Cheeger's Inequality relates the conductance of a graph or manifold to its second eigenvalue $\lam_2$. 

\begin{theorem}[Cheeger's Inequality for Graphs]
For an unweighted $d$-regular graph, 
\[ h(G) \le \sqrt{2 d \lam_2} \]
\end{theorem}

\begin{theorem}[Cheeger's Inequality for Manifolds]
For a closed manifold $\fM$,
\[ h(\fM) \le \sqrt{2 \lam_2} \]
\end{theorem}

The following proofs are due to \cite{cheeger_blog}.

\begin{proof}[Proof (Graphs)]
The proof is based on the Rayleigh characterization of $\lam_2$. For ease of notation, alongside the Rayleigh quotient $R(f)$, define the $L^1$ Rayleigh quotient $R^1(f)$ as
\[ R^1(f) = \frac{ \sum_{(i,j)\in E} |f(i) - f(j)| }{ \sum_{(i,j)\in E} |f(i)| } \]
As an aside, note that we used the $L^1$ Rayleigh quotient above (without defining it) to measure the boundary of subsets. 

The proof proceeds in three lemmas. The outline is as follows: 
\begin{enumerate}
    \item First, we show there exists a nonnegative function $\hat{f}$ supported on at most half the vertices of $G$ such that $R(\hat{f}) \le \lam_2$.
    \item Second, we consider the elementwise square of $\hat{f}$, denoted $g$. We show 
    \[ R^1(g) \le \sqrt{2 d R(\hat{f})} \]
    \item Third, we show there exists a real $t \ge 0$ such that the set $S = \{i: g(i) > t\}$ has 
    \[ h_G(S) \le R^1(g) \]
    Then $h(G) \le h_G(S) \le R^1(g) \le \sqrt{2 d R(\hat{f})} \le \sqrt{2 d \lam_2}$. 
\end{enumerate}

\begin{lemma}[G1] \label{lemma:g1}
Let $f$ be a vector orthogonal to the constant vector. Then there exists a vector $\hat{f}$ with nonnegative entries such that: 
\begin{enumerate}
    \item $|\{i : \hat{f}(i) > 0 \}| \le \tfrac{1}{2}|V|$
    \item $R(\hat{f}) \le R(f)$
\end{enumerate}
\end{lemma}
\begin{proof}
    Denote by $m$ the median of the entires of $f$. Let $\ol{f} = f - m\textbf{1}$, where $\textbf{1}$ is the constant vector of $1$s. We have 
    \[ \br \ol{f}, \lap \ol{f} \kt \br f - m\textbf{1}, \lap (f - m\textbf{1}) \kt = 0 + \br f, \lap f \kt \]
    and 
    \[ \br \ol{f}, \ol{f} \kt = \br f - m\textbf{1}, f - m\textbf{1} \kt = \br f , f \kt + \br m\textbf{1}, m\textbf{1}\kt \ge \br f , f \kt \]
    because $f \perp \textbf{1}$ and $\lap \textbf{1} = 0$. Then
    \[ R(\ol{f}) = \frac{\br \ol{f}, \lap \ol{f} \kt \br}{ \br \ol{f}, \ol{f} \kt } \le \frac{\br f, \lap f \kt \br}{ \br f, f \kt } = R(f) = \lam_2 \]
    Now split $f$ into two vectors consisting of its positive and negative components, $f = f^+ - f^-$. That is, $f^+_i = \max(0, \ol{f}_i)$ and $f^-_i = \max(0, - \ol{f}_i)$. 
    
    Let $\hat{f}$ be the vector in $\{f^+,f^-\}$ with smaller Rayleigh quotient. 
    \[ \hat{f} = \begin{cases} f^+ & R(f^+) < R(f^-) \\ f^- & otherwise \end{cases} \]
    Since both $f^+$ and $f^-$ have at most $|V|/2$ nonzero entries, $\hat{f}$ is supported on at most half the vertices of $G$. It remains to bound $\min(R(f^+), R(f^-))$. 

    Using the fact that for $a_1, b_1, a_2, b_2 > 0$,
    \[ \min\left(\frac{a_1}{b_1}, \frac{a_2}{b_2}\right) \le \frac{a_1 +a_2}{b_1 + b_2} \]
    we obtain
    \begin{align*}
    \min(R(f^+), R(f^-)) 
    &= \min\left(\frac{\br f^+, \lap f^+ \kt}{\br f^+, f^+ \kt}, \frac{\br f^-, \lap f^- \kt}{\br f^-, f^- \kt} \right) \\
    &\le \frac{\br f^+, \lap f^+ \kt + \br f^-, \lap f^- \kt}{\br f^+, f^+ \kt + \br f^-, f^- \kt}
    \end{align*}
    Since $f^+$ and $f^-$ have disjoint support, $\br f^+, f^- \kt = 0$ and 
    \[ \br f^+, f^+ \kt + \br f^-, f^- \kt = \br f^+ - f^-, f^+ - f^- \kt = \br \ol{f}, \ol{f} \kt \]
    Also, by the triangle inequality, 
    \[ \br f^+, \lap f^+ \kt + \br f^-, \lap f^- \kt \le \br \ol{f}, \lap \ol{f} \kt \]
    As a result, 
    \begin{align*}
        \min(R(f^+), R(f^-)) &\le \frac{\br \ol{f}, \lap \ol{f} \kt}{\br f^+, f^+ \kt + \br f^-, f^- \kt} \\
        &= R(\ol{f}) \le R(f) = \lam_2
    \end{align*}
    which completes the proof of the lemma.
\end{proof}    

\begin{lemma}[G2] \label{lemma:g2}
    For a vector $f$, if $g$ is defined by $g_{i} = f_{i}^2$, then $R^1(g) \le \sqrt{2 d R(f)}$. 
\end{lemma}
\begin{proof}
This lemma is the Cauchy-Schwartz inequality in disguise. Applying Cauchy-Schwartz to the numerator of $R^1(g)$ gives 
\begin{align*} 
\sum_{(i,j)\in E} |g(i) - g(j)| &= \sum_{(i,j)\in E} |f^2(i) - f^2(j)| \\
&= \sum_{(i,j)\in E} |f(i) - f(j)| (f(i) + f(j)) \\
&\le \sqrt{\sum_{(i,j)\in E} (f(i) - f(j))^2} \sqrt{\sum_{(i,j)\in E} (f(i) + f(j))^2 } \qquad\text{(CS)}  \\
&= \sqrt{\sum_{(i,j)\in E} R(f) \sum_{i} f(i)^2 } \sqrt{\sum_{(i,j)\in E} (f(i) + f(j))^2 } \qquad\text{(def of $R(f)$)}\\ 
&\le \sqrt{\sum_{(i,j)\in E} R(f) \sum_{i} f(i)^2 } \sqrt{\sum_{(i,j)\in E} 2f(i)^2 + 2f(j)^2 }\\ 
&= \sqrt{R(f) \sum_{i} f(i)^2} \sqrt{2 d \sum_{i} f(i)^2} \\
&= \sqrt{2 d R(f)} \sum_{i} f(i)^2 = \sqrt{2 d R(f)} \sum_{i} g(i)
\end{align*}
Therefore 
\[ R^1(g) = \frac{\sum_{(i,j)\in E} |g(i) - g(j)|}{\sum_{i} g(i) } \le \sqrt{2 d R(f)} \]
\end{proof}

\begin{lemma}[G3] \label{lemma:g3}
    For every nonnegative vector $g$, there is a real $t > 0$ such that 
    \[ \frac{|\partial \{i : g(i) > t\} | }{|\{i : g(i) > t\}|} \le R^1(g) \]
\end{lemma}
\begin{proof}
    Let $S_t = \{i : g(i) > t\}$. These $S_t$ are sometimes called \textit{Sweep sets}. For each edge $(i,j)$, let $\ind_{ij}^{t}$ denote the indicator that $(i,j) \in S_t$. 
    
    First, we relate $|\partial S_t|$ to $R^1(g)$. The numerator of $R_1(g)$ may be expressed as
    \begin{align*}
        \sum_{(i,j) \in E} |g(i) - g(j)| = \sum_{(i,j) \in E} \int_{0}^{\infty} \ind_{ij}^{t} dt   
    \end{align*}
    The size of the boundary of $S_t$ is $|\partial S_t| = \sum_{(i,j)\in E} \ind_{ij}^{t}$, so 
    \[ \int_0^\infty |\partial S_t| = \sum_{(i,j) \in E} |g(i) - g(j)| \]
    Also note that 
    \[ \int_0^\infty |S_t| = \sum_{i} g(i) \]
    Putting these together, we have
    \begin{align*}
    R^1(g) &= \frac{\sum_{(i,j) \in E} |g(i) - g(j)|}{ \sum_{i} g(i) } \\
    &\le \frac{ \int_0^\infty |\partial S_t| }{ \int_0^\infty |S_t| } 
    \end{align*}
    Letting $t^*$ be the minimizer of $|\partial S_t| / |S_t|$, we have 
    \begin{align*}
    R^1(g) &\le \frac{ \int_0^\infty |\partial S_{t^*}| / |S_{t^*}| |S_t| }{ \int_0^\infty |S_t| } \\
    &= |\partial S_{t^*}| / |S_{t^*}|
    \end{align*}
    Therefore $t^*$ satisfies the statement of the lemma. 
\end{proof}

We may now complete the proof of Cheeger's Inequality. Let $f$ be the eigenfunction corresponding to $\lam_2$. By Lemma \nameref{lemma:g1}, we obtain a corresponding nonnegative function $\hat{f}$, supported on at most half the vertices of $G$, such that $R(\hat{f}) \le R(f) = \lam_2$. By Lemma \nameref{lemma:g2}, with $g$ denoting the elementwise square of $\hat{f}$, we obtain
\[ R^1(g) \le \sqrt{2 d R(\hat{f})} \]
Apply Lemma \nameref{lemma:g3} and denote the resulting set by $S = \{i : g(i) > t\}$. By the lemma and the fact that $S$ contains at most half the vertices of $G$, we have 
\[ h_G(S) = \frac{|\partial S|}{|S|} \le R^1(g) \le \sqrt{2 d \hat{f}} \le \sqrt{2 d \lam_2} \]
\end{proof}

The proof of the manifold case follows a nearly identical structure.

\begin{proof}[Proof (Manifolds)]
Similarly to the proof above, define the $L^1$ Rayleigh quotient $R^1(f)$ as
\[ R^1(f) = \frac{ \int_\fM \norm{\grad f} dV }{ \int_\fM \norm{f} dV } \]

The proof proceeds in three lemmas.
\begin{enumerate}
    \item First, we show there exists a nonnegative function $\hat{f}$ with supported on a set of volume at most $\tfrac{1}{2}\vol(\fM)$ such that $R(\hat{f}) \le \lam_2$.
    \item Second, we show that
    \[ R^1(f^2) \le \sqrt{2 R(\hat{f})} \]
    \item Third, we show there exists a real $t \ge 0$ such that the set $S = \{x: f^2(x) > t\}$ has 
    \[ h_\fM(S) \le R^1(f^2) \]
    Then we have 
    \[ h(\fM) \le h_\fM(S) \le R^1(f^2) \le \sqrt{2 R(\hat{f})} \le \sqrt{2 \lam_2} \]
\end{enumerate}

\begin{lemma}[M1  \label{lemma:m1}(Manifolds)]
Let $f$ be a function with $\int_\fM f = 0$. Then there exists a function $\hat{f} \ge 0$ such that: 
\begin{enumerate}
    \item $\vol(\{x : \hat{f}(x) > 0 \}) \le \tfrac{1}{2}\vol(\fM)$
    \item $R(\hat{f}) \le R(f)$
\end{enumerate}
\end{lemma}
\begin{proof}
    Let $m$ be a median of $f$, which is to say the smallest $m$ such that $\vol(\{x : f(x) < m\}) \ge 1/2$. Let $\ol{f}(x) = f(x) - m$. The numerators of $R(\ol{f})$ and $R(f)$ are the same 
    \[ \br \ol{f}, \lap \ol{f} \kt \br f - m, \lap (f - m) \kt = 0 + \br f, \lap f \kt \]
    since the Laplacian of a constant is $0$. The denominator of $R(\ol{f})$ is larger 
    \[ \br \ol{f}, \ol{f} \kt = \br f - m, f - m \kt = \br f , f \kt + \br m, m\kt \ge \br f , f \kt \]
    because $f$ is orthogonal to a constant (i.e. it integrates to $0$). Note that whereas in the proof above, these inner products referred to matrix products, here they refer to integration over $\fM$. 
    
    We then have 
    \[ R(\ol{f}) = \frac{\br \ol{f}, \lap \ol{f} \kt \br}{ \br \ol{f}, \ol{f} \kt } \le \frac{\br f, \lap f \kt \br}{ \br f, f \kt } = R(f) = \lam_2 \]
    Now let $f^+_i = \max(0, \ol{f}_i)$ and $f^-_i = \max(0, - \ol{f}_i)$. Define $\hat{f}$ be the function in $\{f^+,f^-\}$ with smaller Rayleigh quotient. Note that $\hat{f}$ is supported on a region with volume at most half of that of $\fM$. 

    The remainder of the proof is exactly the same as the proof for graphs above. 
    \begin{align*}
    \min(R(f^+), R(f^-)) 
    &= \min\left(\frac{\br f^+, \lap f^+ \kt}{\br f^+, f^+ \kt}, \frac{\br f^-, \lap f^- \kt}{\br f^-, f^- \kt} \right) \\
    &\le \frac{\br f^+, \lap f^+ \kt + \br f^-, \lap f^- \kt}{\br f^+, f^+ \kt + \br f^-, f^- \kt}
    \end{align*}
    Since $f^+$ and $f^-$ have disjoint support, $\br f^+, f^- \kt = 0$ and 
    \[ \br f^+, f^+ \kt + \br f^-, f^- \kt = \br f^+ - f^-, f^+ - f^- \kt = \br \ol{f}, \ol{f} \kt \]
    Also, by the triangle inequality, 
    \[ \br f^+, \lap f^+ \kt + \br f^-, \lap f^- \kt \le \br \ol{f}, \lap \ol{f} \kt \]
    As a result, 
    \begin{align*}
        \min(R(f^+), R(f^-)) &\le \frac{\br \ol{f}, \lap \ol{f} \kt }{\br f^+, f^+ \kt + \br f^-, f^- \kt} \\
        &= R(\ol{f}) \le R(f) = \lam_2
    \end{align*}
    which completes the proof of the lemma.
\end{proof}    

\begin{lemma}[M2] \label{lemma:m2}
For nonnegative $f$, 
\[ R^1(f^2) \le \sqrt{2 R(f)} \]
\end{lemma}
\begin{proof}
We apply the chain rule and the Cauchy-Schwartz inequality: 
\begin{align*}
\int_\fM \norm{\grad (f^2)} \,dV &= \int_\fM 2 |f| \norm{\grad f}\,dV \qquad \qquad\text{(Chain Rule)} \\
&\le \sqrt{ \int_\fM 4 f^2 \,dV } \sqrt{ \int_\fM \ns{\grad f}\,dV } \qquad\text{(CS)}  \\
&= 2 \int_\fM f^2 \,dV \cdot \sqrt{R(f)} 
\end{align*}
Therefore 
\[ R^1(f^2) \le \sqrt{2 R(f)} \]
\end{proof}

\begin{lemma}[M3] \label{lemma:m3}
For every nonnegative function $g$, there is a real $t > 0$ such that
\[ \frac{\area(\partial \{x : g(x) > t\}) | }{ \vol(\{x : g(x) > t\}) } \le R^1(g) \]
\end{lemma}
\begin{proof}
Let $S_t = \{x : g(x) > t\}$. Consider the numerator and denominator of $R^1(g)$. 

For the numerator, the coarea formula (\ref{eq:coarea_formula}) states
\[ \int_\fM \norm{\grad g}\, dV = \int_0^\infty \area(\partial S_t) \,dt  \]
For the denominator, observe that 
\[ \int_\fM \norm{g}\, dV = \int_0^\infty \vol(S_t) \,dt  \]
Putting these together, we have
\begin{align*}
R^1(g) &= \frac{ \int_0^\infty \area(\partial S_t) \,dt }{ \int_0^\infty \vol(S_t) \,dt } 
\end{align*}
Letting $t^*$ be the minimizer of $\area(\partial S_t) / \vol(S_t)$, we have 
\begin{align*}
R^1(g) &\le \frac{ \int_0^\infty \area(\partial S_{t^*}) / \vol(S_{t^*}) \vol(S_t) }{ \int_0^\infty \vol(S_t) } \\
&= \area(\partial S_{t^*}) / \vol(S_{t^*})
\end{align*}
Therefore $t^*$ satisfies the statement of the lemma. 
\end{proof}

To complete the proof of Cheeger's Inequality on manifolds, let $f$ be the eigenfunction corresponding to $\lam_2$. By Lemma \nameref{lemma:m1}, we obtain a function $\hat{f}$ supported on a set with volume at most half that of $\fM$, such that $R(\hat{f}) \le R(f) = \lam_2$. By Lemma \nameref{lemma:m2}, we obtain $R^1(f^2) \le \sqrt{2 R(\hat{f})}$. 
Apply Lemma \nameref{lemma:m3} and denote the result by $S = \{x : g(x) > t\}$. By the lemma and the fact that $\vol(S) \le \tfrac{1}{2}\vol(\fM)$, 
\[ h_G(S) = \frac{|\partial S|}{|S|} \le R^1(g) \le \sqrt{2 \hat{f}} \le \sqrt{2\lam_2} \]
\end{proof}

Upon proving Cheeger's inequality, we have a few remarks. First, Cheeger's inequality is tight; the path graph, which we saw above, has 
\[ h(G) = 1/\ceil{(n-1)/2} \qtxtq{and} \lam_2 \approx \frac{\pi^2}{2(n-1)^2}\]
Second, the proof of Cheeger's inequality for graphs immediately yields an algorithm for finding a subset of vertices with $h_G(S) \le \sqrt{2\lam_2}$. Such a set is called a \textit{sparse cut} of $G$. 

\vspace{1cm}
\begin{algorithm}[H]
\SetAlgoLined
\KwInput{The $2^{\text{nd}}$ eigenfunction $f_2$}
\KwResult{A sparse cut $S \subset V$}
    $f \leftarrow D^{-1/2}f_2$\;
    Sort the vertices so $f(v_1) \le \cdots \le f(v_n)$\;
    Initialize $i\leftarrow 0, \quad S \leftarrow \emptyset, \quad S^{*} \leftarrow \{v_1\}$\;
    \While{$i < n$}{
        $i = i+1$\;
        $S = S \cup \{v_i\}$\; 
        \If{$h_G(S) \le h_G(S^{*})$}{
            $S^{*} \leftarrow S$\;
        }
    }
    \Return{$S^*$}
    \caption[Sparse Cut Algorithm]{Finding a sparse cut from $f_2$}
\end{algorithm}

\subsection{Appendix: Heat Equation} \label{appendix:heat_eq}

\begin{theorem}
Let $u(x,t)$ be a solution to the homogeneous heat equation. Then $\phi(t) = \norm{u(\cdot, t)}_{L^2}$ is a nonincreasing function of $t$. 
\end{theorem}
\begin{proof}
\begin{align*}
\ddx{}{t} \norm{u(\cdot, t)}_{L^2} &= 2 \int_\fM \partial_t u(x,t) u(x,t) \,d\mu(x) \\
&= - 2 \int_\fM \lap u(x,t) u(x,t) \,d\mu(x) \\
&= - 2 \ns{\grad u(\cdot, t)}
\end{align*}
Since the derivative of $\phi(t)$ is always negative, it is a nonincreasing function of $t$. 
\end{proof}

\begin{theorem}
A solution to the homogeneous heat equation is unique. 
\end{theorem}
\begin{proof}
Suppose $u_1$ and $u_2$ solve the homogeneous heat equation. Then $u = u_1 - u_2$ solves 
\begin{align*}
Lu(x,t) &= 0 \\
u(x, 0) &= 0
\end{align*}
By the theorem above, the function $t \mapsto \int_\fM u(x, t)^2\,dx$ is a nonincreasing function of $t$. Since $u(x, 0) = 0$, we must have $u(x,t) = 0$. Therefore $u_1 = u_2$. 
\end{proof}

\begin{theorem}[Sturm-Liouville decomposition]
Denote the eigenvalues and eigenfunctions of the Laplacian $\lap$ by $\lam_1 \le \lam_2 \le \cdots$ and $\phi_1, \phi_2, \dots$, respectively. Then 
\[ p(x,y,t) = \sum_{i=0}^\infty e^{-\lam_i t}\phi_i(x)\phi_i(y) \] 
\end{theorem}

The following proof is adopted from \cite{canzani2013analysis}. 

\begin{proof}
By the spectral theorem, as $e^{-\lap}$ is a compact self-adjoint operator, it has eigenvalues
\[ \be_1 \ge \be_2 \ge \cdots \]
with corresponding eigenfunctions $\phi_1, \dots, \phi_n$. 

Let $\lam_i = - \ln\be_i$. We aim to show these $\lam_i$ are the eigenvalues of $\lap$. By the properties of the heat operator, 
\[ e^{-t\lap}\phi_k = \left(e^{-\lap}\right)^t\phi_k = \be^t_k\phi_k = e^{-t\lam_k}\phi_k \]
As $e^{-t\lap}\phi_k$ solves the heat equation, we have 
\begin{align*}
0   &= L(e^{-t\lap}\phi_k) = L(e^{-t\lam}\phi_k) \\ 
    &= \lap e^{-t\lam} \phi_k + \partial_t e^{-t\lam} \phi_k \\
    &= e^{-t\lam}(\lap \phi_k - \lam_k \phi_k)
\end{align*}
so $\lap \phi_k = \lam_k \phi_k$, and $\lam_k$ is an eigenvalue of $\lap$ corresponding to eigenfunction $\phi_k$. 

Note that by the definition of the heat propagator,
\begin{align*}
\br p(x,\cdot,t), \phi_k \kt \phi_k(y) = \int_\fM p(x,y,t) \phi_k(y)\,d\mu(y) = e^{-t\lap}\phi_k(x) = e^{-t\lam_k}\phi_k(x) 
\end{align*}
Finally, since the $\phi_i$ form a basis for $L^2(\fM)$, we can write $p$ as 
\[ p(x,y,t) = \sum_{k=0}^\infty \br p(x,\cdot,t), \phi_k \kt \phi_k(y) = \sum_{i=0}^\infty e^{-\lam_i t}\phi_i(x)\phi_i(y) \]
\end{proof}

\subsection{Appendix: Integral Operators} \label{appendix:integral_ops}

\begin{lemma}
    The functions $\sqrt{\lam_i} e_i$ form an orthonormal basis for $\fH_K$. 
\end{lemma}
\begin{proof}

    First, we show the collection $\{ \sqrt{\lam_i} e_i \}$ are orthonormal in $\fH_K$. 
    Observe that
    \begin{align*}
        \br K_x, e_i \kt_\rho  = \int f(y) K(x,y)\, d\rho(y) = (I_K e_i)(x) = \lam_i e_i(x)
    \end{align*}
    so $K_x = K(x, \cdot) = \sum_{i=1}^\infty \lam_i e_i(x) e_i$. Then by the reproducing property, 
    \begin{align*}
        e_j(x) = \br K_x, e_j \kt_K = \sum_{i=1}^\infty \lam_i e_i(x) \br e_i, e_j \kt_K
    \end{align*}
    which implies 
    \begin{align*}
    \br e_i, e_j \kt_K = \begin{cases} 
    0 & i \ne j \\
    1/\lam_i & i = j
    \end{cases}
    \end{align*}
    Therefore the rescaled vectors $\{ \sqrt{\lam_i} e_i \}$ are orthonormal in $\fH_K$.

    Second, we show that $\{e_i\}$ spans $\fH_K$. Let $f \in \fH_K$ be orthogonal to $e_i$ for all $i$. Then by the reproducing property: 
    \begin{align*}
        f(x) &= \br f, K_x \kt_K = \left\br f, \sum_{i=1}^\infty \lam_i e_i(x) e_i \right\kt_K \\
        &= \left\br f, \sum_{i=1}^\infty \lam_i \br e_i, K_x \kt_K e_i \right\kt_K \\
        &= \sum_{i=1}^\infty \lam_i e_i(x) \br f, e_i \kt_K \\
        &= 0
    \end{align*}
    where the last step holds because $f$ is orthogonal to all $e_i$. This result shows that $\fH_K$ is spanned by $\{e_i\}$. Therefore $\sqrt{\lam_i} e_i$ is an orthonormal basis for $\fH_K$. 

    \textit{Note: } Another way of proving this result would be to consider the square root $I_K^{1/2}$ of the integral operator. $I_K^{1/2}$ is an isometry $L_\rho^2 \to \fH_K$, which is to say: 
    \[ \br f,g \kt_\rho = \br L_K^{1/2}f, L_K^{1/2}g \kt_K, \qquad \qquad \forall f,g \in \fH_K  \]
    And a unit-norm eigenbasis for $L_K^{1/2}$ is $\sqrt{\lam_i} e_i$. 
\end{proof}

\begin{lemma}
    A function $f = \sum_{i=1}^\infty a_i e_i$ lies in the image of $I_K$ if and only if
    \begin{equation} \label{eq:cond_1_app}
        \sum_{i = 1}^{\infty} b_i^{2} < \infty
    \end{equation}
    where $b_i = a_i / \lam_i$. 
\end{lemma}
\begin{proof}
    Suppose Equation \ref{eq:cond_1_app} holds. Let $g = \sum_{i=1}^\infty b_i e_i \in L^2_\rho$. Applying $I_K$ yields:
    \begin{align*}
        I_K(g) &= I_K(\sum_{i = 1}^{\infty} b_i e_i) = \sum_{i = 1}^{\infty} \lam_i b_i e_i \\
               &= \sum_{i = 1}^{\infty} a_i e_i = f
    \end{align*}
    Then $f$ lies in the span of $I_K$. 

    For the converse, suppose $f = I_K(g)$ for some $g \in L^2_\rho$. By the lemma above, we can write $g = \sum_{i=1}^\infty b_i e_i \in L^2_\rho$, so we have $\sum_{i=1}^\infty b_i < \infty$, which is Equation \ref{eq:cond_1_app}.  
\end{proof}

\subsection{Appendix: The Closure of $\vspan{k_x}$} \label{appendix:s_properties}

Let $\fS$, $\fH_{K_\fM}$ and $\fS_\fM$ be defined as in \autoref{sec:representer_theorems_sec} (\autoref{lemma:first_prop_of_s}).

\begin{lemma}
    $\fH_{K_\fM} = \fS_\fM$
\end{lemma}
\begin{proof}
    Let $f_\fM$ be an arbitrary function in $\fM$. By the completeness of $\fH_{K_\fM}$, we can write $f_\fM = \lim_{n\to\infty} f_\fM^{(n)}$, where $f_\fM^{(n)}$ lies in the span of the kernel functions: $f_\fM^{(n)} = \sum_i a_{i}^{(n)} K_{\fM, x}$. 

    Let $f^{(n)}$ be the corresponding sequence in $\fH_K$: $f^{(n)} = \sum_i a_{i}^{(n)} K_{x}$. 

    We see that $f^{(n)}$ is a Cauchy sequence because $\norm{f^{(n)} - f^{(k)}}_K = \norm{f^{(n)}_\fM - f^{(n)}_\fM}_{\fM_K}$ and $f^{(n)}_{\fM_K}$ converges. Then the limit $f = \lim_{n\to\infty} f^{(n)}$ exists and $f = f_\fM$. 

    Therefore $\fH_{K_\fM} \subset \fS_\fM$, and the converse follows with the spaces swapped. 
\end{proof}

\begin{lemma}
    The complement of $\fS$ is $\fS^\perp = \{f \in \fH : f(\fM) = 0\}$. 
\end{lemma}
\begin{proof}
    If $f \in \fS^\perp$, then $f$ vanishes on $\fM$ because $f(x) = \br k_x, f \kt_K = 0$ for $x \in \fM$. 

    Conversely, if $f(\fM) = 0$, then $\br k_x, f \kt_K = f(x) = 0$ for every $x \in \fM$ so $f$ is orthogonal to the closure of $\text{span}\{k_x : x \in \fM\}$.
\end{proof}
